\documentclass{article}

\usepackage{tikz}
\usepackage{bbm}
\usepackage{eqnarray}
\usepackage{color}
\usepackage{bm}
\usepackage{amssymb}
\usepackage{epsfig}
\usepackage{epsf}
\usepackage{float}
\usepackage{subfigure}
\usepackage{amsfonts,amsmath,amsthm,fancyhdr,multirow,hyperref}
\usepackage{booktabs,longtable,authblk}
\usepackage{mathrsfs,hhline,enumerate}

\usepackage[top=2.5cm, bottom=2.5cm, left=3cm, right=3cm]{geometry}
\newtheorem{theorem}{Theorem}[section]
\newtheorem{assumption}{Assumption}
\newtheorem{corollary}[theorem]{Corollary}
\newtheorem{definition}[theorem]{Definition}

\newtheorem{lemma}[theorem]{Lemma}
\newtheorem{proposition}[theorem]{Proposition}
\newtheorem{remark}{Remark}

\newtheorem*{condition}{Condition}

\renewcommand{\H}{\mathcal{H}}

\newcommand{\T}{\mathcal{T}}

\newcommand{\W}{\mathcal{W}}
\renewcommand{\l}{\left}
\renewcommand{\r}{\right}

\newcommand{\be}{\mathbb{E}}
\newcommand{\bn}{\mathbb{N}}
\newcommand{\br}{\mathbb{R}}

\newcommand{\md}{\mathcal{D}}
\newcommand{\me}{\mathcal{E}}
\newcommand{\mg}{\mathcal{G}}
\newcommand{\mv}{\mathcal{V}}
\numberwithin{equation}{section}

\title{Resolution-Independent Analysis of Encoder--Decoder Operator Learning via Limiting Kernels$^\dag$\footnotetext{\dag~All authors contributed equally to this work and are listed alphabetically. The work described in this paper is supported by the National Natural Science Foundation of China under Grant Nos.~12171093 and 12571099, 
and by the Discovery Project (DP240101919) of the Australian Research Council. Email addresses: leishi@fudan.edu.cn (L. Shi), jqyang24@m.fudan.edu.cn (J.-Q. Yang), dingxuan.zhou@sydney.edu.au (D.-X. Zhou). The corresponding author
is Jia-Qi Yang.}}

\author[1]{Lei Shi}
\author[1]{Jia-Qi Yang}
\author[2]{Ding-Xuan Zhou}
\affil[1]{School of Mathematical Sciences and Shanghai Key Laboratory for Contemporary Applied Mathematics, Fudan University, Shanghai 200433, China.}
\affil[2]{School of Mathematics and Statistics, the University of Sydney, Sydney, New South Wales 2006, Australia.}
\date{}

\begin{document}
	\maketitle
\begin{abstract}

Operator learning is formulated on function spaces, but training data are typically available only through finite-dimensional representations. In encoder--decoder architectures, a matrix-valued kernel on the encoded space induces an operator-valued kernel on the original function spaces, and the corresponding reproducing kernel Hilbert spaces are isometrically isomorphic. As the input and output resolutions increase, the induced kernels converge to a limiting kernel, in the sense of operator-norm
convergence of their associated integral operators, allowing regularity assumptions to be stated independently of the encoding resolution. 
For regularized stochastic gradient descent, we establish upper bounds
for decreasing and fixed step sizes, separating the encoding and
regularization terms from optimization terms of order \(t^{-\theta}\)
and \(T^{-\theta'}\), respectively, for any \(\theta,\theta'\in(0,1)\).
We further prove lower bounds showing that these encoding-induced terms are generally unavoidable. The analysis is further extended to encoder--decoder neural networks through the limiting neural tangent kernel (NTK), yielding error bounds with an additional finite-width term and polynomial parameter and sample complexity guarantees when the encoding errors decay algebraically. The framework covers matrix-valued kernels constructed from radial and dot product kernels, NTKs arising from wide encoder--decoder neural networks, and encoder--decoder pairs based on Fourier, Legendre polynomial, wavelet, PCA, or pointwise sampling representations.



\end{abstract}


\section{Introduction}

Operator learning \cite{boulle2024mathematical,subedi2025operator} provides a natural framework for learning maps between function spaces from data. Such operators arise, for example, in partial differential equation models \cite{2021Fourier}, inverse problems \cite{nelsen2025operator}, dynamical systems \cite{wang2023long}, optimal control problems \cite{hwang2022solving}, and statistical problems involving functional data \cite{shi2025nonlinear}.  The target operator is typically defined intrinsically, independently of any finite-dimensional representation. 
In practice, however, the input and output functions are often observed, stored, and processed only through point samples, truncated basis coefficients, or other finite-resolution representations. 
These are not merely implementation details but can shape the induced model class and enter the error estimates. This paper investigates the operator classes induced on the original function spaces through finite representations and develops rigorous error estimates that explicitly quantify their effects.

Finite-resolution representations enter operator learning in two natural ways. The first begins with a model formulated as a mapping between infinite-dimensional function spaces and then examines its finite-resolution realization. Fourier neural operators (FNOs) \cite{2021Fourier,kovachki2023neural} provide a representative example. 
They are formulated as mappings between function spaces, whereas their numerical implementation relies on function values on grids, discrete Fourier transforms, and spectral truncation. The implemented network may therefore be viewed as a finite-resolution realization of an architecture defined on function spaces.
Recent work has made this perspective quantitative by analyzing discretization errors and aliasing effects in finite grid FNOs \cite{lanthaler2024discretization}, and by decomposing statistical, truncation, and discretization errors from a learning theoretic perspective \cite{Subedi2024ControllingSD}. These results quantify how grid discretization, spectral truncation, and sampling affect the finite-resolution realization of the underlying function space model. This viewpoint treats finite resolution mainly as an implementation of a model already specified at the function space level.

A complementary perspective starts from finite-dimensional representations and studies the operator models they induce on the original function spaces. This viewpoint is naturally realized by encoder--decoder architectures
\cite{bhattacharya2021model,gupta2021multiwavelet,lu2021learning,fanaskov2023spectral,song2023approximation,tripura2023wavelet,batlle2024kernel,nelsen2024operator,liu2024deep,cheng2026learning,yang2026efficient}. Given an input encoder $\me_1^{d_1}$ mapping into \(\mathbb R^{d_1}\), a learned map $f:\mathbb R^{d_1}\to \mathbb R^{d_2}$, and an output decoder $\md_2^{d_2}$ mapping into the output function space, the induced operator takes the form
\[
    u \longmapsto \md_2^{d_2} f(\me_1^{d_1}u),
\]
where $d_1$ and $d_2$ denote the input and output encoding dimensions, respectively. Representative encoder--decoder constructions include principal component analysis (PCA) representations \cite{bhattacharya2021model}, pointwise sampling combined with minimum-norm kernel interpolation \cite{batlle2024kernel}, and expansions in Fourier \cite{fanaskov2023spectral}, Legendre polynomial \cite{song2023approximation}, or wavelet \cite{gupta2021multiwavelet} bases. 
Through composition with the input encoder and output decoder, a model class of maps from $\mathbb R^{d_1}$ to $\mathbb R^{d_2}$ induces an operator class on the original function spaces. Varying $d_1$ and $d_2$ then yields a family of such induced operator classes.

Encoder--decoder architectures have been studied from the perspectives of approximation theory, complexity analysis, and statistical learning. In approximation theory, universal approximation results have been obtained for continuous operators under such architectures \cite{godeke2025new}. Within the framework of deep operator networks \cite{lu2021learning}, algebraic approximation rates have been proved for classes of Lipschitz operators \cite{2026Deep}. These approximation rates, however, rely on super expressive activations or nonstandard neural network architectures, and these highly expressive constructions can entail substantial parameter encoding costs \cite{lanthaler2024operator,lanthaler2025parametric}. Regarding complexity, the analysis of PCA-Net shows that, for general Lipschitz or Hölder-smooth
operator classes, the number of nonzero network parameters required to achieve accuracy $\epsilon$ cannot in general be bounded polynomially in $\epsilon^{-1}$, although this obstruction can be avoided for operators arising from certain PDEs \cite{lanthaler2023operator}. From a statistical learning perspective, \cite{liu2024deep} established non-asymptotic generalization error bounds 
for neural network estimators obtained by empirical risk minimization in the encoder--decoder framework
when the target operator is Lipschitz continuous or exhibits a low dimensional structure. More recently, \cite{cheng2026learning} obtained improved approximation and generalization bounds for Fr\'echet differentiable operators, though the required network size and sample size still do not, in general, admit polynomial bounds in $\epsilon^{-1}$.

Despite these advances, existing studies typically focus on the approximation, generalization, and complexity properties of encoder--decoder architectures for prescribed classes of target operators. This paper instead starts from a model class on encoded finite-dimensional representation spaces, characterizes the operator class induced on the original function spaces through the encoder and decoder, and studies its limiting behavior as the input and output resolutions are refined.

To make this viewpoint precise, we work within the reproducing kernel Hilbert space (RKHS) framework
\cite{christmann2008support,berlinet2011reproducing,micchelli2005learning,carmeli2006vector,carmeli2010vector}.
The learned map $f:\mathbb R^{d_1}\to \mathbb R^{d_2}$ between encoded spaces is taken from the vector-valued RKHS $\mathcal H_{\widetilde K}$ associated with a matrix-valued kernel $\widetilde K$ on the encoded input space. This includes kernels constructed from radial or dot product kernels, as well as neural tangent kernels (NTKs) arising when $f$ is implemented by a wide neural network. The input and output encoder--decoder pairs may be based on Fourier, Legendre polynomial, or wavelet expansions, PCA representations, or pointwise sampling combined with minimum-norm kernel interpolation. This setting goes beyond the standard fixed-kernel framework: the induced kernel, the corresponding RKHS, and the associated integral operator all vary with the encoding resolution. We therefore study the family of operator-valued kernels induced by the encoder--decoder mechanism and its limiting behavior as the input and output resolutions are refined. 

The main contributions of this paper can be summarized as follows.

We identify the operator-valued kernel naturally induced by an encoder--decoder representation. Given a matrix-valued kernel $\widetilde K$ on the encoded input space, the input encoder and the output decoder induce an operator-valued kernel $K$ on the original function spaces. In this formulation, the dependence of the lifted model on the encoder and decoder is incorporated into $K$. We show that the RKHS $\mathcal H_K$ associated with $K$ is isometrically identified with the RKHS $\mathcal H_{\widetilde K}$ through the lifting map. Under this identification, kernel ridge regression with $\widetilde K$ on encoded vector data is equivalent, after lifting, to operator-valued kernel ridge regression with $K$ on the original function spaces. The same correspondence carries gradient descent and stochastic gradient descent (SGD) iterations in the encoded space to recursions on the original function spaces, allowing the convergence analysis to be carried out directly there.

Building on this kernel lifting, we characterize the limiting behavior of the operator-valued kernels induced by encoder--decoder representations as the input and output resolutions are refined. We give conditions under which these induced kernels converge to a limiting kernel, in the sense that their associated integral operators converge to the integral operator generated by the limiting kernel. The finite-resolution operator classes and their limiting counterpart are described by the RKHSs associated with the induced kernels and the limiting kernel, respectively. This limiting-kernel viewpoint provides a resolution-independent reference for formulating the assumptions used in the subsequent convergence analysis.

Within this framework, we analyze SGD for kernel-based learning with encoder--decoder architectures. We establish error bounds both in the online setting with decreasing step sizes and in the finite-horizon setting with fixed step sizes. These bounds are not direct consequences of existing analyses of operator-valued kernel learning \cite{guo2023capacity,li2024towards,meunier2024optimal,shi2024learning,yang2025learning,yang2025kernel}, which are typically formulated for a fixed operator-valued kernel on the original function spaces and do not explicitly track the errors induced by encoding and decoding. In contrast, our bounds exhibit the three-term decomposition
\[
\mathrm{error}\ \lesssim\ \text{input-side error}
+\text{output-encoding error}
+C_{\mathrm{sgd}} t^{-\theta},
\qquad \theta\in(0,1).
\]
The displayed optimization term corresponds to decreasing step sizes.
For fixed step sizes, it is
$\widetilde C_{\mathrm{sgd}}T^{-\theta'}$; see Theorems~\ref{Thm1} and~\ref{Thm2}. The first term is governed by the kernel discrepancy and the regularization parameter, while the second term is the output-encoding error. After these representation and regularization effects are separated from the optimization error, the iteration-dependent term decays as $t^{-\theta}$ for any $\theta<1$, where the decay exponent \(\theta\) can be chosen arbitrarily close to \(1\).
This decomposition contributes not only to operator learning but also to kernel learning beyond the fixed-kernel setting, by providing error bounds for a family of resolution-dependent kernels induced by finite representations. We further prove lower bounds showing that, when $r\le1$, the first two terms cannot be improved in general.


For encoder--decoder neural networks, we introduce a limiting NTK and characterize its main properties. 
For both step-size regimes, we prove a parameter deviation bound
of order $\lambda^{-1}$, uniform over the training horizon,
under the width condition $M\gtrsim\lambda^{-4}$;
see Propositions~\ref{prop13} and~\ref{prop:constant-stability}. 
This leads to error bounds that separate the finite-width error of order $\lambda^{-6}/M$, the errors induced by finite-resolution representations, and the optimization error. In particular, when the encoding errors decay algebraically with the encoding dimensions, our bounds imply that the number of nonzero network parameters and the sample size required to achieve accuracy $\epsilon$ grow only polynomially in $\epsilon^{-1}$. Consequently, in this regime, encoder--decoder neural networks admit polynomial complexity guarantees, whereas such polynomial bounds in $\epsilon^{-1}$ are generally unavailable for classes of Fr\'echet differentiable operators.


The rest of the paper is organized as follows. Section \ref{Sec: Mathematical Preliminaries} collects the mathematical preliminaries. Section \ref{Section: Kernel Lifting} constructs the operator-valued kernels and RKHSs induced on the original function spaces, and introduces the corresponding limiting kernels. Section \ref{section 2.1} develops the convergence analysis for kernel-based encoder--decoder learning, including online and finite-horizon SGD upper and lower bounds. Section \ref{Section examples} presents concrete kernel families and encoder--decoder constructions covered by our framework. Section \ref{section 2.2} extends the analysis to encoder--decoder neural networks in the NTK regime. 
Section \ref{Sec 7} provides the main proofs of the kernel SGD results, while the proofs of the NTK results are collected in Appendix \ref{Sec 8}.

\section{Mathematical Preliminaries} \label{Sec: Mathematical Preliminaries}

Suppose that $\mathcal{U}$ and $\mathcal{V}$ are separable Hilbert spaces,
equipped with inner products $\langle\cdot,\cdot\rangle_{\mathcal{U}}$
and $\langle\cdot,\cdot\rangle_{\mathcal{V}}$, and corresponding norms
$\|\cdot\|_{\mathcal{U}}$ and $\|\cdot\|_{\mathcal{V}}$. 
Let $\rho$ be a Borel probability measure on $\mathcal{U}\times\mathcal{V}$, and let
$(u,v)\sim\rho$. We denote by $\rho_u$ the marginal distribution of $\rho$ on $\mathcal{U}$. Write
    $L^2(\mathcal{U},\rho_{u};\mv)$ for the Lebesgue-Bochner space \cite[Chapter 1]{hytonen2016analysis} consisting of (equivalence classes of) strongly measurable operators $\mathcal{G}:\mathcal{U}\to\mv$ such that the Bochner norm
    \[
    \l\|\mathcal{G}\r\|_{\rho_{u}}:=\l(\int_{\mathcal{U}}\l\|\mathcal{G}(u)\r\|_{\mv}^2\mathrm{d}\rho_{u}(u)\r)^{1/2}
    \] is finite. In the scalar-valued case $\mathcal{V}=\mathbb{R}$, we simply write
$L^2(\mathcal{U},\rho_u)$.
In this paper, we assume that the data are generated according to the regression model
\begin{equation} \label{regression}
    v=\mathcal{G}^{\dagger}(u)+\epsilon,
\end{equation}
where the target operator $\mathcal{G}^\dagger:\mathcal{U}\to\mathcal{V}$ is defined as the conditional expectation
\[
\mathcal{G}^\dagger(u)
:= \mathbb{E}[v |u]
\in L^2(\mathcal{U},\rho_u;\mathcal{V}),
\]
and $\epsilon$ is a zero-mean noise term satisfying $\mathbb{E}\big[\|\epsilon\|_{\mathcal{V}}^2\big]\le \sigma^2$. 

Throughout the paper, we write
\[
z:=(u,v)\in\mathcal U\times\mathcal V,
\qquad
z_t:=(u_t,v_t),
\]
where \(\{z_t\}_{t\ge1}\) is an i.i.d. sequence distributed according to
\(\rho\). Given \(t\ge1\), let \(z^t:=(z_1,\ldots,z_t)\), and let
\(\mathbb E_{z^t}\) be the expectation with respect to the law of \(z^t\). If
\(\Phi:\mathcal U\to\mathcal V\) is measurable and \(\mu\) is a probability
measure on \(\mathcal U\), then \(\Phi_{\#}\mu\) denotes the push-forward of
\(\mu\) under \(\Phi\). On Euclidean spaces, \(\|\cdot\|_2\) is the standard
Euclidean norm. Given a measurable function \(F:\mathcal U\to\mathbb R\), we
write \(\operatorname*{ess\,sup}_{u\sim\rho_u}F(u)\) for its essential supremum
with respect to \(\rho_u\). Moreover, given \(R>0\), define
\[
\mathcal B_R(\mathcal U)
:=
\{u\in\mathcal U:\|u\|_{\mathcal U}\le R\},
\]
the closed ball of radius \(R\) in \(\mathcal U\).

\paragraph{Operator-theoretic notions.}
Let $A:H_1\to H_2$ be a linear operator between two Hilbert spaces $(H_1,\langle\cdot,\cdot\rangle_{H_1},\|\cdot\|_{H_1})$ and $(H_2,\langle\cdot,\cdot\rangle_{H_2},\|\cdot\|_{H_2})$. The collection of bounded linear operators from $H_1$ to $H_2$, equipped with the operator norm $\|A\|= \sup _{\| f\| _{H_1}\leq 1}\| Af\| _{H_2}$, forms a Banach space, denoted by $\mathcal{L}( H_1 , H_2)$; when $H_1=H_2$, we write $\mathcal{L}(H_1)$ for brevity. Recall that an operator $A\in\mathcal{L}(H_1,H_2)$ is said to be Hilbert-Schmidt if $\sum_{k\geq1}\l\|Ae_i\r\|_{H_2}^2<\infty$ for some (equivalently, any) orthonormal basis $\{e_k\}_{k\geq1}$ of $H_1$. The class of Hilbert-Schmidt operators from $H_1$ to $H_2$ constitutes a Hilbert space when endowed with the inner product $\langle A,B\rangle_{\mathrm{HS}}=\sum_{k\geq1}\langle Ae_k,Be_k\rangle_{H_2}$ and the associated Hilbert–Schmidt norm $\|\cdot\|_{\mathrm{HS}}$. We denote this space by
    $S_2(H_1,H_2)$. For any $A\in \mathcal{L}(H_1,H_2)$, the adjoint operator $A^*\in\mathcal{L}(H_2,H_1)$ is defined as the unique operator satisfying $\langle Af,f^{\prime}\rangle_{H_2}=\langle f,A^*f^{\prime}\rangle_{H_1}$ for all $f\in H_1$ and $f^\prime\in H_2$. Moreover, there holds $\|A\|=\|A^*\|$. An operator $A\in\mathcal{L}(H_1)$ is called self-adjoint if $A^*=A$, and positive if it is self-adjoint and satisfies $\langle Af,f\rangle_{H_1}\geq0$ for all $f\in H_1$.

\paragraph{Vector-valued reproducing kernel Hilbert spaces.}
Let $K:\mathcal{U}\times\mathcal{U}\to\mathcal{L}(\mathcal{V})$ be an operator-valued kernel satisfying:
\begin{enumerate}[(1)]
    \item \emph{Hermitian symmetry}: for all $u,u'\in\mathcal{U}$,
    $
    K(u,u') = K(u',u)^*;
    $
    \item \emph{Positive semi-definiteness}: for any $n\in\mathbb{N}$,
    any $\{u_i\}_{i=1}^n\subset\mathcal{U}$, and any $\{v_i\}_{i=1}^n\subset\mathcal{V}$,
    \[
        \sum_{i,j=1}^n \langle K(u_i,u_j)v_j, v_i\rangle_{\mathcal{V}} \ge 0 .
    \]
\end{enumerate}
Such a kernel $K$ uniquely induces a reproducing kernel Hilbert space (RKHS)
$\mathcal{H}_K$ of $\mathcal{V}$-valued functions on $\mathcal{U}$
\cite{micchelli2005learning,carmeli2006vector,carmeli2010vector},
defined as the completion of the linear span
\[
    \mathcal{H}_K := \overline{\mathrm{span}}
    \left\{ K(\cdot,u)v \mid u\in\mathcal{U},\ v\in\mathcal{V} \right\},
\]
equipped with an inner product $\langle\cdot,\cdot\rangle_K$
satisfying the reproducing property
\[
    \langle K(\cdot,u)v, K(\cdot,u')v' \rangle_K
    = \langle K(u',u)v, v' \rangle_{\mathcal{V}},
    \qquad
    \langle \mathcal{G}, K(\cdot,u)v \rangle_K
    = \langle \mathcal{G}(u), v \rangle_{\mathcal{V}},
\]
for all $\mathcal{G}\in\mathcal{H}_K$, $u,u'\in\mathcal{U}$, and $v,v'\in\mathcal{V}$. Throughout the paper, for a (scalar-, matrix-, or operator-valued) kernel $K$, we denote by
$\H_K$ the induced RKHS and write $\langle\cdot,\cdot\rangle_K$ and $\|\cdot\|_K$
for its inner product and norm, respectively.

When $\mathcal{V}$ is finite-dimensional, $K$ reduces to a matrix-valued kernel.
In the scalar-valued case $\mathcal{V}=\mathbb{R}$, $K$ coincides with a
scalar-valued kernel, and the reproducing property simplifies to
\[
    \langle f, K(\cdot,u) \rangle_K = f(u),
    \qquad f\in\mathcal{H}_K,\ u\in\mathcal{U}.
\]
Associated with $K$ and a probability measure $\rho_u$ on $\mathcal{U}$, assume that
\[
    \int_{\mathcal U}\int_{\mathcal U}
    \|K(u,u')\|^2\,
    \mathrm d\rho_u(u)\,\mathrm d\rho_u(u')<\infty .
\]
We define the integral operator
$L_K : L^2(\mathcal{U},\rho_u;\mathcal{V})
\to L^2(\mathcal{U},\rho_u;\mathcal{V})$ by
\[
    (L_K \mathcal{G})(\cdot)
    := \int_{\mathcal{U}} K(\cdot,u)\,\mathcal{G}(u)\,\mathrm{d}\rho_u(u).
\]
The operator $L_K$ is bounded, self-adjoint, and positive.

For the classical theory of scalar-valued RKHSs and their associated integral
operators, we refer to \cite{christmann2008support,berlinet2011reproducing}.
Further background on operator-valued kernels can be found in
\cite{micchelli2005learning,carmeli2006vector,carmeli2010vector}.

\paragraph{Tensor product of Hilbert spaces.} Let $H_1$ and $H_2$ be real Hilbert spaces. The algebraic tensor product 
$H_1 \odot H_2$ is equipped with the inner product
\[
\langle u \otimes v,\; u' \otimes v' \rangle
= \langle u,u' \rangle_{H_1}\,\langle v,v' \rangle_{H_2},
\]
extended by bilinearity to all finite sums. The Hilbert tensor product 
$H_1 \otimes H_2$ is then defined as the completion of $H_1 \odot H_2$, so that 
every element can be approximated by finite linear combinations of simple 
tensors $u \otimes v$. If $\{e_i\}$ and $\{f_j\}$ are orthonormal bases of 
$H_1$ and $H_2$, then $\{e_i \otimes f_j\}$ forms an orthonormal basis of 
$H_1 \otimes H_2$. Moreover, simple tensors $u \otimes v$ can be naturally regarded as 
rank-one operators, yielding the canonical isometric identification
\[
H_1 \otimes H_2 \;\cong\; S_2(H_2,H_1),
\]
via $u \otimes v \mapsto \langle \cdot,v\rangle_{H_2} u$.
Furthermore, one has the standard identification 
\[
L^2(\mathcal{U},\rho_u;\mathcal{V}) \;\cong\; L^2(\mathcal{U},\rho_u)\otimes \mathcal{V},
\]
for measurable space $(\mathcal{U},\rho_u)$ and real Hilbert space $\mathcal{V}$,
where the correspondence is given by
\begin{equation} \label{temp27}
    \varphi(\cdot)\,v \;\longmapsto\; \varphi \otimes v,
\quad \varphi \in L^2(\mathcal{U},\rho_u),\; v \in \mathcal{V},
\end{equation}
see \cite[Section 12]{aubin2011applied} for details. Given bounded linear operators $A:H_1\to H_1$ and $B:H_2\to H_2$, 
we define their tensor product $A\otimes B$ on simple tensors by
\[
(A\otimes B)(u\otimes v) := (Au)\otimes (Bv), 
\quad u\in H_1,\; v\in H_2.
\]
This prescription extends by linearity to the algebraic tensor product 
$H_1\odot H_2$ and by continuity to the Hilbert tensor product 
$H_1\otimes H_2$. The operator $A\otimes B$ is bounded with 
$\|A\otimes B\|=\|A\|\,\|B\|$, and the mapping 
$(A,B)\mapsto A\otimes B$ is bilinear. 
Moreover, $(A_1\otimes B_1)(A_2\otimes B_2)
    =(A_1A_2)\otimes (B_1B_2)$.

\paragraph{Tensor-product representation of integral operators.} Let $K:\mathcal{U}\times\mathcal{U}\to\mathcal{L}(\mathcal{V})$ be an
operator-valued kernel of the separable form
\[
K(u,u') = k(u,u')\,P,
\]
where $k:\mathcal{U}\times\mathcal{U}\to\mathbb{R}$ is a scalar-valued Mercer
kernel satisfying $\int_{\mathcal{U}}k(u,u)\mathrm{d}\rho_u(u)<\infty$
(so that the associated integral operator $L_k$ is trace-class \cite[Theorem 4.27]{christmann2008support}, thus compact) and
$P:\mathcal{V}\to\mathcal{V}$ is an orthogonal projection.
Then the corresponding integral operator
$L_K:L^2(\mathcal{U},\rho_u;\mathcal{V})\to L^2(\mathcal{U},\rho_u;\mathcal{V})$
admits the tensor-product representation
\[
L_K = L_k \otimes P.
\]

Indeed, since $L_k$ is compact, self-adjoint, and positive, it admits the
spectral decomposition
\[
L_k = \sum_{i\ge1}\lambda_i
\langle \cdot,\varphi_i\rangle_{\rho_u}\,\varphi_i.
\]
Let $\{v_j\}_{j=1}^m$ ($1\leq m\le\infty$) be an orthonormal basis of
$\mathrm{ran}(P)\subset\mathcal{V}$.
Then a direct computation shows that
\[
L_K\big(\varphi_i(\cdot)\,v_j\big)
= \lambda_i\,\varphi_i(\cdot)\,v_j
= (L_k\varphi_i)(\cdot)\,v_j,
\qquad i\ge1,\; j=1,\dots,m,
\]
which yields $L_K=L_k\otimes P$.
In particular, if $P=I_{\mathcal{V}}$, then $L_K=L_k\otimes I_{\mathcal{V}}$.

Moreover, in the strong operator topology,
\[
L_K = \sum_{i\ge1}\sum_{j=1}^m
\lambda_i\,
\langle \cdot,\varphi_i(\cdot)\,v_j\rangle_{\rho_u}
\,\varphi_i(\cdot)\,v_j,
\]
and consequently \footnote{
More generally, if the operator-valued kernel is of the form $K(u,u') = k(u,u')\,A$, with $A$ a bounded self-adjoint positive operator
on $\mathcal V$ (not necessarily compact),
a direct computation shows that
$L_K = L_k \otimes A$.
Since $L_k \otimes I_{\mathcal V}$ and $I \otimes A$ are bounded self-adjoint
operators that commute, the joint spectral theorem and the associated
joint functional calculus apply.
Consequently, the fractional powers defined via the spectral calculus satisfy
\[
L_K^{\,r} = (L_k \otimes A)^r = L_k^{\,r} \otimes A^{\,r},
\qquad r > 0.
\]
See, for instance, Schmüdgen~\cite[Section~5.5]{schmudgen2012unbounded}.
} 
\[
L_K^r = L_k^r \otimes P,\qquad r>0.
\]

\section{Kernel Lifting and Limiting Kernels} \label{Section: Kernel Lifting}
In applications, the functions $u$ and $v$ in regression model \eqref{regression} are accessed through finite-dimensional encodings rather than as elements of the underlying function spaces. Specifically, we consider observations of the form $\l(\mathcal{E}_1^{d_1}u_t,\mathcal{E}_2^{d_2}v_t\r)$, where
\[
    \mathcal{E}_1^{d_1}:\mathcal{U}\to\br^{d_1},\qquad\mathcal{E}_2^{d_2}:\mathcal{V}\to\br^{d_2}
\]
are prescribed encoding operators on the input and output spaces. Our objective is to approximate the target operator $\mathcal{G}^{\dagger}$ using only these encoded observations. On the output side, we assume that the encoding admits an associated decoder $\mathcal{D}_2^{d_2}=\l(\mathcal{E}_2^{d_2}\r)^*$ satisfying
\[
\mathcal{E}_2^{d_2}\circ\mathcal{D}_2^{d_2}=I_{\br^{d_2}}, \qquad\mathcal{D}_2^{d_2}\circ\mathcal{E}_2^{d_2}={P_2^{d_2}},
\]
where ${P_2^{d_2}}$ denotes the orthogonal projection onto a $d_2$-dimensional subspace of $\mv$. 


\paragraph{From matrix-valued kernels to operator-valued kernels.}

We formalize the encoder--decoder learning procedure by showing how a matrix-valued kernel on the encoded input space induces an operator-valued kernel on the original function spaces. This identifies the operator class associated with each resolution and allows us to analyze the resulting family as the resolution is refined.

The construction is motivated by the following procedure: one learns a map between encoded finite-dimensional spaces using a matrix-valued kernel and then lifts the estimator to the original function spaces through the encoder--decoder architecture. The lifted estimator is then formulated through an operator-valued kernel.


For concreteness, we consider a diagonal matrix-valued kernel, defined as a scalar-valued
Mercer kernel multiplied by the $d_2$-dimensional identity matrix,
\begin{equation}\label{MVK} \tag{MVK}
    \widetilde{K}(x,x') := k(x,x') I_{d_2},
\end{equation}
where $k$ is a scalar-valued kernel.
This construction induces a vector-valued RKHS and simplifies the exposition. The construction extends directly to general matrix-valued kernels.

Learning with $\widetilde{K}$ yields an estimator $f:\br^{d_1}\to\br^{d_2}$. Composing $f$ with the encoder--decoder pair produces an operator estimate
\[
\mathcal{G}(u)=\mathcal{D}_2^{d_2}f\l(\mathcal{E}_1^{d_1}u\r).
\]
As $f$ varies in $\mathcal H_{\widetilde K}$, these estimates form the induced operator class.

Next, we present two regularized least-squares formulations of this learning problem. Given encoded samples $x_t=\mathcal{E}_1^{d_1}u_t$ and $y_t=\mathcal{E}_2^{d_2}v_t$, we first consider the matrix-valued kernel ridge regression problem
\begin{equation}\label{MVK-KRR} \tag{MVK-KRR}
    \hat{f}\in\arg\min_{f\in\H_{\widetilde{K}}}\l\{\frac{1}{n}\sum_{t=1}^n\|f(x_t)-y_t\|_2^2+\lambda\l\|f\r\|^2_{\widetilde{K}}\r\},
\end{equation}
where $\H_{\widetilde{K}}$ denotes the vector-valued RKHS induced by $\widetilde{K}$. The corresponding operator estimator is obtained by lifting $\hat{f}$ through the encoder--decoder maps.
Alternatively, exploiting the encoder--decoder architecture, one may formulate
the learning problem directly on the original function spaces. Define the operator-valued kernel $K:\mathcal{U}\times\mathcal{U}\to\mathcal{L}\l(\mathcal{V}\r)$ by
\begin{equation}\label{OVK} \tag{OVK}
    K(u,u'):=k\l(\mathcal{E}^{d_1}_1u,\mathcal{E}^{d_1}_1u'\r){P_2^{d_2}}.
\end{equation}
Let $\H_{K}$ be the associated $\mathcal{V}$-valued RKHS. We then consider the operator-valued kernel ridge regression problem
\begin{equation}\label{OVK-KRR} \tag{OVK-KRR}
    \hat{\mathcal{G}}\in\arg\min_{\mathcal{G}\in\H_{K}}\l\{\frac{1}{n}\sum_{t=1}^n\|\mathcal{G}(u_t)-v_t\|_{\mathcal{V}}^2+\lambda\l\|\mathcal{G}\r\|^2_{K}\r\}.
\end{equation}
For the isometric lifting result below, we additionally assume that the input encoder $\mathcal{E}^{d_1}_1$ is surjective. Under the encoder--decoder architecture, the two formulations are related
by an explicit isometric correspondence between the associated RKHSs.
\begin{proposition}[Equivalence via isometric lifting] \label{prop1.1}
    The following statements hold.
    \begin{enumerate}[(1)]
        \item (Isometric correspondence between RKHSs) The linear map
        \[
        T:\mathcal{H}_{\widetilde{K}}\to\mathcal{H}_K,\quad(Tf)(u):=\mathcal{D}_2^{d_2}f(\mathcal{E}^{d_1}_1u)
        \]
        is well-defined and is an isometric isomorphism onto $\mathcal{H}_K$.
        \item (Equivalence of regularized objectives)
        For any $f\in\mathcal{H}_{\widetilde{K}}$, letting $\mg=Tf$, one has
        \begin{equation*}
            \frac{1}{n}\sum_{t=1}^n\|\mg(u_t)-v_t\|_{\mathcal{V}}^2+\lambda\|\mg\|_{K}^2=\frac{1}{n}\sum_{t=1}^n\|f(x_t)-y_t\|_{2}^2+\lambda\|f\|_{{\widetilde{K}}}^2+\frac{1}{n}\sum_{t=1}^n\|(I-{P_2^{d_2}})v_t\|_{\mathcal{V}}^2.
        \end{equation*}
        Consequently, the minimizers of \eqref{MVK-KRR} and \eqref{OVK-KRR} are related by
        \[
        \widehat{\mg}=T\hat{f},\quad\mathrm{~i.e.,~}\quad\widehat{\mg}(u)=\mathcal{D}^{d_2}_2\hat{f}(\mathcal{E}^{d_1}_1u),\quad\forall u\in\mathcal{U},
        \]
        and the two problems differ only by an additive constant independent of the estimator.
    \end{enumerate}
\end{proposition}
The above equivalence extends to general matrix-valued kernels $\widetilde{K}$. In this case, the induced operator-valued kernel is given by  
\[
K(u,u')=\mathcal{D}_2^{d_2}\widetilde{K}\l(\mathcal{E}^{d_1}_1u,\mathcal{E}^{d_1}_1u'\r)\mathcal{E}_2^{d_2},
\]
and the proof proceeds analogously; see Appendix~\ref{Appendix 1}. 

We also record the reverse construction. Starting from an operator-valued kernel $K:\mathcal{U}\times\mathcal{U}\to\mathcal{L}(\mathcal{V})$, one can derive a matrix-valued kernel adapted to a chosen encoder--decoder
representation.  This construction requires an input-side decoder $\mathcal{D}_1^{d_1}$, which maps $\br^{d_1}$ back to $\mathcal{U}$.
The induced matrix-valued
kernel on the encoded spaces is given by
\[
\widetilde{K}(x,x')
:=\mathcal{E}_2^{d_2}
K\bigl(\mathcal{D}_1^{d_1}x,\mathcal{D}_1^{d_1}x'\bigr)
\mathcal{D}_2^{d_2}.
\]
By the same argument, the kernel ridge regression problem based on
$\widetilde{K}$ corresponds to an operator-valued kernel ridge regression
problem associated with the kernel
\[
P_2^{d_2}
K\bigl(\mathcal{D}_1^{d_1}\mathcal{E}_1^{d_1}u,
\mathcal{D}_1^{d_1}\mathcal{E}_1^{d_1}u'\bigr)
P_2^{d_2}.
\]
This reverse viewpoint shows that passing through finite representations
generally modifies the operator-valued kernel.
Universal approximation results give qualitative approximation guarantees for encoder--decoder architectures.  Under mild continuity and integrability assumptions on
the target operator, approximation in $L^2$ can be obtained
by combining the $L^2$-universality of scalar-valued kernels with the decay of the
encoding error as the encoding dimension increases.
Stronger qualitative approximation results, such as uniform approximation of
continuous operators on compact sets, require additional structural assumptions,
for instance the bounded approximation property
\cite[Chapter~7]{johnson2001handbook} or the encoder--decoder approximation
property introduced in \cite{godeke2025new}, together with $cc$-universal kernels
\cite{sriperumbudur2011universality}.

For quantitative approximation results based on RKHS methods, whether for functions or operators, it is common to impose a source condition that characterizes the regularity of the target relative to the kernel. Such conditions are typically formulated through fractional powers of the integral operator induced by the kernel, interpreted via functional calculus. Sharper rates often require additional spectral assumptions on this integral operator, such as trace or capacity conditions that quantify the effective dimension of the associated RKHS.

In the encoder--decoder setting, however, these kernel-based regularity assumptions become resolution dependent. At the encoded level, the integral operator associated with \eqref{MVK} depends on the push-forward distribution $(\mathcal{E}_1^{d_1})_{\#}\rho_u$. After lifting to the operator-valued kernel \eqref{OVK}, the kernel still varies with $d_1$, and so does the corresponding integral operator. Hence, source conditions imposed separately at each fixed resolution are not stable under refinement of the input encoding.

For quantitative analysis, regularity assumptions should be formulated in a way that is stable under refinement of the encoding. Motivated by Proposition~\ref{prop1.1}, we lift the finite-dimensional problem \eqref{MVK-KRR} to the operator-valued formulation \eqref{OVK-KRR} and study the limit as $d_1,d_2\to\infty$. 
We summarize the limiting-kernel viewpoint schematically as
\begin{equation*} 
\lim_{d_1,d_2\to\infty}
k\bigl(\mathcal{E}^{d_1}_1u,\mathcal{E}^{d_1}_1u'\bigr){P_2^{d_2}}
= k_\infty(u,u') I_{\mathcal V},
\end{equation*}
where $k_\infty:\mathcal{U}\times\mathcal{U}\to\mathbb{R}$ is a scalar-valued
kernel, referred to as the limiting scalar kernel.

For notational convenience, we introduce the encoder-induced scalar kernel
\[
k_{d_1}(u,u'):=k\l(\mathcal{E}^{d_1}_1u,\mathcal{E}^{d_1}_1u'\r),
\]
the finite-resolution operator kernel
\[
K_{d_1,d_2}:=k_{d_1}{P_2^{d_2}},
\]
the projected limiting kernel
\[
K_{\infty,d_2}:=k_\infty {P_2^{d_2}},
\]
and the limiting operator kernel
\[
K_{\infty,\infty}:=k_{\infty}I_{\mv}.
\]
These kernels are summarized in Table~\ref{tab:kernels}. To measure the effect of the input-side encoding on the induced kernel, we introduce the kernel discrepancy
\[
\Delta := \|L_{d_1}-L_\infty\|,
\]
where $L_{d_1}$ and $L_\infty$ are the integral operators induced by $k_{d_1}$ and $k_\infty$, respectively. 
For the quantitative analysis, we impose the following conditions
on the input kernel discrepancy and the output-projection error.


\begin{table}[tbp]
\centering
\renewcommand{\arraystretch}{1.25}
\begin{tabular}{ll@{\qquad}}
\hline
\textbf{Kernel} & \textbf{Definition} \\
\hline
Encoder-induced scalar kernel 
& $k_{d_1}:\mathcal{U}\times \mathcal{U}\to \mathbb{R}$ \\

Finite-resolution operator kernel
& $K_{d_1,d_2} = k_{d_1} {P_2^{d_2}}$ \\

Limiting scalar kernel 
& $k_{\infty}:\mathcal{U}\times \mathcal{U}\to \mathbb{R}$ \\

Projected limiting kernel 
& $K_{\infty,d_2} = k_{\infty} {P_2^{d_2}}$ \\

Limiting operator kernel 
& $K_{\infty,\infty} = k_{\infty} I_{\mv}$ \\

\hline
\end{tabular}
\caption{Summary of kernels used in the analysis.}
\label{tab:kernels}
\end{table}

\begin{condition}
Suppose that there exists $\kappa>0$ such that
\[
\operatorname*{ess\,sup}_{u\sim\rho_u} k_{d_1}(u,u)\le \kappa^2
\quad\text{for all sufficiently large } d_1,
\]
and that
\[
\Delta = \bigl\|L_{d_1}-L_{\infty}\bigr\|\;\longrightarrow\;0,
\qquad \text{as } d_1\to\infty.
\]
In addition, assume that the output-projection error vanishes, namely
\[
\mathbb{E}_{v\sim(\mathcal{G}^\dagger)_{\#}\rho_u}
\bigl[\|(P_2^{d_2}-I_{\mathcal V})v\|_{\mathcal V}^2\bigr]
\;\longrightarrow\;0,
\qquad \text{as } d_2\to\infty.
\]
\end{condition}


The above condition is satisfied for a broad class of kernels and encoder--decoder constructions. For kernels, it covers radial and dot product kernels, including Gaussian, Mat\'ern, inverse multiquadric, Cauchy, polynomial, and exponential dot product kernels. For encoders, it covers spectral truncations based on Fourier, Legendre, wavelet, or PCA expansions; see Section~\ref{Section examples}.


The limiting-kernel viewpoint also applies to other kernel limits. For instance, random-feature approximations give empirical kernels that converge to a population kernel as the number of sampled features increases. Neural tangent kernels provide another limiting kernel in the infinite-width regime. Their interaction with encoder--decoder representations is discussed in Section~\ref{section 2.2}.


\section{Error Analysis for Kernel-Based Encoder--Decoder Learning} \label{section 2.1}

\paragraph{SGD under encoder--decoder architectures.}
By the representer theorem, \eqref{MVK-KRR} can be reduced to a finite-dimensional problem. It can also be analyzed from a spectral regularization perspective. In this paper, we focus on SGD, which is natural for online learning and is consistent with the kernel gradient dynamics used later to relate the present RKHS analysis to the training dynamics of wide encoder--decoder neural networks.

We first formulate regularized SGD in the encoded space and then lift the
iterates through the encoder--decoder maps. Fix $\lambda>0$. For the diagonal
kernel \eqref{MVK}, consider the regularized single-sample objective
\[
\ell(f;(x,y))
:=\frac12\|f(x)-y\|_2^2+\frac{\lambda}{2}\|f\|_{\widetilde K}^2,
\qquad f\in\mathcal H_{\widetilde K}.
\]
Its gradient in $\mathcal H_{\widetilde K}$ is
\[
\nabla_f \ell(f;(x,y))
=
k(\cdot,x)\bigl(f(x)-y\bigr)+\lambda f .
\]
Given encoded samples
$x_t=\mathcal E_1^{d_1}u_t$ and $y_t=\mathcal E_2^{d_2}v_t$, the SGD iteration
is
\begin{equation*}
\begin{cases}
f_1 := 0,\\[2pt]
f_{t+1}
:= (1-\lambda\eta_t)f_t
-\eta_t\,k(\cdot,x_t)\bigl(f_t(x_t)-y_t\bigr),
\end{cases}
\end{equation*}
where $\eta_t>0$ is the step size at iteration $t$.

Define the lifted iterates
\[
\mathcal G_t:=\mathcal D_2^{d_2}\circ f_t\circ\mathcal E_1^{d_1}.
\]
Using $y_t=\mathcal E_2^{d_2}v_t$ and
$\mathcal D_2^{d_2}\mathcal E_2^{d_2}=P_2^{d_2}$, the corresponding recursion for
$\mathcal G_t$ is
\begin{equation}\label{SGD1}
\begin{cases}
\mathcal G_1 := 0,\\[2pt]
\mathcal G_{t+1}
:= (1-\lambda\eta_t)\mathcal G_t
-\eta_t\,k_{d_1}(\cdot,u_t)
\bigl(\mathcal G_t(u_t)-P_2^{d_2}v_t\bigr).
\end{cases}
\end{equation}
By the iteration \eqref{SGD1}, for all $u\in\mathcal{U}$,
$\mathcal G_t(u)\in\operatorname{ran}(P_2^{d_2})$, and hence \eqref{SGD1} can
equivalently be written as
\[
\mathcal G_{t+1}
=
(1-\lambda\eta_t)\mathcal G_t
-\eta_t\,k_{d_1}(\cdot,u_t)
P_2^{d_2}\bigl(\mathcal G_t(u_t)-v_t\bigr).
\]
This is the SGD recursion associated with the finite-resolution
operator-valued kernel $K_{d_1,d_2}=k_{d_1}P_2^{d_2}$.

Under the kernel discrepancy $\Delta$, the SGD dynamics driven by $K_{d_1,d_2}$ 
can first be compared to the recursion associated with $K_{\infty,d_2}$, and 
subsequently, as the output-encoding dimension $d_2$ increases, to 
$K_{\infty,\infty}$. The regularization parameter $\lambda$ is chosen to 
balance the effect of the kernel discrepancy with the regularization bias.

The recursion \eqref{SGD1} is obtained by computing the lift of the encoded-space 
SGD iterates through the encoder--decoder maps. For encoders satisfying the 
conditions in Section \ref{Section: Kernel Lifting}, this lift is isometric; however, the 
derivation here only requires that $\mathcal D_2^{d_2}\mathcal E_2^{d_2} = P_2^{d_2},$
which allows the inclusion of pointwise sampling combined with minimum-norm kernel interpolation in the analysis.


We denote by
$L_{d_1}$, $L_{d_1,d_2}$, $L_{\infty}$, $L_{\infty,d_2}$, and
$L_{\infty,\infty}$
the integral operators associated with
$k_{d_1}$, $K_{d_1,d_2}$, $k_\infty$, $K_{\infty,d_2}$, and
$K_{\infty,\infty}$, respectively.
Here $L_{d_1}$ and $L_\infty$ act on $L^2(\mathcal U,\rho_u)$, while
$L_{d_1,d_2}$, $L_{\infty,d_2}$, and $L_{\infty,\infty}$ act on
$L^2(\mathcal U,\rho_u;\mathcal V)$. In particular,
\[
    L_{d_1,d_2}=L_{d_1}\otimes P_2^{d_2},\qquad
    L_{\infty,d_2}=L_\infty\otimes P_2^{d_2},\qquad
    L_{\infty,\infty}=L_\infty\otimes I_{\mathcal V}.
\]
Throughout the paper, we write $a\lesssim b$ if $a\le Cb$, where $C$ is
independent of $\Delta$, $\lambda$, $t$, $T$, and the step-size scale $\eta_1$ or $\eta$. The constant may depend on the fixed exponents $r$, $\theta$,
and $\theta'$. 
When Section~\ref{section 2.2} is invoked, $C$ is
also independent of the network width $M$ and the confidence level $\delta$.

The limiting kernel $K_{\infty,\infty}$ serves as the reference for the
regularity assumption. The discrepancy with finite-resolution kernels is quantified
by the input kernel discrepancy and the output-encoding error. We formulate
the source condition with respect to the associated integral operator $L_{\infty,\infty}$,
which does not depend on the finite encoding dimensions $d_1$ and $d_2$.
\begin{assumption} \label{source condition}
    We assume that the target operator $\mathcal{G}^{\dagger}$ satisfies
    \begin{equation*}
        \mathcal{G}^{\dagger} = L_{\infty,\infty}^r\mathcal{G}^{\mathrm{src}} = \l(L_{\infty}^r\otimes I_{\mv}\r)\mathcal{G}^{\mathrm{src}}
    \end{equation*}
for some $r>0$ and $\mathcal{G}^{\mathrm{src}}\in L^2(\mathcal U,\rho_u;\mathcal V)$.
\end{assumption}
This assumption quantifies the regularity of $\mathcal G^\dagger$ through
$L_{\infty,\infty}$ and is independent of the encoding dimensions $d_1$ and
$d_2$. Such source-type assumptions arise naturally in regularization theory for
kernel-based learning \cite{caponnetto2007optimal,smale2007learning,ying2008online,yang2025learning,yang2025kernel}.

As shown in \cite[Remark~2]{yang2025kernel}, Assumption \ref{source condition} is
equivalent to $\mathcal{G}^{\dagger}\in
    \bigl[\mathcal H_{K_{\infty,\infty}}\bigr]^{2r},$ where \(\bigl[\mathcal H_{K_{\infty,\infty}}\bigr]^{2r}\) denotes the
vector-valued interpolation space associated with \(K_{\infty,\infty}\).
More explicitly, for \(\alpha\geq 0\), define
\[
\bigl[\mathcal H_{K_{\infty,\infty}}\bigr]^{\alpha}
:=
\Psi\!\left(
S_2\bigl(\bigl[\mathcal H_{k_\infty}\bigr]^{\alpha},\mathcal V\bigr)
\right)
=
\bigl\{
\mathcal G:\mathcal G=\Psi(A),\;
A\in S_2\bigl(\bigl[\mathcal H_{k_\infty}\bigr]^{\alpha},\mathcal V\bigr)
\bigr\},
\]
where \(\bigl[\mathcal H_{k_\infty}\bigr]^{\alpha}\) denotes the scalar-valued
interpolation space associated with \(k_\infty\), \(S_2(\cdot,\cdot)\) denotes
the Hilbert--Schmidt class, and \(\Psi\) is the canonical isometric
identification induced by \eqref{temp27}.

\subsection{Upper Bounds on the Prediction Error}
\label{subsection: upper bound}
We consider an online learning regime in which data arrive sequentially and the
total number of samples is unknown and potentially unbounded. The estimator is
updated iteratively, and convergence must be maintained despite the persistent
stochastic noise introduced by newly arriving observations. This naturally calls
for a decreasing step-size schedule. In particular, polynomially decaying step
sizes provide an effective way to gradually mitigate the effect of noise while
retaining sufficient learning capacity in the early stages of iteration, without
relying on a fixed horizon or averaging of iterates.

\begin{theorem} \label{Thm1}
Suppose that Assumption \ref{source condition} holds with $r>0$ and 
$\mathcal{G}^{\mathrm{src}}\in L^2(\mathcal{U},\rho_{u};\mv)$. 
Let $\{\mathcal{G}_t\}_{t\geq1}$ be defined through \eqref{SGD1} with decreasing step sizes
\(\eta_t = \eta_1 t^{-\theta}\) for $0<\theta<1$ and a regularization parameter $\lambda>0$. 
Assume further that $\lambda\le1$. 

Then there exists a constant $\bar\eta>0$ such that for all $0<\eta_1\le \bar\eta$, 
\begin{equation*}
\begin{aligned}
\mathbb{E}_{z^t}\Big[\|\mathcal{G}_{t+1}-\mathcal{G}^{\dagger}\|_{\rho_u}^2\Big]
\;\lesssim\;&
\underbrace{E_{\mathrm{enc}}(\lambda,\Delta,r)}_{\text{input-side error}} \;+\;
\underbrace{\mathbb{E}_{v\sim (\mathcal{G}^\dagger)_{\#}\rho_u}
\Big[\big\|(P_2^{d_2}-I_{\mv})v\big\|_{\mv}^2\Big]}_{\text{output-encoding error}}\;+
\underbrace{C_{\mathrm{sgd}}\, t^{-\theta}}_{\text{optimization error}}.
\end{aligned}
\end{equation*}
Here the encoder- and regularization-dependent bias term is
\[
E_{\mathrm{enc}}(\lambda, \Delta, r)
:=
\Delta^{2\min\{r,1\}}
\;+\;\lambda^{\min\{2r,2\}},
\]
and the iteration-independent coefficient is
\[
\begin{aligned}
C_{\mathrm{sgd}}
:={}&
\eta_1^{-2r}
(\lambda\eta_1)^{
\min\left\{2r-\frac{\theta}{1-\theta},\,0\right\}
}
+\eta_1 
\begin{cases}
1+\log(1+\lambda^{-1}),
&0<\theta<\frac12,\\[1mm]
1+\log\left(1+(\lambda\eta_1)^{-1}\right),
&\theta=\frac12,\\[1mm]
(\lambda\eta_1)^{-\frac{2\theta-1}{1-\theta}},
&\frac12<\theta<1.
\end{cases}
\end{aligned}
\]
\end{theorem}

\paragraph{Discussion of Theorem~\ref{Thm1}.}
The error bound in Theorem~\ref{Thm1} naturally decomposes into three components:
\[
\text{(i) input-side error}, \qquad
\text{(ii) output-encoding error}, \qquad
\text{(iii) optimization error}.
\]
We discuss the role of each term below.

\medskip
\noindent
\emph{Input-side error.}
The input-side error \(E_{\mathrm{enc}}(\lambda,\Delta,r)\) depends on the
kernel discrepancy \(\Delta\), the regularity exponent \(r\), and the regularization
parameter \(\lambda\), but is independent of the iteration index \(t\). For orthogonal input encoders and Lipschitz kernels satisfying the assumptions of Corollary \ref{prop22}, Section~\ref{Section examples} shows that
\[
\Delta^2
\lesssim
\mathbb E_{u\sim\rho_u}
\Big[
\big\|\bigl(\mathcal D_1^{d_1}\mathcal E_1^{d_1}-I_{\mathcal U}\bigr) u\big\|_{\mathcal U}^2
\Big].
\]
By choosing \(\lambda\asymp\Delta\), one obtains
\[
E_{\mathrm{enc}}(\lambda,\Delta,r)
\lesssim
\bigg(
\mathbb E_{u\sim\rho_u}
\Big[
\big\|\bigl(\mathcal D_1^{d_1}\mathcal E_1^{d_1}-I_{\mathcal U}\bigr) u\big\|_{\mathcal U}^2
\Big]
\bigg)^{\min\{r,1\}}.
\]
Thus, for \(r\ge1\), the contribution saturates: the input-side error is governed by the kernel discrepancy \(\Delta\) rather than by the smoothness of the target operator. 
The regularization parameter \(\lambda\) balances the effect of \(\Delta\) with
the regularization bias and is chosen independently of \(t\).

\medskip
\noindent
\emph{Output-encoding error.}
The output-encoding error
\[
\be_{v\sim\l(\mathcal{G}^\dagger\r)_{\#}\rho_u}\l[\l\|\bigl({P_2^{d_2}}-I_{\mv}\bigr)v\r\|^2_{\mv}\r]
\]
depends solely on the output encoder. Together with the input-side error,
it constitutes an iteration-independent approximation barrier, separate from the
stochastic optimization error. 

\medskip
\noindent
\emph{Optimization error.}
The iteration-dependent term decays as
\[
t^{-\theta}, \qquad 0<\theta<1,
\]
and represents the contribution of the stochastic gradient procedure. Since
\(\theta\) can be chosen arbitrarily close to \(1\), the decay exponent of this optimization term can be made arbitrarily close to $1$.
The associated constants, however, depend on $\lambda$ and on the step-size
schedule, in particular on $\eta_1$ and $\theta$.

For the separated optimization term, the present decomposition differs from the
usual analyses of regularized SGD formulated directly on the underlying function
space. In such results, the iteration-dependent rates are expressed in terms of
the regularity exponent of the target and the eigenvalue decay of the associated
integral operator. In the encoder--decoder setting considered here, the
finite-resolution effects induced by the input encoder and the output decoder
are separated from the SGD dynamics, leaving the residual iteration-dependent
contribution of order \(t^{-\theta}\). For reference, existing
function-space analyses yield rates of order
\[
\mathcal{O}(t^{-1/2}) \cite{shi2024learning},\qquad
\mathcal{O}\!\left(
t^{-\min\left\{\frac{2r}{2r+1},\frac{2-s}{3-s}\right\}}
\right)
\cite{guo2023capacity,MAO2024101825,yang2025kernel},
\]
and
\[
\mathcal{O}\!\left(
t^{-\min\left\{\frac{2r}{2r+1},\frac{2}{3}\right\}}
\right)
\cite{ying2008online,yang2025learning},
\]
where \(r\) denotes the regularity exponent and \(s\) characterizes the
eigenvalue decay of the associated integral operator.\footnote{More precisely,
\cite{ying2008online} obtains
\(\mathcal{O}(t^{-2r/(2r+1)}\log t)\) for \(0<r\le 1/2\).}
An exception occurs in the noiseless setting \cite[Theorem~2.2]{GUO2022288},
where the rate \(\mathcal{O}(t^{-2r-1/2})\) is obtained for \(0<r\le 1/4\).

\medskip
\noindent
\emph{Resolution-independent formulation.}
A key feature of Theorem~\ref{Thm1} is that the regularity assumption on the
target operator is imposed with respect to the limiting kernel \(K_{\infty,\infty}\),
and is therefore independent of the encoder--decoder resolution, such as the
encoding dimensions \(d_1\) and \(d_2\). In the analysis, the encoder enters
only via the resolution-dependent scalar kernel \(k_{d_1}\) and operator-valued
kernels \(K_{d_1,d_2}\) and \(K_{\infty,d_2}\), which serve as intermediate
objects connecting the finite-resolution kernels to the limiting kernel \(K_{\infty,\infty}\).



\begin{remark}[No eigenvalue decay condition]
Theorem~\ref{Thm1} does not require additional capacity or eigenvalue decay
assumptions \cite{caponnetto2007optimal,guo2023capacity,li2024towards,meunier2024optimal,yang2025kernel} on the integral operator associated with the limiting kernel. 
If such spectral information is available, it affects only the constants $C_{\mathrm{sgd}}$ in the optimization term, without changing the \(t^{-\theta}\) decay rate of this separated optimization term.

Appendix~\ref{Appendix:Eigenvalues decay} provides an example illustrating
the spectral decay of kernel integral operators on infinite-dimensional separable Hilbert
spaces. In particular, for the Gaussian kernel $k$ under a nondegenerate Gaussian
measure with covariance operator \(C\), one has
\(\mathrm{Tr}(L_k^q)<\infty\) if and only if
\(\mathrm{Tr}(C^q)<\infty\). Thus, the spectral behavior of \(L_k\) is
determined by the covariance spectrum of the measure.
\end{remark}



We next consider the finite-horizon setting, where a fixed dataset of size $T<\infty$ is available and $T$ is known in advance. The algorithm processes samples sequentially, typically in a single pass, and the knowledge of $T$ allows the use of a constant step size chosen as a function of the horizon, which can be optimized (at least asymptotically) in 
$T$. A limitation of this setting is that such a step-size choice does not naturally support efficient warm-starting when additional samples become available.

\begin{theorem} \label{Thm2}
Suppose that Assumption~\ref{source condition} holds with \(r>0\) and
\(\mathcal{G}^{\mathrm{src}}\in L^2(\mathcal U,\rho_u;\mathcal V)\).
Let \(\{\mathcal G_t\}_{t\ge 1}\) be defined through \eqref{SGD1} with constant
step sizes $\eta_t=\eta T^{-\theta'}$ for $t=1,\ldots,T$,
where \(0<\theta'<1\), and with a regularization parameter \(\lambda>0\).
Assume further that $\lambda\le1$.

Then, there exists a constant $\bar\eta'>0$ such that for all \(0<\eta\le \bar\eta'\),
\begin{equation*}
\begin{aligned}
\mathbb E_{z^T}
\Big[
\big\|\mathcal G_{T+1}-\mathcal G^\dagger\big\|_{\rho_u}^2
\Big]
\;\lesssim\;&
\underbrace{\widetilde E_{\mathrm{enc}}(\lambda,\Delta,r)}_{\text{input-side error}}
\;+\;
\underbrace{
\mathbb E_{v\sim (\mathcal{G}^\dagger)_{\#}\rho_u}
\Big[
\big\|(P_2^{d_2}-I_{\mathcal V})v\big\|_{\mathcal V}^2
\Big]
}_{\text{output-encoding error}}
\;+
\underbrace{
\widetilde C_{\mathrm{sgd}}\,T^{-\theta'}
}_{\text{optimization error}}.
\end{aligned}
\end{equation*}
Here
\[
\widetilde E_{\mathrm{enc}}(\lambda,\Delta,r)
:=
\Delta^{2\min\{r,1\}}
+
\lambda^{\min\{2r,2\}},
\]
and the iteration-independent coefficient is
\[
\widetilde C_{\mathrm{sgd}}
:=
\eta^{-2r}
(\lambda\eta)^{\min\left\{2r-\frac{\theta'}{1-\theta'},\,0\right\}}
+\eta\bigl[1+\log(1+\lambda^{-1})\bigr].
\]
\end{theorem}

In the finite-horizon estimate, the choice of \(\lambda\) also depends on the \(\lambda\)-dependence of \(\widetilde C_{\mathrm{sgd}}\). If \(\theta'\le 2r/(2r+1)\), then \(\widetilde C_{\mathrm{sgd}}\) depends on \(\lambda\) only logarithmically, and setting \(\lambda\asymp\Delta\) 
balances the two terms in $\widetilde{E}_{\mathrm{enc}}$. 
If \(\theta'>2r/(2r+1)\), however, \(\widetilde C_{\mathrm{sgd}}\) also contains a negative power of \(\lambda\), so this additional \(\lambda\)-dependence should be taken into account when choosing \(\lambda\).

\subsection{Lower Bounds}\label{section:Lower bounds}

In this subsection, we show that the dependence on the kernel discrepancy and the output-encoding error in the preceding error bounds, namely the input-side error and the output-encoding error, cannot be improved in general when \(r\le1\). This conclusion remains valid even under additional spectral assumptions, such as \(\mathrm{Tr}(L_\infty^s)<\infty\) for some \(0<s<1\).


It suffices to prove the lower bound in a concrete setting. Let \(\mathcal U\) be a separable Hilbert space, and let
\(\rho_u\) be the law of
\[
u=\sum_{i\ge1}\sqrt{\mu_i}\,\xi_i\varphi_i,
\qquad
\mathbb P(\xi_i=1)=\mathbb P(\xi_i=-1)=\frac12,
\]
where the \(\xi_i\) are independent. Its covariance operator is
\[
\Gamma
:=\mathbb E_{u\sim\rho_u}\bigl[\langle\cdot,u\rangle_{\mathcal U}\,u\bigr]
=\sum_{i\ge1}\mu_i\,\langle\cdot,\varphi_i\rangle_{\mathcal U}\,\varphi_i,
\]
where \(\{\varphi_i\}_{i\ge1}\) is an orthonormal basis of \(\mathcal U\), and
\(\{\mu_i\}_{i\ge1}\) is a nonincreasing sequence of positive eigenvalues with
\(\sum_{i\ge1}\mu_i<\infty\).

Define the input-side encoder--decoder pair by
\[
\mathcal E_1^{d_1}u
:=
\bigl(\langle u,\varphi_1\rangle_{\mathcal U},\ldots,
\langle u,\varphi_{d_1}\rangle_{\mathcal U}\bigr)\in\mathbb R^{d_1},
\qquad
\mathcal D_1^{d_1}c
:=
\sum_{i=1}^{d_1}c_i\varphi_i ,
\]
and set \(P_1^{d_1}:=\mathcal D_1^{d_1}\mathcal E_1^{d_1}\). Similarly, let
\(\{\psi_i\}_{i\ge1}\) be a fixed orthonormal basis of \(\mathcal V\), and define
\[
\mathcal E_2^{d_2}v
:=
\bigl(\langle v,\psi_1\rangle_{\mathcal V},\ldots,
\langle v,\psi_{d_2}\rangle_{\mathcal V}\bigr)\in\mathbb R^{d_2},
\qquad
\mathcal D_2^{d_2}c
:=
\sum_{i=1}^{d_2}c_i\psi_i .
\]
Then \(P_2^{d_2}:=\mathcal D_2^{d_2}\mathcal E_2^{d_2}\) is the orthogonal projection
onto \(\operatorname{span}\{\psi_1,\ldots,\psi_{d_2}\}\).

Corresponding to Assumption~\ref{source condition}, define the operator class
associated with the limiting operator-valued kernel \(K_{\infty,\infty}\) by
\[
\mathcal O
:=
\Bigl\{
\mathcal G^\dagger
= L_{\infty,\infty}^r \mathcal G^{\mathrm{src}}
:\;
\mathcal G^{\mathrm{src}}\in L^2(\mathcal U,\rho_u;\mathcal V),
\ \|\mathcal G^{\mathrm{src}}\|_{\rho_u}\le1
\Bigr\}.
\]

Since every element of \(\mathcal H_{K_{d_1,d_2}}\) takes values in
\(\operatorname{ran}(P_2^{d_2})\), the output-side truncation error can be
separated by orthogonal projection. For any
\(\mathcal G^\dagger\in\mathcal O\) and
\(\mathcal G\in\mathcal H_{K_{d_1,d_2}}\),
\[
\|\mathcal G^\dagger-\mathcal G\|_{\rho_u}^2
=
\|P_2^{d_2}\mathcal G^\dagger-\mathcal G\|_{\rho_u}^2
+
\|(I_{\mathcal V}-P_2^{d_2})\mathcal G^\dagger\|_{\rho_u}^2 .
\]
It remains to quantify the approximation error over the projected target class
\[
P_2^{d_2}\mathcal O
:=
\{\,P_2^{d_2}\mathcal G^\dagger:\ \mathcal G^\dagger\in\mathcal O\,\}.
\]
If \(\mathcal G^\dagger=L_{\infty,\infty}^r\mathcal G^{\mathrm{src}}\) with
\(\mathcal G^{\mathrm{src}}\in L^2(\mathcal U,\rho_u;\mathcal V)\), then the
tensor-product representation gives
\[
P_2^{d_2}\mathcal G^\dagger
=
\bigl(I_{L^2(\mathcal U,\rho_u)}\otimes P_2^{d_2}\bigr)
\bigl(L_\infty^r\otimes I_{\mathcal V}\bigr)\mathcal G^{\mathrm{src}}
=
\bigl(L_\infty^r\otimes P_2^{d_2}\bigr)\mathcal G^{\mathrm{src}}
=
L_{\infty,d_2}^r\mathcal G^{\mathrm{src}},
\]
where \(L_{\infty,d_2}\) is the integral operator induced by the projected
limiting kernel \(K_{\infty,d_2}=k_\infty P_2^{d_2}\). We define the corresponding
worst-case approximation error by
\[
\mathfrak R_{d_1,d_2}
:=
\sup_{\mathcal G^\dagger\in\mathcal O}
\inf_{\mathcal G\in\mathcal H_{K_{d_1,d_2}}}
\bigl\|P_2^{d_2}\mathcal G^\dagger-\mathcal G\bigr\|_{\rho_u}.
\]


It follows that, for every \(\varepsilon>0\), there exists
\(\mathcal G^\dagger_\varepsilon\in\mathcal O\) such that, for all
\(\mathcal G\in\mathcal H_{K_{d_1,d_2}}\),
\[
\|\mathcal G^\dagger_\varepsilon-\mathcal G\|_{\rho_u}^2
\ge
\mathfrak R_{d_1,d_2}^2
+
\mathbb E_{v\sim(\mathcal G^\dagger_\varepsilon)_\#\rho_u}
\bigl[\|(I_{\mv}-{P_2^{d_2}})v\|_{\mathcal V}^2\bigr]
-\varepsilon .
\]
This separates the input-side approximation barrier from the output-encoding error.


We now take the limiting scalar kernel to be the linear kernel
\[
k_\infty(u,u'):=\langle u,u'\rangle_{\mathcal U}.
\]
The encoder-induced scalar kernel is then
\[
k_{d_1}(u,u')
=
\bigl\langle \mathcal E_1^{d_1}u,\mathcal E_1^{d_1}u'\bigr\rangle_2
=
\langle P_1^{d_1}u,P_1^{d_1}u'\rangle_{\mathcal U}.
\]
For \(u\sim\rho_u\), we have almost surely
\[
k_{d_1}(u,u)=\sum_{i=1}^{d_1}\mu_i
\le \|u\|_{\mathcal U}^2
=\sum_{i\ge1}\mu_i<\infty.
\]
Thus the kernel boundedness condition in
Section~\ref{Section: Kernel Lifting} holds uniformly in \(d_1\).
Under this choice, we obtain the following lower bound.

\begin{theorem}\label{lower bound}
Under the above setting, we have
\[
\mathfrak R_{d_1,d_2}\ge \mu_{d_1+1}^{r}.
\]
Moreover,
\[
\Delta=\|L_\infty-L_{d_1}\|=\mu_{d_1+1}.
\]
Consequently,
\[
\mathfrak R_{d_1,d_2}^2\ge \Delta^{2r}.
\]
\end{theorem}


Together with the preceding orthogonal decomposition, Theorem~\ref{lower bound} shows that,
already in this concrete setting, for \(r\le1\), 
the input-side dependence on
\(\Delta\) cannot in general be improved, while the output-encoding term is
unavoidable.


We next point out that the input-side lower bound is not caused by the
restriction to the RKHS \(\mathcal H_{K_{d_1,d_2}}\). It persists even when the
finite-dimensional learner is allowed to be an arbitrary 
\(f:\mathbb R^{d_1}\to\mathbb R^{d_2}\). Define the function class
\[
\mathcal O' := \left\{ f^\dagger = L_{\infty}^r f^{\mathrm{src}} :\; 
f^{\mathrm{src}} \in L^2(\mathcal U, \rho_u),\ 
\|f^{\mathrm{src}}\|_{\rho_u} \le 1 \right\}.
\]
Since \(\xi_{d_1+1}\) is centered and independent of
\(\mathcal E_1^{d_1}u\), the same construction as in the proof of
Theorem~\ref{lower bound} gives
\[
\begin{aligned}
\sup_{\mathcal G^\dagger \in \mathcal O} 
\inf_{f: \mathbb{R}^{d_1} \to \mathbb{R}^{d_2}} 
\bigl\| P_2^{d_2} \circ \mathcal G^\dagger 
- \mathcal D_2^{d_2} \circ f \circ \mathcal E_1^{d_1} \bigr\|_{\rho_u}
&\geq 
\sup_{\mathcal G^\dagger \in \mathcal O} 
\inf_{\mathcal G : \mathcal U \to \mathcal V} 
\bigl\| P_2^{d_2} \circ \mathcal G^\dagger 
- \mathcal G \circ P_1^{d_1} \bigr\|_{\rho_u} \\
&\geq 
\sup_{f^\dagger \in \mathcal O'} 
\inf_{f : \mathcal U \to \mathbb{R}} 
\bigl\| f^\dagger - f \circ P_1^{d_1} \bigr\|_{\rho_u} \\
&\geq \mu_{d_1+1}^r .
\end{aligned}
\]
Thus the obstruction is due to the loss of information
under the input encoder \(\mathcal E_1^{d_1}\).


\section{Admissible Kernels and Encoder--Decoder Constructions} \label{Section examples}
This section presents admissible kernel families and encoder--decoder
constructions. We verify dimension-uniform positive definiteness and
H\"older-type regularity for commonly used kernels, and derive estimates
for the discrepancy parameter \(\Delta\) and the encoding errors for several
representative encoders. The proofs are deferred to Appendix~\ref{Appendix 2}.

\subsection{Admissible Kernels}
We first discuss kernel families whose positive definiteness and
regularity properties are stable across ambient dimensions and encoding
resolutions.

\noindent \textbf{Radial kernels.}
A kernel $k$ is called radial if it can be written in the form
\begin{equation} \label{radial kernel}
    k(x,x')=\phi(\|x-x'\|_2)
\end{equation}
where $\phi:[0,\infty)\to\mathbb{R}$ is a univariate function. Such kernels are translation invariant and isotropic, as they depend solely on the Euclidean distance between inputs. Their regularity is determined by the smoothness of the even extension of
$\phi$ to $\mathbb{R}$, which in turn characterizes the regularity of the
associated RKHS.
Many kernels commonly used in practice belong to this class, including the Gaussian kernel \cite[Theorem 6.10]{wendland2004scattered}, $\phi(r)=e^{-\alpha r^2}$ for $\alpha>0$; the inverse multiquadric kernel \cite[Theorem 7.15]{wendland2004scattered} $\phi(r)=(c^2+r^2)^{-\beta}$, for $c>0$ and $\beta>0$; the Matérn kernel \cite{stein1999interpolation} (which reduces to the Laplacian kernel when $\nu=1/2$); and the Cauchy kernel \cite{gneiting2004stochastic}, which contains the rational quadratic kernel as a special case. All these kernels are positive definite in arbitrary dimensions.

A complete characterization of dimension-independent positive definiteness for radial kernels was established by Schoenberg (see also \cite[Chapter 7]{wendland2004scattered}). Specifically, a continuous radial kernel is positive definite on $\mathbb{R}^d$ for all $d\in\mathbb{N}$ if and only if
\[
\phi(r)=\int_0^\infty e^{-sr^2}d\gamma(s),
\]
for some finite nonnegative Borel measure $\gamma$. Equivalently, the function $\phi(\sqrt{\cdot})$ is completely monotone. This characterization includes many classical radial kernels that are positive definite in all finite dimensions.

\noindent \textbf{Dot product kernels.} Another important family consists of dot product kernels, which depend only on the inner product between two inputs and take the form
\begin{equation} \label{inner product kernel}
    k(x,x')=\phi(\langle x,x'\rangle_2).
\end{equation}
Such kernels are invariant under orthogonal transformations of the input space and naturally capture feature interactions through analytic expansions of inner products \cite{agrawal2019kernel}. Kar and Karnick \cite{kar2012random} showed that a function $\phi:\mathbb{R}\to\mathbb{R}$ defines a positive definite kernel on $\mathcal{B}_1\times\mathcal{B}_1$, via $(x,x')\mapsto\phi(\langle x,x'\rangle_2)$, if and only if $\phi$ is analytic with a Maclaurin expansion
\[
\phi(t)=\sum_{n\geq0}a_nt^n,\qquad a_n\geq0,
\]
where $\mathcal{B}_1$ denotes the unit ball of a (possibly infinite-dimensional) Hilbert space. Representative examples include polynomial kernels 
$\phi(t)=(t+\nu)^p$ for $\nu\geq0$ and $p\in\bn$, exponential dot product kernels of the form $\phi(t)=e^{\frac{t}{\sigma^2}}$ for $\sigma>0$, and Vovk's kernel $\phi(t)=\frac{1-t^p}{1-t}$ for $p\in\bn$. These constructions again yield kernels that remain positive definite independently of the ambient dimension.

The dimension-independent positive definiteness of radial and dot product kernels
ensures that kernel-based encoder--decoder constructions remain well-posed across
varying encoding resolutions and ambient dimensions. The following proposition shows that an \(\alpha\)-H\"older condition on
\(\phi\) yields dimension-uniform \(\alpha\)-H\"older estimates for the induced
radial or dot product kernels.

\begin{proposition} \label{prop20}
Let $\phi$ be a univariate function inducing either a radial kernel
$k(x,x')=\phi(\|x-x'\|_2)$ with $\phi:[0,\infty)\to\mathbb{R}$,
or a dot product kernel $k(x,x')=\phi(\langle x,x'\rangle_2)$ with
$\phi:\mathbb{R}\to\mathbb{R}$, and suppose that $k$ is positive definite on
$\mathbb{R}^d$ for all $d\ge1$.
Let $R>0$ and assume that $\phi$ is $\alpha$--H\"older continuous on $[0,2R]$
(for dot product kernels, on $[-R^2,R^2]$) for some $\alpha\in(0,1]$.

Then, for any $d\ge1$, the induced kernel $k$ is $\alpha$--H\"older continuous on
$\mathcal{B}_R(\mathbb{R}^d)\times \mathcal{B}_R(\mathbb{R}^d)$, with a H\"older
constant independent of $d$. More precisely, for any $x_1,x_1',x_2,x_2'\in \mathcal{B}_R(\mathbb{R}^d)$,
\begin{equation} \label{holder}
    \bigl|k(x_1,x_1')-k(x_2,x_2')\bigr|
\le C_{\alpha,R}
\bigl(\|x_1-x_2\|_2^{\alpha}+\|x_1'-x_2'\|_2^{\alpha}\bigr),
\end{equation}
where $C_{\alpha,R}$ depends only on $\alpha$, $R$, and the \(\alpha\)-H\"older constant
of \(\phi\) on the relevant interval, but not on the dimension $d$.

Moreover, if $k$ is a radial kernel and $\phi$ is globally $\alpha$--H\"older continuous on
$[0,\infty)$, then the above $\alpha$--Hölder continuity holds on the whole space $\mathbb{R}^d\times\mathbb{R}^d$, with a constant depending only on the Hölder constant of $\phi$. 
\end{proposition}

Most common radial kernels, including Gaussian, Laplacian (Matérn with $\nu=1/2$), Matérn with \(\nu\ge 1/2\), inverse multiquadric, and Cauchy kernels, admit dimension-uniform Lipschitz continuity on bounded subsets of 
$\mathbb{R}^d$ for any $d\geq1$. Moreover, the Gaussian, Laplacian, and Cauchy kernels are Lipschitz continuous on the whole space $\mathbb{R}^d\times\mathbb{R}^d$ in the sense of \eqref{holder}. 
In addition, common dot product kernels such as polynomial and exponential kernels are Lipschitz continuous on bounded subsets of the input space.

We extend these dimension-uniform regularity properties from finite-dimensional Euclidean spaces to Hilbert spaces.

\begin{proposition} \label{prop21}
Suppose that \(\mathcal U\) is a Hilbert space, and let \(k_\infty':\mathcal U\times\mathcal U\to\mathbb R\) be defined by either
\[
k_\infty'(u,u')=\phi(\|u-u'\|_{\mathcal U})
\]
or
\[
k_\infty'(u,u')=\phi(\langle u,u'\rangle_{\mathcal U}).
\]
Then the following statements hold.
\begin{itemize}
    \item[(1)] If the corresponding finite-dimensional kernel $k$ is positive definite on
    \(\mathbb R^d\) for every \(d\ge1\), then \(k_\infty'\) is positive definite on
    \(\mathcal U\).
    \item[(2)] For any \(R>0\), if \(\phi\) satisfies the assumptions of
    Proposition~\ref{prop20}, then \(k_\infty'\) is
    \(\alpha\)-H\"older continuous on \(B_R(\mathcal U)\times B_R(\mathcal U)\), with a constant independent of the ambient dimension.
\end{itemize}
\end{proposition}


We next use the dimension-uniform regularity above to control the discrepancy
parameter \(\Delta\), which measures the difference between the encoder-induced
kernel \(k_{d_1}\) and the limiting kernel \(k_\infty\). Recall that
\[
k_{d_1}(u,u')
=
k\bigl(\mathcal E_1^{d_1}u,\mathcal E_1^{d_1}u'\bigr),
\qquad u,u'\in\mathcal U .
\]
The following assumption describes the limiting geometry induced by the input
encoder and leads to the definition of the limiting scalar kernel.


Set $A_{d_1}:=
\bigl((\mathcal E_1^{d_1})^*\mathcal E_1^{d_1}\bigr)^{1/2}.$
Suppose that there exists a bounded linear operator \(T:\mathcal U\to\mathcal U\)
such that
\[
\mathbb E_{u\sim\rho_u}
\left[
\left\|
A_{d_1}u-Tu
\right\|_{\mathcal U}^2
\right]
\longrightarrow 0,
\qquad d_1\to\infty .
\]
We then define the limiting scalar kernel
\(k_\infty:\mathcal U\times\mathcal U\to\mathbb R\) by
\[
k_\infty(u,u') := k_\infty'(Tu,Tu').
\]


\begin{corollary} \label{prop22}
Suppose that the assumptions of Proposition~\ref{prop21} hold. In the
locally H\"older case, assume moreover that the local H\"older estimate is
valid on \(B_R(\mathcal U)\times B_R(\mathcal U)\) for some \(R>0\), and that
\begin{equation} \label{Holder cond}
    \operatorname*{ess\,sup}_{u\sim\rho_u}
\|A_{d_1}u\|_{\mathcal U}\le R,
\quad \text{for all sufficiently large } d_1 .
\end{equation}
Then the discrepancy parameter
\[
\Delta= \|L_{d_1}-L_\infty\|
\]
admits the bound
\[
\Delta \;\lesssim\;
\left(
\mathbb{E}\left[
\left\|A_{d_1}u-Tu\right\|_{\mathcal U}^{2\alpha}
\right]
\right)^{1/2},
\]
where the implicit constant is independent of \(d_1\). 

If the H\"older estimate holds globally, the boundedness condition \eqref{Holder cond} is not required.
\end{corollary}


In particular, consider Lipschitz kernels, corresponding to the case \(\alpha=1\)
in \eqref{holder}. If the input encoder--decoder pair is constructed from an
orthogonal expansion, then
\[
\mathcal D_1^{d_1}=(\mathcal E_1^{d_1})^*
\qquad\text{and}\qquad
A_{d_1}
=
\mathcal D_1^{d_1}\mathcal E_1^{d_1}.
\]
Taking \(T=I_{\mathcal U}\), Corollary~\ref{prop22} yields
\[
\Delta \;\lesssim\;
\left(
\mathbb E_{u\sim\rho_u}
\left[
\left\|
\mathcal D_1^{d_1}\mathcal E_1^{d_1}u-u
\right\|_{\mathcal U}^{2}
\right]
\right)^{1/2}.
\]
Consequently, for the simplified choice \(\lambda\asymp\Delta\), the
input-side error satisfies
\[
E_{\mathrm{enc}}(\lambda,\Delta,r)
\;\lesssim\;
\left(
\mathbb E_{u\sim\rho_u}
\left[
\left\|
\mathcal D_1^{d_1}\mathcal E_1^{d_1}u-u
\right\|_{\mathcal U}^{2}
\right]
\right)^{\min\{r,1\}}.
\]

For dot product kernels, the preceding estimate can be sharpened.
Let \(P=\mathcal D_1^{d_1}\mathcal E_1^{d_1}\) be the orthogonal
projection above. Assume that \(\rho_u(B_R(\mathcal U))=1\)
and that \(\phi\) is Lipschitz on \([-R^2,R^2]\). Write
\(e_{d_1}:=\mathbb E_{u\sim\rho_u}
\|(I_{\mathcal U}-P)u\|_{\mathcal U}^2\).
Orthogonality gives
\[
\langle u,u'\rangle_{\mathcal U}
-\langle Pu,Pu'\rangle_{\mathcal U}
=\langle(I_{\mathcal U}-P)u,(I_{\mathcal U}-P)u'\rangle_{\mathcal U}.
\]
Hence, by the Lipschitz property and Cauchy--Schwarz inequality,
\[
\Delta
\le \|k_{d_1}-k_\infty\|_{L^2(\rho_u\otimes\rho_u)}
\lesssim e_{d_1}.
\]
Consequently, taking \(\lambda\asymp\Delta\) gives
\[
E_{\mathrm{enc}}(\lambda,\Delta,r)
\lesssim e_{d_1}^{\,2\min\{r,1\}}.
\]


\subsection{Admissible Encoder–Decoder Constructions}


We first consider a general class of linear encoder--decoder pairs generated by
orthogonal expansions. This viewpoint gives a unified way to estimate the
encoding error for several concrete constructions below. For clarity, we focus
on the input-side error
\[
\mathbb{E}_{u\sim\rho_u}
\bigl[
\|\mathcal{D}_1^{d_1}\mathcal{E}_1^{d_1}u-u\|_{\mathcal{U}}^{2}
\bigr],
\]
since the output-side error can be treated analogously.

Let \(\{u_i\}_{i\ge1}\) be a complete orthonormal basis of the separable Hilbert
space \(\mathcal U\), and let \(\{\omega_i\}_{i\ge1}\) be a nondecreasing sequence
of positive weights. For \(s>0\), define the weighted Hilbert space
\begin{equation} \label{temp25}
    \mathcal U^s
    :=
    \left\{
    u\in\mathcal U:
    \|u\|_{\mathcal U^s}^2
    :=
    \sum_{i\ge1}
    \omega_i^{2s}
    \bigl|\langle u,u_i\rangle_{\mathcal U}\bigr|^2
    <\infty
    \right\}.
\end{equation}
Equipped with this norm, \(\mathcal U^s\) is a Hilbert space. It can be viewed as
a regularity class whose norm controls the decay of the expansion coefficients
of \(u\) in the basis \(\{u_i\}_{i\ge1}\).

We consider linear encoders and decoders of the form
\begin{equation} \label{temp26}
    \mathcal{E}^{d}(u):=\l(\langle u,u_1\rangle_{\mathcal{U}},\cdots,\langle u,u_d\rangle_{\mathcal{U}}\r)\in\mathbb{R}^d,\qquad\mathcal{D}^{d}(c):=\sum_{i=1}^dc_iu_i.
\end{equation}
In this case, the composition $\mathcal{D}^{d}\circ\mathcal{E}^{d}$ coincides with the orthogonal projection onto the subspace $\mathrm{span}\{u_1,\cdots,u_d\}$.

Suppose that the random input $u$ satisfies $\be_{u\sim\rho_u}\l[\|u\|_{\mathcal{U}^s}^2\r]\leq R^2$. Then,
\begin{align*}
    \be_{u\sim\rho_u}\l[\l\|\mathcal{D}^{d}\circ\mathcal{E}^{d}(u)-u\r\|_{\mathcal{U}}^{2}\r]&=\be_{u\sim\rho_u}\Big[\sum_{i\geq d+1}\l\lvert\langle u,u_i\rangle_{\mathcal{U}}\r\rvert^2\Big]\\&=\be_{u\sim\rho_u}\Big[\sum_{i\geq d+1}\omega_i^{-2s}\omega_i^{2s}\l\lvert\langle u,u_i\rangle_{\mathcal{U}}\r\rvert^2\Big]
    \leq R^2\omega_{d+1}^{-2s}.
\end{align*}
This estimate shows how linear encoders control the encoding error:
the encoder truncates the orthogonal expansion of \(u\), and the error is
controlled by the tail energy of the discarded modes. Its decay is therefore
determined by the growth of the weights \(\{\omega_i\}\) and $R^2$. In
particular, the encoding error decays at the rate
\(\omega_{d+1}^{-2s}\). 

We next present several representative examples that fit naturally within this setting. In each example, we specify a regularity class $\mathcal{U}^s$
 for which the encoding error admits an explicit decay rate in $d$.

\paragraph{Spectral truncation via Fourier bases.}
    Consider periodic functions defined on the domain $[-1,1]^m$ and let $\mathcal{U}=L^2(\mathbb{T}^m)$, where $\mathbb{T}^m$ denotes the $m$-dimensional flat torus. Let
    \[
    \psi_k(x):=e^{i\pi k\cdot x},\qquad k\in\mathbb{Z}^{m},
    \]
    be the standard Fourier basis of $\mathcal{U}$. For an integer $r\geq1$, we define the low-frequency index set
    \begin{equation} \label{temp24}
        \Lambda_r:=\l\{k\in\mathbb{Z}^m:\|k\|_{\infty}\leq r\r\},
    \end{equation}
    with $d =|\Lambda_r|$. The encoder $\mathcal{E}^d$ is defined by extracting the Fourier coefficients $\mathcal{E}^d(u) := \bigl(\langle u,\psi_k\rangle_{L^2(\mathbb{T}^m)}\bigr)_{k\in\Lambda_r},$
while the decoder $\mathcal{D}^d$ reconstructs the function from these coefficients via $\mathcal{D}^d(c) := \sum_{k\in\Lambda_r} c_k \psi_k.$

    When restricted to real-valued functions, the same encoding can equivalently be expressed using trigonometric basis functions, namely sine and cosine functions. In this representation, the encoding dimension satisfies $d \asymp |\Lambda_r|$.

    If the input distribution has finite second moment in a Sobolev space $H^s\l(\mathbb{T}^m\r)$ with $s>0$, then the corresponding encoding error obeys
    \[
    \be_{u\sim\rho_u}\l\|\l(\mathcal{D}^{d}\circ\mathcal{E}^{d}-I\r)(u)\r\|_{L^2\l(\mathbb{T}^m\r)}^{2} \lesssim d^{-\frac{2s}{m}},
    \]
    where $d\asymp\lvert\Lambda_r\rvert\asymp r^m$. This construction naturally applies to Sobolev spaces on compact Riemannian manifolds \cite{chavel1984eigenvalues}, which admit a spectral characterization in terms of the Laplace–Beltrami operator.





\paragraph{Kernel spectral encoding.} 
Let $\mathcal U=L^2(D)$, where $D$ is a bounded domain, and let
$\mathcal H$ be the RKHS on $D$ induced by a positive definite Mercer kernel $k$. Assume that \(\int_D k(x,x)\,dx<\infty\). Let $L_k:L^2(D)\to L^2(D)$ be the associated kernel integral operator, and let
$\{(\mu_i,\psi_i)\}_{i\ge1}$ be its positive eigenpairs, ordered so that
$\{\mu_i\}_{i\ge1}$ is nonincreasing. Then $\{\psi_i\}_{i\ge1}$ forms an
orthonormal basis of the subspace
$\overline{\operatorname{ran}(L_k)}=(\ker L_k)^\perp$.

Taking $\omega_i=\mu_i^{-1/2}$ in \eqref{temp25}, the corresponding weighted
Hilbert space is
\[
\mathcal H^s
:=
\operatorname{ran}(L_k^{s/2})
=
[L^2(D),\mathcal H]_{s,2},
\qquad 0<s<1,
\]
with equivalent norms; see \cite[Theorem~2.3]{yang2025kernel}.

Define the encoder $\mathcal E^d$ and decoder $\mathcal D^d$ by truncating and
reconstructing the kernel eigenfunction expansion using
$\{\psi_i\}_{i=1}^d$, as in \eqref{temp26}. If the input distribution has finite
second moment in $\mathcal H^s$, then
\[
\mathbb E_{u\sim\rho_u}
\bigl\|(\mathcal D^d\circ\mathcal E^d-I)u\bigr\|_{L^2(D)}^2
\;\lesssim\;
\mu_{d+1}^{s}.
\]
Thus the encoding error is controlled by the spectral decay of $L_k$: since
$\sum_{i\ge1}\mu_i<\infty$ and $\{\mu_i\}$ is nonincreasing, one has
$\mu_{d+1}\lesssim d^{-1}$ and hence the crude rate $d^{-s}$. More generally,
if $\mu_i\lesssim i^{-\beta}$, the rate is $d^{-\beta s}$.



\paragraph{Polynomial basis encoding.}
Consider functions defined on the hypercube $[-1,1]^m$ and let 
\[
\mathcal U = L^2([-1,1]^m).
\]
Let $\{P_n\}_{n\ge 0}$ denote the normalized one-dimensional Legendre polynomials, which form an orthonormal basis of $L^2([-1,1])$. For a multi-index
\(
\alpha=(\alpha_1,\dots,\alpha_m)\in\mathbb N_0^m,
\)
we define the $m$-variate Legendre basis functions by the tensor-product construction
\[
P_\alpha(x)
:=
\prod_{j=1}^m P_{\alpha_j}(x_j),
\qquad
x=(x_1,\dots,x_m)\in[-1,1]^m.
\]
Then $\{P_\alpha\}_{\alpha\in\mathbb N_0^m}$ forms an orthonormal basis of $\mathcal U$. Given an integer $r\ge 0$, define
\[
\Lambda_r^{\mathrm{poly}}
:=
\{\alpha\in\mathbb N_0^m:\|\alpha\|_\infty\le r\}.
\]
We construct the encoder $\mathcal E^d$ and decoder $\mathcal D^d$ by truncating and reconstructing the Legendre expansion with respect to the basis functions indexed by $\Lambda_r^{\mathrm{poly}}$.

If the input distribution has finite second moment in
\(C^{k,\alpha}([-1,1]^m)\) for some integer \(k>0\) and
\(0<\alpha\le1\), then the associated encoding error satisfies
\[
\be_{u\sim\rho_u}\bigl\|(\mathcal D^{d}\circ\mathcal E^{d}-I)(u)\bigr\|_{L^2([-1,1]^m)}^{2}
\;\lesssim\;
d^{-\frac{2s}{m}},
\]
where $s = k+\alpha$ and $d=|\Lambda_r^{\mathrm{poly}}|=(r+1)^m\asymp r^m$. See, e.g., \cite[Section~4.3]{liu2024deep}.

\paragraph{Multiscale wavelet encoding.} 
Let $\mathcal U = L^2(\mathbb T^m)$.  
For $j\in\mathbb N_0$, define the index sets
\[
K_j := \{k\in\mathbb Z^m : 0 \le k_\ell < 2^j,\ \ell=1,\dots,m\},\quad
L_0 := \{0,1\}^m,\quad
L_j := L_0 \setminus \{(0,\dots,0)\},\ j\ge1,
\]
and set
\[
\mathcal I := \{(j,k,\ell): j\in\mathbb N_0,\ k\in K_j,\ \ell\in L_j\}.
\]
According to \cite[Proposition~1.34]{triebel2008function}, the periodic wavelets constructed in \cite[Sec.~1.3.2]{triebel2008function}
form an orthonormal basis
\(
\Psi=\{\psi_{j,k}^{\ell}\}_{(j,k,\ell)\in\mathcal I}
\)
of $L^2(\mathbb T^m)$.

For a fixed $s>0$, choose the wavelets with sufficient regularity and vanishing moments for the characterization in \cite[Theorem~1.36]{triebel2008function}. With $p=2$, 
the associated periodic Besov space $B^{s}_{p,p}(\mathbb T^m)$ admits the characterization
\[
\|u\|_{B^{s}_{2,2}(\mathbb T^m)}^2
\asymp
\sum_{(j,k,\ell)\in\mathcal I}
2^{2js}\,\l|\l\langle u,\psi_{j,k}^{\ell}\r\rangle_{L^2(\mathbb T^m)}\r|^2,
\]
see \cite[Theorem~1.36]{triebel2008function}. Moreover, $B^{s}_{2,2}(\mathbb T^m)=H^{s}(\mathbb T^m)$.

Note that the number of wavelets at scale $j$ satisfies
$|\{(k,\ell):k\in K_j,\ \ell\in L_j\}|\asymp 2^{jm}$.
Reindexing $\Psi$ as $\{\psi_i\}_{i\ge1}$ in increasing scale order therefore yields the equivalent norm
\[
\|u\|_{H^{s}(\mathbb T^m)}^2
\asymp
\sum_{i\ge1} i^{\frac{2s}{m}}
\,\l|\l\langle u,\psi_i\r\rangle_{L^2(\mathbb T^m)}\r|^2 .
\]
We define the encoder and decoder by
\[
\mathcal E^{d}(u) := (\langle u,\psi_1\rangle_{L^2(\mathbb T^m)},\dots,\langle u,\psi_d\rangle_{L^2(\mathbb T^m)}),
\qquad
\mathcal D^{d}(c) := \sum_{i=1}^d c_i \psi_i .
\]
If the input distribution has finite second moment in $H^{s}(\mathbb T^m)$, then
\[
\be_{u\sim\rho_u}\|(\mathcal D^{d}\circ\mathcal E^{d}-I)(u)\|_{L^2(\mathbb {T}^m)}^2
\;\lesssim\;
d^{-2s/m}.
\]

\paragraph{(Empirical) principal component encoding.} 
Assume that \(\mathbb E\|u\|_{\mathcal U}^2<\infty\) and,
for simplicity, that \(\mathbb E[u]=0\).
The distribution $\rho_u$ induces a covariance operator on $\mathcal U$ defined by
\[
C
:=
\mathbb{E}_{u\sim\rho_u}\bigl[\langle \cdot , u\rangle_{\mathcal U}\, u\bigr].
\]
The operator $C$ is self-adjoint, positive, and trace-class, and admits the spectral decomposition
\[
C
=
\sum_{i\ge1}
\mu_i \langle \cdot , \psi_i\rangle_{\mathcal U}\, \psi_i,
\]
where $\{\mu_i\}_{i\ge1}$ is a nonincreasing sequence of positive eigenvalues and
$\{\psi_i\}_{i\ge1}$ is an orthonormal basis of \(\overline{\operatorname{ran}(C)}\).

In practice, the covariance operator is approximated using i.i.d.\ samples
$\{u_1,\dots,u_n\}\sim\rho_u$ via the empirical covariance operator
\[
\hat C
:=
\frac{1}{n}\sum_{i=1}^n \langle \cdot , u_i\rangle_{\mathcal U}\, u_i,
\]
whose eigenpairs are denoted by $\{(\hat\mu_i,\hat\psi_i)\}_{i\ge1}$.
We define the encoder $\mathcal E^{d}$ and decoder $\mathcal D^{d}$ by truncating and reconstructing
the spectral decomposition of $\hat C$ using the first $d$ empirical eigenfunctions
$\{\hat\psi_i\}_{i=1}^d$, as in \eqref{temp26}.

According to \cite[Proposition~2]{lanthaler2023operator}, if the distribution $\rho_u$ is sub-Gaussian,
then for any $\delta\in(0,\tfrac12)$ and any sample size $n\ge \log(2/\delta)$,
the following bound holds with probability at least $1-\delta$:
\[
\mathbb{E}_{u\sim\rho_u}
\bigl[
\|\mathcal D^{d}\circ\mathcal E^{d}(u)-u\|_{\mathcal U}^2
\bigr]
\;\lesssim\;
\sum_{i>d}\mu_i
+
\sqrt{\frac{d\log(2/\delta)}{n}}.
\]

This estimate highlights that the approximation error of the PCA encoder is governed by the
tail sum $\sum_{i>d}\mu_i$ of the population covariance spectrum, which depends on the regularity of $\rho_u$.
Moreover, by \cite[Proposition~15]{lanthaler2023operator}, if
$\mathcal U=H^s(D;\mathbb R^{m'})$, where $m'\in\mathbb N$ and $D$ is either a Lipschitz domain
in $\mathbb R^m$ or the torus $\mathbb T^m$, and if
\[
\mathbb{E}_{u\sim\rho_u}\bigl[\|u\|_{H^{s+\zeta}}^2\bigr]<\infty,
\quad\text{for some }\zeta>0,
\]
then the eigenvalue tail satisfies
\[
\sum_{i>d}\mu_i
\;\lesssim\;
d^{-2\zeta/m}.
\]
\begin{remark}
    For the theoretical analysis, we assume that the samples used to construct the
empirical PCA encoder are independent of those used to train the learning
algorithm, as in \cite{liu2024deep}. This assumption separates the randomness
introduced by empirical PCA from the randomness in the subsequent learning
procedure.

Furthermore, on the output side, empirical PCA is typically performed on noisy observations
$\{v_1,\dots,v_n\}$ rather than on the noise-free outputs
$\{\mathcal G^\dagger(u_1),\dots,\mathcal G^\dagger(u_n)\}$.
Conditioning on the sample used to construct the PCA encoder,
the output-side encoding error satisfies
\[
\mathbb{E}_{u\sim\rho_u}
\bigl[
\|\mathcal D^{d}\circ\mathcal E^{d}(\mathcal G^\dagger(u))
-
\mathcal G^\dagger(u)\|_{\mathcal V}^2
\bigr]
\;\le\;
\mathbb{E}_{v\sim\rho_v}
\bigl[
\|\mathcal D^{d}\circ\mathcal E^{d}(v)-v\|_{\mathcal V}^2
\bigr],
\]
Indeed, for a fresh pair \((u,v)\sim\rho\) independent of the
PCA sample, \(\mathbb E[\epsilon\mid u]=0\) and the linearity
of \(I_{\mathcal V}-\mathcal D^d\circ\mathcal E^d\)
make the cross term vanish.
\end{remark}

We next consider encoder--decoder pairs induced by point sampling.

\paragraph{Random sampling--based input encoding.} 
Assume that $\mathcal{U}$ is a separable Hilbert space that is continuously embedded into
$C^0(D)$, so that point-evaluation functionals are bounded.
Let
\[
\mathcal{E}_1^{d_1}
:=\frac{1}{\sqrt{d_1}}\bigl(\delta_{x_1},\ldots,\delta_{x_{d_1}}\bigr)^{\top},
\]
where $x_1,\ldots,x_{d_1}$ are drawn independently according to a probability measure $\mu$ on
$D$.
A direct calculation shows that
\[
(\mathcal{E}_1^{d_1})^*\mathcal{E}_1^{d_1}
=\frac{1}{d_1}\sum_{i=1}^{d_1} \delta_{x_i}^*\delta_{x_i}.
\]
Since point-evaluation functionals are bounded on $\mathcal{U}$, each operator
$\delta_x^*\delta_x$ is rank-one and hence Hilbert--Schmidt.

Define the mean operator
\[
C := \mathbb{E}_{x\sim\mu}\bigl[\delta_x^*\delta_x\bigr].
\]
Then, by a Hilbert--Schmidt Bernstein inequality
(Proposition~\ref{concentration}), with probability
at least $1-\delta$,
\[
\bigl\|(\mathcal{E}_1^{d_1})^*\mathcal{E}_1^{d_1} - C\bigr\|_{\mathrm{HS}}
\;\lesssim\;
\sqrt{\frac{1}{d_1}}\log(2/\delta).
\]
In particular, since $\|\cdot\|\le \|\cdot\|_{\mathrm{HS}}$, the same bound holds in operator
norm.

Let $T:=C^{1/2}$. By a perturbation bound for square roots of positive self-adjoint operators \cite{aleksandrov2010operator}, we obtain
\[
\bigl\| \bigl((\mathcal{E}_1^{d_1})^*\mathcal{E}_1^{d_1}\bigr)^{1/2} - T \bigr\|
\;\le\;
\bigl\|(\mathcal{E}_1^{d_1})^*\mathcal{E}_1^{d_1} - C\bigr\|^{1/2}.
\]
Assume that the hypotheses of Corollary~\ref{prop22}
hold with \(\alpha=1\). 
Combining this estimate with Corollary~\ref{prop22} and assuming that
$\mathbb{E}_{u\sim\rho_u}\|u\|_{\mathcal{U}}^2<\infty$, we conclude that, with probability at least
$1-\delta$,
\[
\Delta \;\lesssim\;
{d_1}^{-\frac{1}{4}}\;\sqrt{\log\frac{2}{\delta}}.
\]

\paragraph{Kernel interpolation–based output encoding.} 
On the output side, we require that
$\mathcal D_2^{d_2}\mathcal E_2^{d_2}=P_2^{d_2}$ is an orthogonal projection,
without imposing that the decoder be the adjoint of the encoder.
Let $\mathcal V=\mathcal H_Q$ be a RKHS over a compact
domain $D\subset\mathbb R^m$ with Lipschitz boundary, associated with a positive definite kernel $Q$.

Assume that $\mathcal H_Q$ is continuously embedded into a Sobolev space \(H^{s'}(D)\), where \(s'>s>\frac m2\).
Let
\[
\mathcal{E}_2^{d_2}
:=\frac{1}{\sqrt{d_2}}\bigl(\delta_{y_1},\ldots,\delta_{y_{d_2}}\bigr)^{\top}
\]
be the (scaled) point-evaluation operator at sampling locations
$Y=\{y_1,\ldots,y_{d_2}\}\subset D$. Assume that the Gram matrix
\(\bigl(Q(y_i,y_j)\bigr)_{i,j=1}^{d_2}\) is invertible.
We define the decoder $\mathcal D_2^{d_2}$ as the (scaled) minimum-norm kernel
interpolant \cite{scholkopf2001generalized}, i.e.,
for any $c\in\mathbb R^{d_2}$,
\[
\mathcal D_2^{d_2}(c)
:=\arg\min_{v\in\mathcal V}\|v\|_{\mathcal H_Q},
\quad\text{s.t.}\quad v(y_i)=\sqrt{d_2}\,c_i.
\]
With this choice, $\mathcal D_2^{d_2}\mathcal E_2^{d_2}=P_2^{d_2}$ coincides with the
$\mathcal H_Q$-orthogonal projection onto
$\mathrm{span}\{Q(\cdot,y_i)\}_{i=1}^{d_2}$.

Denote the fill distance of the set $Y$ by
\[
h_Y:=\max_{y\in D}\min_{i=1,\ldots,d_2}\|y-y_i\|_2.
\]
For \(v\in\mathcal V\), set \(e=v-P_2^{d_2}v\).
Since \(e|_Y=0\), the sampling inequality
\cite[Theorem~4.1]{arcangeli2007extension}, followed by
Sobolev interpolation, yields, for sufficiently small \(h_Y\),
\[
\|e\|_{H^s(D)}
\lesssim h_Y^{s'-s}\|e\|_{H^{s'}(D)}
\lesssim h_Y^{s'-s}\|e\|_{\mathcal V}
\le h_Y^{s'-s}\|v\|_{\mathcal V}.
\]
Squaring and taking expectations gives
\[
\be_{v\sim(\mathcal G^\dagger)_\#\rho_u}\|P_2^{d_2}v-v\|_{H^s(D)}^2
\;\lesssim\;
h_Y^{\,2(s'-s)}\,\be_{v\sim(\mathcal G^\dagger)_\#\rho_u}\|v\|_{\mathcal V}^2.
\]
Consequently, since
\(\mathcal G^\dagger\in L^2(\mathcal U,\rho_u;\mathcal V)\),
\[
\mathbb E_{v\sim(\mathcal G^\dagger)_\#\rho_u}
\bigl[\|(P_2^{d_2}-I_{\mathcal V})v\|_{H^s(D)}^2\bigr]
\;\lesssim\;
h_Y^{\,2(s'-s)}.
\]

Finally, we note that in Theorem~\ref{Thm1} the output-side encoding error
can be measured in the weaker Sobolev norm $H^s(D)$.
Indeed, in the proof one may use the decomposition
\[
\mathbb E_{z^t}\!\left[
\|\mathcal G_{t+1}-\mathcal G^\dagger\|_{L^2(\mathcal U,\rho_u;H^s(D))}^2
\right]
\;\lesssim\;
\mathbb E_{z^t}\!\left[
\|\mathcal G_{t+1}-P_2^{d_2}\mathcal G^\dagger\|_{L^2(\mathcal U,\rho_u;\mathcal V)}^2
\right]
+
\mathbb E_{v\sim(\mathcal G^\dagger)_\#\rho_u}\!\left[
\|(P_2^{d_2}-I_{\mathcal V})v\|_{H^s(D)}^2
\right],
\]
while all other steps of the analysis remain unchanged.




\section{Encoder--Decoder Neural Networks in the NTK Regime} \label{section 2.2}

In this section, we analyze overparameterized two-layer neural networks trained by SGD under the encoder--decoder architecture. With a symmetric random initialization, the linearized training dynamics are governed by the empirical NTK, which concentrates around its infinite-width limit as the width increases \cite{jacot2018neural}. We show that these NTK dynamics can be incorporated into the limiting-kernel framework developed in the preceding sections. The proofs of the results in this section are deferred to Appendix \ref{Sec 8}.

We define a limiting NTK on \(\mathcal U\) as the resolution limit of the
finite-dimensional NTKs induced by the encoders. This yields a comparison
between finite-width neural SGD and kernel SGD driven by the limiting NTK. The
resulting error bounds retain the same structure as the kernel-SGD bounds in
Subsection~\ref{subsection: upper bound}, with an additional finite-width term
of order \(M^{-1}\).


The analysis first treats the scalar-output case \(\mathcal V=\mathbb R\). The
extension to neural operators under an encoder--decoder architecture is discussed
in Subsection~\ref{section: Neural Operators}. Recall that
\(\mathcal E_1^{d_1}:\mathcal U\to\mathbb R^{d_1}\) denotes the input encoder. The
learned operator is represented as
\begin{equation} \label{temp28}
    \mathcal G^{\mathrm{NN}}
    =
    f_\Theta\circ\mathcal E_1^{d_1}:\mathcal U\to\mathbb R,
\end{equation}
where \(f_\Theta\) is a two-layer neural network with \(M\) hidden neurons,
\begin{equation} \label{temp29}
    f_\Theta(x)
    :=
    \frac{1}{\sqrt M}\sum_{r=1}^M
    a_r\,\sigma\bigl(b_r^\top x+\gamma c_r\bigr),
    \qquad x\in\mathbb R^{d_1}.
\end{equation}
Here \(\sigma:\mathbb R\to\mathbb R\) is the activation function and \(\gamma>0\)
is a fixed scaling parameter. The parameter vector is
\[
\Theta
:=
\bigl(a_1,\ldots,a_M,b_1,\ldots,b_M,c_1,\ldots,c_M\bigr)
\in\mathbb R^{M(d_1+2)}.
\]

\paragraph{Initialization.}
The network is initialized symmetrically as follows. Assume that \(M\) is even. Sample
\[
b_1,\ldots,b_{M/2}\stackrel{\mathrm{i.i.d.}}{\sim}\mathcal{N}(0,I_{d_1}).
\]
Set $b_{r+M/2}=b_r$ for $r=1,\ldots,M/2$,
and initialize $c_r=0$ for all $r=1,\ldots,M$.
Independently, sample
\[
a_1,\ldots,a_{M/2}\stackrel{\mathrm{i.i.d.}}{\sim}\mathrm{Rad},
\qquad
\text{where } \mathbb{P}(a_r=\pm1)=\frac12,
\]
and define \(a_{r+M/2}=-a_r\) for \(r=1,\ldots,M/2\).

Under this initialization, the network output vanishes identically:
\[
f_{\Theta_1}(x)=0,\qquad x\in\mathbb R^{d_1}.
\]
Hence the induced operator iterate starts from the zero element, matching the
initialization of the corresponding kernel-SGD recursion.

We consider the squared-loss expected risk associated with the induced operator,
\[
\mathcal R(\Theta)
:=
\mathbb E_{(u,v)\sim\rho}
\left[
\bigl(f_\Theta(\mathcal E_1^{d_1}u)-v\bigr)^2
\right].
\]
To control the training dynamics in the overparameterized regime, we use the
regularized objective
\[
\min_{\Theta}\;
\frac12\mathcal R(\Theta)
+
\frac{\lambda}{2}\|\Theta-\Theta_1\|_2^2,
\]
where \(\Theta_1\) denotes the random initialization. At iteration \(t\), an independent sample \((u_t,v_t)\sim\rho\) is drawn, and we set
\(x_t=\mathcal E_1^{d_1}u_t\). The SGD update is
\begin{equation}\label{NN update}
    \Theta_{t+1}
    =
    \Theta_t
    -
    \eta_t\bigl(f_{\Theta_t}(x_t)-v_t\bigr)
    \nabla_\Theta f_\Theta(x_t)\big|_{\Theta=\Theta_t}
    -
    \eta_t\lambda(\Theta_t-\Theta_1).
\end{equation}
The associated operator iterate is $\mathcal G_t^{\mathrm{NN}}
:=
f_{\Theta_t}\circ\mathcal E_1^{d_1}.$

In the NTK setting, \(\Delta\) measures the discrepancy between the kernel
associated with the finite-width, finite-resolution network and the
resolution-independent limiting NTK on \(\mathcal U\).

\subsection{Limiting NTK} \label{section: Limiting NTK}

We now construct the limiting NTK on \(\mathcal U\) and record its basic properties.

Recall that the infinite-width NTK associated with the two-layer network is given by
\begin{equation} \label{NTK}
\begin{aligned}
    k^{\mathrm{NTK}}\left(x,x^{\prime}\right)
    &:=
    \mathbb E_{\Theta_1}\left[
    \left\langle
    \nabla_\Theta f_{\Theta}(x)\big|_{\Theta=\Theta_1},
    \nabla_\Theta f_{\Theta}(x')\big|_{\Theta=\Theta_1}
    \right\rangle_2
    \right]
    \\
    &=
    \mathbb{E}_{b\sim\mathcal{N}(0, I_{d_1})}
    \left[
    \sigma\left(b^\top x\right)
    \sigma\left(b^\top x^{\prime}\right)
    \right]
    +
    \left(x^\top x^{\prime}+\gamma^2\right)
    \mathbb{E}_{b\sim\mathcal{N}(0, I_{d_1})}
    \left[
    \sigma^{\prime}\left(b^\top x\right)
    \sigma^{\prime}\left(b^\top x^{\prime}\right)
    \right].
\end{aligned}
\end{equation}
Since the biases are initialized at zero, the preactivations at
initialization are $b^\top x$. The term $\gamma^2$ arises from
differentiation with respect to the trainable bias parameters. 
For a network of finite width \(M\), the corresponding empirical NTK is
given by
\[
k^M\left(x,x^{\prime}\right)
:=
\frac{1}{M}\sum_{r=1}^M
\sigma\left(b_r^\top x\right)
\sigma\left(b_r^\top x^{\prime}\right)
+
\frac{x^\top x^{\prime}+\gamma^2}{M}
\sum_{r=1}^M
\sigma^{\prime}\left(b_r^\top x\right)
\sigma^{\prime}\left(b_r^\top x^{\prime}\right).
\]

We impose the following assumption on the activation function.

\begin{assumption}\label{assumption: NTK}
There exists a constant \(C_\sigma>0\) such that, for all \(x\in\mathbb R\),
\[
|\sigma(x)|\le C_\sigma,\qquad
|\sigma'(x)|\le C_\sigma,\qquad
|\sigma''(x)|\le C_\sigma .
\]
\end{assumption}

ReLU and leaky ReLU activations do not satisfy Assumption~\ref{assumption: NTK}. Although their limiting NTKs are well defined, extending the finite-width analysis to such activations may require different arguments.

The following lemma shows that the NTK depends on the inputs only through their
norms and inner product, and hence admits a dimension-independent representation.

\begin{lemma} \label{lemma NTK1}
Let \(k^{\mathrm{NTK}}\) be the NTK defined in \eqref{NTK}. For any \(d\ge2\)
and \(x,x'\in\mathbb R^d\), the following dimension-independent formulas hold.

If \(x\neq0\) and \(x'\neq0\), set
\[
\vartheta
:=
\arccos\left\langle
\frac{x}{\|x\|_2},
\frac{x'}{\|x'\|_2}
\right\rangle_2
\in[0,\pi].
\]
Then
\[
\begin{aligned}
k^{\mathrm{NTK}}(x,x')
&=
\mathbb E_{(b_1,b_2)\sim\mathcal N(0,I_2)}
\left[
\sigma\bigl(b_1\|x\|_2\bigr)
\sigma\bigl((b_1\cos\vartheta+b_2\sin\vartheta)\|x'\|_2\bigr)
\right]
\\
&\quad+
\bigl(x^\top x'+\gamma^2\bigr)
\mathbb E_{(b_1,b_2)\sim\mathcal N(0,I_2)}
\left[
\sigma'\bigl(b_1\|x\|_2\bigr)
\sigma'\bigl((b_1\cos\vartheta+b_2\sin\vartheta)\|x'\|_2\bigr)
\right].
\end{aligned}
\]
If \(x=0\), then
\[
k^{\mathrm{NTK}}(0,x')
=
\sigma(0)\,
\mathbb E_{b\sim\mathcal N(0,1)}
\bigl[\sigma(b\|x'\|_2)\bigr]
+
\gamma^2\sigma'(0)\,
\mathbb E_{b\sim\mathcal N(0,1)}
\bigl[\sigma'(b\|x'\|_2)\bigr].
\]
Similarly, if \(x'=0\), then
\[
k^{\mathrm{NTK}}(x,0)
=
\sigma(0)\,
\mathbb E_{b\sim\mathcal N(0,1)}
\bigl[\sigma(b\|x\|_2)\bigr]
+
\gamma^2\sigma'(0)\,
\mathbb E_{b\sim\mathcal N(0,1)}
\bigl[\sigma'(b\|x\|_2)\bigr].
\]
\end{lemma}

The proof of Lemma~\ref{lemma NTK1} follows from the orthogonal invariance of the
standard Gaussian distribution. After an orthogonal change of coordinates in
\(\mathbb R^d\), the pair of Gaussian projections \((b^\top x,b^\top x')\) has
the same distribution as
\[
\bigl(b_1\|x\|_2,\ (b_1\cos\vartheta+b_2\sin\vartheta)\|x'\|_2\bigr),
\qquad
(b_1,b_2)\sim\mathcal N(0,I_2),
\]
where \(\vartheta\) is the angle between \(x\) and \(x'\). The cases where one of
the two inputs is zero follow directly.


Lemma~\ref{lemma NTK1} shows that \(k^{\mathrm{NTK}}\) depends on \(x\) and
\(x'\) only through \(\|x\|_2\), \(\|x'\|_2\), and
\(\langle x,x'\rangle_2\). This dimension-independent structure motivates the
definition of a limiting NTK on \(\mathcal U\), which will be used in the
subsequent analysis.

\begin{remark}
One alternative initialization scheme samples the weight vectors \(b_r\) uniformly
from the unit sphere \(\mathbb S^{d-1}\); see, e.g.,
\cite{nitandaoptimal,nguyen2024optimal,cao2024stochastic}. Although this
distribution is also orthogonally invariant, its high-dimensional behavior
differs from that of the Gaussian initialization. By Stam's formula
\cite{stam1982limit,khokhlov2006uniform}, the joint density of the first two
coordinates of a random vector uniformly distributed on \(\mathbb S^{d-1}\) is
\[
p(x_1,x_2)
=
\frac{d-2}{2\pi}
\bigl(1-x_1^2-x_2^2\bigr)^{\frac{d-4}{2}}
\mathbf 1_{\{x_1^2+x_2^2\le 1\}} .
\]
As \(d\to\infty\), the corresponding law converges weakly to a Dirac mass at the
origin. Consequently, without rescaling, the associated NTK degenerates in the
high-dimensional limit to an inhomogeneous linear kernel, whose RKHS contains
only affine functions. This is too restrictive for nonlinear operator learning.

A natural alternative is to use the rescaled spherical initialization
\(b_r\sim \mathrm{Unif}(\sqrt d\,\mathbb S^{d-1})\), as considered 
in \cite{klukowski2022rate}. Under this scaling, the joint law of any fixed
number of coordinates converges weakly to the standard Gaussian law as
\(d\to\infty\), making it possible to compare the resulting NTK with the
Gaussian-initialized NTK in Lemma~\ref{lemma NTK1}. In this paper, we focus on Gaussian initialization, for which the limiting NTK
admits a direct characterization and fits naturally into the limiting-kernel
analysis.
\end{remark}

The next proposition records uniform boundedness and dimension-independent
Lipschitz continuity of the NTK.
\begin{proposition} \label{k Lip}
    Under Assumption \ref{assumption: NTK}, for any $d>0$ and $R>0$, and for any $x,x'\in\mathcal{B}_R\l(\mathbb{R}^{d}\r)$, 
    \[
    \l\lvert k^{\mathrm{NTK}}\left(x,x^{\prime}\right)\r\rvert\leq (1+R^2+\gamma^2)C_\sigma^2,
    \]
    and for any $x_1,x_1',x_2,x_2'\in\mathcal{B}_R(\br^{d})$,
    \begin{equation*}
        \l\lvert k^{\mathrm{NTK}}\left(x_1,x_1^{\prime}\right)-k^{\mathrm{NTK}}\left(x_2,x_2^{\prime}\right)\r\rvert\leq\l(1+R+R^2+\gamma^2\r)C_{\sigma}^2\l(\|x_1-x_2\|_2+\|x_1'-x_2'\|_2\r).
    \end{equation*}
\end{proposition}

Motivated by the dimension-independent structure in Lemma~\ref{lemma NTK1}, we introduce the following representation of the NTK.
\begin{definition}
    Define $\phi:\mathcal{D}\to\mathbb{R}$ by
    \[
    \phi(\|x\|,\|x'\|,\langle x,x'\rangle):=k^{\mathrm{NTK}}\left(x,x^{\prime}\right),
    \]
    where $\mathcal{D}:=\l\{\l(\xi_1,\xi_2,\eta\r):\xi_1,\xi_2\geq0,\lvert\eta\rvert\leq\xi_1\xi_2\r\}$.
\end{definition}

The following proposition gives a \(1/2\)-H\"older estimate for \(\phi\) on compact subsets of \(\mathcal D\) away from the origin.
\begin{proposition} \label{prop11}
    The function \(\phi\) satisfies a \(1/2\)-H\"older estimate on regions where the first two coordinates are bounded away from zero. More precisely, for any $0<r\leq R$, let 
    \[
    \mathcal{D}_r^R:=\l\{\l(\xi_1,\xi_2,\eta\r):r\leq\xi_1,\xi_2\leq R,\lvert\eta\rvert\leq\xi_1\xi_2\r\}.
    \]
    Then, there exists a constant $C_{r,R}>0$ such that for any $\l(\xi_1,\xi_2,\eta\r),\;\l(\xi'_1,\xi'_2,\eta'\r)\in\mathcal{D}_r^R$, the following inequality holds:
    \[
    \l\lvert\phi\l(\xi_1,\xi_2,\eta\r)-\phi\l(\xi'_1,\xi'_2,\eta'\r)\r\rvert\leq C_{r,R}
    \l(\sqrt{\lvert\xi_1-\xi_1'\rvert}+\sqrt{\lvert\xi_2-\xi_2'\rvert}+\sqrt{\lvert\eta-\eta'\rvert}\r).
    \]
\end{proposition}

\begin{remark}
The function  $\phi$ is continuous on $\mathcal{D}$, and is differentiable on the interior 
\[
\mathcal{D}^\circ = \left\{ (\xi_1, \xi_2, \eta) : \xi_1, \xi_2 > 0, \lvert \eta \rvert < \xi_1 \xi_2 \right\}.
\]

\end{remark}

We lift the NTK and its empirical counterpart to the input space \(\mathcal U\) by defining
\[
k^{\mathrm{NTK}}_{d_1}(u,u')
:=
k^{\mathrm{NTK}}\bigl(\mathcal E_1^{d_1}u,\mathcal E_1^{d_1}u'\bigr),
\qquad
k^M_{d_1}(u,u')
:=
k^M\bigl(\mathcal E_1^{d_1}u,\mathcal E_1^{d_1}u'\bigr).
\]
As in Section~\ref{Section examples}, assume that there exists a bounded linear
operator \(T:\mathcal U\to\mathcal U\) such that
\[
\mathbb E_{u\sim\rho_u}
\left[
\left\|
\bigl((\mathcal E_1^{d_1})^*\mathcal E_1^{d_1}\bigr)^{1/2}u
-
Tu
\right\|_{\mathcal U}^2
\right]
\longrightarrow 0,
\qquad d_1\to\infty .
\]

This leads to the following definition of the limiting NTK on \(\mathcal U\).

\begin{definition}\label{limiting NTK}
Define \(k'_\infty:\mathcal U\times\mathcal U\to\mathbb R\) by
\[
k'_\infty(u,u')
:=
\phi\bigl(
\|u\|_{\mathcal U},
\|u'\|_{\mathcal U},
\langle u,u'\rangle_{\mathcal U}
\bigr).
\]
The limiting NTK \(k^{\mathrm{NTK}}_\infty:\mathcal U\times\mathcal U\to\mathbb R\) is defined by
\[
k^{\mathrm{NTK}}_\infty(u,u')
:=
k'_\infty(Tu,Tu').
\]
\end{definition}

The kernel \(k'_\infty\) is symmetric by construction. To verify positive
definiteness, fix \(n\ge1\), \(a_1,\ldots,a_n\in\mathbb R\), and
\(u_1,\ldots,u_n\in\mathcal U\). Let
\(S:=\operatorname{span}\{u_1,\ldots,u_n\}\), and let
\(\{e_\ell\}_{\ell=1}^m\) be an orthonormal basis of \(S\). Define
\[
\widetilde u_i
:=
\bigl(\langle u_i,e_\ell\rangle_{\mathcal U}\bigr)_{\ell=1}^m
\in\mathbb R^m,
\qquad i=1,\ldots,n .
\]
Then $\|\widetilde u_i\|_2=\|u_i\|_{\mathcal U}$ and $\langle \widetilde u_i,\widetilde u_j\rangle_2
=
\langle u_i,u_j\rangle_{\mathcal U}.$ By the definition of \(\phi\) and Lemma~\ref{lemma NTK1}, 
\[
k'_\infty(u_i,u_j)
=
k^{\mathrm{NTK}}(\widetilde u_i,\widetilde u_j).
\]
Since \(k^{\mathrm{NTK}}\) is positive definite on finite-dimensional Euclidean
spaces, it follows that
\[
\sum_{i,j=1}^n a_i a_j k'_\infty(u_i,u_j)
=
\sum_{i,j=1}^n a_i a_j
k^{\mathrm{NTK}}(\widetilde u_i,\widetilde u_j)
\ge0.
\]
Thus \(k'_\infty\) is positive definite on \(\mathcal U\). Consequently,
\(k^{\mathrm{NTK}}_\infty(u,u')\) is also positive definite.

The dimension-independent Lipschitz estimate for \(k^{\mathrm{NTK}}\) also extends to \(k'_\infty\).

\begin{proposition}\label{NTK: lipschitz}
Under Assumption~\ref{assumption: NTK}, the kernel \(k'_\infty\) is locally
Lipschitz continuous on \(\mathcal U\times\mathcal U\). More precisely, for any
\(R>0\) and any \(u_1,u_1',u_2,u_2'\in\mathcal B_R(\mathcal U)\), one has
\[
\bigl| k'_\infty(u_1,u_1')-k'_\infty(u_2,u_2')\bigr|
\le
(1+R+R^2+\gamma^2)C_\sigma^2
\bigl(
\|u_1-u_2\|_{\mathcal U}
+
\|u_1'-u_2'\|_{\mathcal U}
\bigr).
\]
\end{proposition}


The following corollary controls the discrepancy induced by the finite-resolution NTK.

\begin{corollary}\label{coro1}
Suppose that Assumption~\ref{assumption: NTK} holds. Let $A_{d_1}
:=
\bigl((\mathcal E_1^{d_1})^*\mathcal E_1^{d_1}\bigr)^{1/2},$
and assume that there exists a bounded linear operator \(T:\mathcal U\to\mathcal U\) such that
\[
\mathbb E_{u\sim\rho_u}
\left[
\|A_{d_1}u-Tu\|_{\mathcal U}^2
\right]
\longrightarrow 0,
\qquad d_1\to\infty .
\]
Assume also that
\(\kappa_E:=\sup_{d_1\ge1}\|\mathcal E_1^{d_1}\|<\infty\),
and set \(R_*:=R\max\{\kappa_E,\|T\|\}\).
Then, for any $u,u'\in\mathcal{B}_R(\mathcal{U})$, 
\begin{equation}\label{temp21}
\begin{aligned}
\bigl|
k^{\mathrm{NTK}}_{d_1}(u,u')
-
k^{\mathrm{NTK}}_\infty(u,u')
\bigr|
\le\;&
(1+R_*+R_*^2+\gamma^2)C_\sigma^2
\\
&\times
\bigl(
\|A_{d_1}u-Tu\|_{\mathcal U}
+
\|A_{d_1}u'-Tu'\|_{\mathcal U}
\bigr).
\end{aligned}
\end{equation}
Consequently, if $\mathrm{supp}(\rho_u)\subset\mathcal{B}_R(\mathcal{U})$, then
\[
\bigl\|
L_{k^{\mathrm{NTK}}_{d_1}}
-
L_{k^{\mathrm{NTK}}_\infty}
\bigr\|_{\mathrm{HS}}
\le
2(1+R_*+R_*^2+\gamma^2)C_\sigma^2
\left(
\mathbb E_{u\sim\rho_u}
\left[
\|A_{d_1}u-Tu\|_{\mathcal U}^2
\right]
\right)^{1/2}.
\]
\end{corollary}


The next proposition controls the finite-width discrepancy between the empirical NTK
\(k^M_{d_1}\) and its infinite-width counterpart \(k^{\mathrm{NTK}}_{d_1}\).

\begin{proposition} \label{prop12}
Suppose that Assumption~\ref{assumption: NTK} holds. Assume further that
\(\operatorname{supp}(\rho_u)\subset \mathcal B_R(\mathcal U)\) and
$\sup_{d_1\ge1}\|\mathcal E_1^{d_1}\|<\infty.$ Then, with probability at least \(1-2\delta\) over the random initialization,
\[
\bigl\|
L_{k^{\mathrm{NTK}}_{d_1}}
-
L_{k^M_{d_1}}
\bigr\|_{\mathrm{HS}}
\lesssim
\frac{\log(2/\delta)}{\sqrt M}.
\]
Consequently, if the conditions of Corollary~\ref{coro1} are satisfied, then
\[
\Delta
=
\bigl\|
L_{k^M_{d_1}}
-
L_{k^{\mathrm{NTK}}_\infty}
\bigr\|
\lesssim
\frac{\log(2/\delta)}{\sqrt M}
+
\left(
\mathbb E_{u\sim\rho_u}
\left[
\left\|
\bigl((\mathcal E_1^{d_1})^*\mathcal E_1^{d_1}\bigr)^{1/2}u
-
Tu
\right\|_{\mathcal U}^2
\right]
\right)^{1/2}.
\]
\end{proposition}

Thus, in the NTK setting, the kernel discrepancy \(\Delta\) is controlled by the
finite width \(M\) and the encoding resolution.

\subsection{From Neural SGD to Kernel SGD} \label{section: Neural SGD to Kernel SGD}
This subsection compares the neural SGD iterates with the kernel-SGD iterates
driven by the empirical NTK. Through this comparison, the error bounds from
Subsection~\ref{subsection: upper bound} are transferred to the finite-width
neural network setting. The limiting NTK enters through the Assumption \ref{source condition}
and the kernel discrepancy $\Delta$.

We first consider the kernel-SGD recursion in the encoded space associated with
the empirical NTK \(k^M\):
\[
\begin{cases}
f_1 := 0,\\[2pt]
f_{t+1}
:= (1-\lambda\eta_t)f_t
-\eta_t\,k^M(\cdot,x_t)\bigl(f_t(x_t)-v_t\bigr),
\end{cases}
\]
where \(x_t=\mathcal E_1^{d_1}u_t\). 

Composing with the input encoder gives the lifted kernel \(k^M_{d_1}\) on
\(\mathcal U\). Defining $\mathcal G_t := f_t\circ \mathcal E_1^{d_1},$ we obtain the lifted recursion
\[
\mathcal G_{t+1}
=
(1-\lambda\eta_t)\mathcal G_t
-
\eta_t\,
k^M_{d_1}(\cdot,u_t)
\bigl(\mathcal G_t(u_t)-v_t\bigr).
\]
This recursion serves as the reference kernel-SGD dynamics for comparison with
the neural SGD iterates.

We impose the following boundedness condition on the data distribution and the
input encoders.

\begin{assumption}\label{assumption: sample 1}
There exist constants \(R>0\) and \(B>0\) such that, for \((u,v)\sim\rho\),
\[
\|u\|_{\mathcal U}\le R,
\qquad
\|v\|_{\mathcal V}\le B,
\quad \text{almost surely}.
\]
In addition, the input encoders satisfy
\[
\sup_{d_1\ge 1}\,\|\mathcal E_1^{d_1}\| < \infty.
\]
\end{assumption}
Enlarging \(R\) if necessary, we also have $\sup_{d_1\ge1}\|(\mathcal E_1^{d_1}u,\gamma)\|_2\le R$ $\rho_u$-a.s. We use this enlarged radius in the following NTK estimates.

The NTK approximation along the SGD trajectory relies on keeping the parameters
uniformly close to their initialization throughout training.
The following proposition gives a uniform stability estimate for polynomially
decaying step sizes. 
It shows that, for fixed $\lambda$, the parameter drift is bounded
uniformly over the training horizon whenever $M\gtrsim\lambda^{-4}$.

\begin{proposition}[Uniform stability under decreasing step sizes]\label{prop13}
Suppose that Assumptions~\ref{assumption: NTK} and~\ref{assumption: sample 1}
hold. Let $0<\lambda\le1$. Consider SGD with polynomially decaying step sizes $\{\eta_t=\eta_1 t^{-\theta}\}_{t\in\mathbb{N}_T}$, where \(0<\theta<1\) and
\[
0<\eta_1
\le
\min\left\{
\frac{1}{(1+R^2)C_\sigma^2+1},\,1-\theta
\right\}.
\]
Assume further that
\[
M \gtrsim \lambda^{-4}.
\]
Then, for all \(1\le t\le T+1\), the SGD iterates satisfy
\[
\|\Theta_t-\Theta_1\|_2
\lesssim
\lambda^{-1}.
\]
\end{proposition}


The uniform stability estimate allows the neural-network iterates to be compared
with the corresponding kernel-SGD dynamics driven by the empirical NTK. Combining
this comparison with the kernel-SGD error bounds in
Subsection~\ref{subsection: upper bound} gives the following theorem.

\begin{theorem}\label{NTK SGD1}
Suppose that Assumption~\ref{source condition} holds with \(r>0\), namely
\(\mathcal G^\dagger=L_{k^{\mathrm{NTK}}_\infty}^{\,r}\mathcal G^{\mathrm{src}}\)
for some \(\mathcal G^{\mathrm{src}}\in L^2(\mathcal U,\rho_u)\), where
\(k^{\mathrm{NTK}}_\infty\) is defined in Definition~\ref{limiting NTK}, and that Assumptions~\ref{assumption: NTK} and~\ref{assumption: sample 1} hold.

Let \(\{\mathcal G_t^{\mathrm{NN}}\}_{t\ge1}\) be generated by \eqref{NN update}
with step sizes \(\eta_t=\eta_1t^{-\theta}\), \(0<\theta<1\), and regularization
parameter \(\lambda>0\). Set
\(\Delta:=\|L_{k^M_{d_1}}-L_{k^{\mathrm{NTK}}_\infty}\|\). Assume that
\(\lambda\le1\),
\[
M\gtrsim \lambda^{-4}.
\]
Then there exists a constant $\tilde{\eta}>0$, independent of
$\Delta$, $\lambda$, $M$, and $T$, 
 such that for any
\(0<\eta_1\le\tilde{\eta}\) and any \(1\le t\le T\),
\[
\mathbb E_{z^t}\!\left[
\bigl\|\mathcal G_{t+1}^{\mathrm{NN}}-\mathcal G^\dagger\bigr\|_{\rho_u}^2
\right]
\;\lesssim\;
\underbrace{E_{\mathrm{enc}}(\lambda,\Delta,r)}
        _{\text{input-side error}}
\;+\;
\underbrace{\frac{\lambda^{-6}}{M}}_{\text{finite-width error}}
\;+\;
\underbrace{C_{\mathrm{sgd}}\,t^{-\theta}}_{\text{optimization error}}.
\]
Here \(E_{\mathrm{enc}}(\lambda,\Delta,r)\) is the input-side error term and
\(C_{\mathrm{sgd}}\) is the SGD coefficient in Theorem~\ref{Thm1}.
\end{theorem}

\begin{remark}[Choice of the regularization parameter]
The following consequences of Theorem~\ref{NTK SGD1} make the role of the
regularization parameter explicit. Set
\[
\mathfrak e_{d_1}^2
:=
\mathbb{E}_{u\sim\rho_u}\!\left[
\bigl\|
\bigl((\mathcal E_1^{d_1})^*\mathcal E_1^{d_1}\bigr)^{1/2}u-Tu
\bigr\|_{\mathcal U}^2
\right].
\]
If $M\gtrsim \frac{\bigl(\log(2/\delta)\bigr)^2}{\mathfrak e_{d_1}^2},$
then Proposition~\ref{prop12} implies that, with probability at least \(1-2\delta\), $\Delta\lesssim \mathfrak e_{d_1}.$

On this event, and under the width condition required in
Theorem~\ref{NTK SGD1}, the estimate takes the following simplified forms. If
\(0<r<1\) and \(0<\Delta\le1\), choosing \(\lambda=M^{-\frac{1}{2(r+3)}}\) gives
\[
\mathbb{E}_{z^t}\!\left[
\bigl\|\mathcal G_{t+1}^{\mathrm{NN}}-\mathcal G^\dagger\bigr\|_{\rho_u}^2
\right]
\lesssim
\Delta^{2r}
+
M^{-\frac{r}{r+3}}
+
C_{\mathrm{sgd}}\,t^{-\theta}.
\]
If \(r\ge1\), choosing \(\lambda=M^{-1/8}\) gives, whenever the conditions of
Theorem~\ref{NTK SGD1} are satisfied,
\[
\mathbb{E}_{z^t}\!\left[
\bigl\|\mathcal G_{t+1}^{\mathrm{NN}}-\mathcal G^\dagger\bigr\|_{\rho_u}^2
\right]
\lesssim
\Delta^{2}
+
M^{-1/4}
+
C_{\mathrm{sgd}}\,t^{-\theta}.
\]
\end{remark}

For completeness, we record the constant step-size analogue of
Proposition~\ref{prop13}. The parameter deviation again admits a bound of order
$\lambda^{-1}$, uniformly over the training horizon.

\begin{proposition}[Uniform stability under constant step sizes]\label{prop:constant-stability}
Suppose that Assumptions~\ref{assumption: NTK} and~\ref{assumption: sample 1}
hold. Let $0<\lambda\le1$. Consider SGD with constant step sizes \(\eta_t=\eta T^{-\theta'}\),
\(t\in\mathbb N_T\), where \(0<\theta'<1\) and
\(0<\eta\le((1+R^2)C_\sigma^2+1)^{-1}\). Assume that
\[
M \gtrsim \lambda^{-4}.
\]
Then, for all \(1\le t\le T+1\), the SGD iterates satisfy
\[
\|\Theta_t-\Theta_1\|_2
\lesssim
\lambda^{-1}.
\]
\end{proposition}


We state the constant step-size analogue of Theorem~\ref{NTK SGD1}.

\begin{theorem}\label{NTK SGD2}
Suppose that the assumptions of Theorem~\ref{NTK SGD1} hold, except that the
step sizes are constant, \(\eta_t=\eta T^{-\theta'}\), with \(0<\theta'<1\).
Then there exists 
a constant $\tilde{\eta}'>0$, independent of
$\Delta$, $\lambda$, $M$, and $T$, 
such that for any
\(0<\eta\le\tilde{\eta}'\),
\[
\mathbb E_{z^T}\!\left[
\bigl\|\mathcal G_{T+1}^{\mathrm{NN}}-\mathcal G^\dagger\bigr\|_{\rho_u}^2
\right]
\;\lesssim\;
\underbrace{\widetilde E_{\mathrm{enc}}(\lambda,\Delta,r)}_{\text{input-side error}}
\;+\;
\underbrace{\frac{\lambda^{-6}}{ M}}_{\text{finite-width error}}
\;+\;
\underbrace{\widetilde C_{\mathrm{sgd}}\,
T^{-\theta'}}_{\text{optimization error}}.
\]
Here \(\widetilde E_{\mathrm{enc}}(\lambda,\Delta,r)\) is the input-side error
term and \(\widetilde C_{\mathrm{sgd}}\) is the SGD coefficient in
Theorem~\ref{Thm2}.
\end{theorem}



The choice of \(\lambda\) involves considerations similar to those above, and we do not further optimize this choice in the constant step-size case.

\subsection{Encoder--Decoder Neural Operators} \label{section: Neural Operators}

We now incorporate the output encoder--decoder pair into the neural network
construction. The output space is accessed through an encoder
\(\mathcal E_2^{d_2}:\mathcal V\to\mathbb R^{d_2}\) and a decoder
\(\mathcal D_2^{d_2}:\mathbb R^{d_2}\to\mathcal V\). The induced neural operator
is
\[
    \mathcal G^{\mathrm{NN}}
    =
    \mathcal D_2^{d_2}\circ F_\Theta\circ\mathcal E_1^{d_1}
    :\mathcal U\to\mathcal V,
\]
where \(F_\Theta:\mathbb R^{d_1}\to\mathbb R^{d_2}\) is a two-layer multi-output
network.

We write
\[
F_\Theta(x)
=
\bigl(
f_{\Theta^{(1)}}(x),\,
f_{\Theta^{(2)}}(x),\,
\ldots,\,
f_{\Theta^{(d_2)}}(x)
\bigr)^\top,
\qquad x\in\mathbb R^{d_1},
\]
where each \(f_{\Theta^{(j)}}:\mathbb R^{d_1}\to\mathbb R\) is a two-layer
network of width \(M\), and
\(\Theta=(\Theta^{(1)},\ldots,\Theta^{(d_2)})\) collects all trainable
parameters. All components are evaluated on the same encoded input
\(\mathcal E_1^{d_1}u\). 
The component networks have the same architecture and width.
They use separate trainable parameter blocks, initialized using
the same draw from the symmetric scheme in the scalar-output setting.

Under the symmetric initialization above, each scalar component
\(f_{\Theta^{(j)}}\) is initialized at zero. Consequently,
\[
F_{\Theta_1}(x)=0\in\mathbb R^{d_2},
\qquad x\in\mathbb R^{d_1}.
\]
We consider the squared-loss expected risk associated with the encoded outputs,
\[
\mathcal R(\Theta)
:=
\mathbb E_{(u,v)\sim\rho}
\left[
\bigl\|
F_\Theta(\mathcal E_1^{d_1}u)-\mathcal E_2^{d_2}v
\bigr\|_2^2
\right].
\]
To control the training dynamics in the overparameterized regime, we use the
regularized objective
\[
\min_{\Theta}\;
\frac12\mathcal R(\Theta)
+
\frac{\lambda}{2}\|\Theta-\Theta_1\|_2^2.
\]
At iteration \(t\), an independent sample \((u_t,v_t)\sim\rho\) is drawn, with
\[
x_t=\mathcal E_1^{d_1}u_t,
\qquad
y_t=\mathcal E_2^{d_2}v_t.
\]
Writing \(\Theta=(\Theta^{(1)},\ldots,\Theta^{(d_2)})\) and
\(y_t=(y_t^{(1)},\ldots,y_t^{(d_2)})\), the componentwise SGD updates are
\begin{equation}\label{eq:multi-output-SGD}
\Theta^{(j)}_{t+1}
=
\Theta^{(j)}_t
-
\eta_t
\bigl(
f_{\Theta^{(j)}_t}(x_t)-y_t^{(j)}
\bigr)
\nabla_{\Theta^{(j)}} f_{\Theta^{(j)}}(x_t)\big|_{\Theta^{(j)}=\Theta_t^{(j)}}
-
\eta_t\lambda(\Theta^{(j)}_t-\Theta^{(j)}_1),
\quad j=1,\ldots,d_2.
\end{equation}
The associated neural operator iterate is $\mathcal G_t^{\mathrm{NN}}
:=
\mathcal D_2^{d_2}\circ F_{\Theta_t}\circ\mathcal E_1^{d_1}.$


We define the limiting operator-valued NTK by
\begin{equation} \label{operator-valued NTK}
    K^{\mathrm{NTK}}_{\infty}(u,u')
    :=
    k^{\mathrm{NTK}}_{\infty}(u,u') I_{\mathcal V},
\end{equation}
and denote by \(L_{K^{\mathrm{NTK}}_{\infty}}\) the associated integral operator.

To obtain error bounds for encoder--decoder neural operators, we impose the
following boundedness and stability assumptions, which ensure uniformity of the
estimates with respect to the output encoding dimension.

\begin{assumption}\label{ass:output-encoding}
There exists a constant \(R>0\) such that
\(\|u\|_{\mathcal U}\le R\) \(\rho_u\)-a.s., and the input encoders satisfy
\(\sup_{d_1\ge1}\|\mathcal E_1^{d_1}\|<\infty\).
Moreover, the output encoder--decoder pairs are uniformly stable in the sense that
\[
\sup_{d_2\ge1}
\sum_{j=1}^{d_2}
\operatorname*{ess\,sup}_{v\sim\rho_v}
\bigl|(\mathcal E_2^{d_2}v)_j\bigr|^2
<\infty,
\]
and \(\sup_{d_2\ge1}\|\mathcal D_2^{d_2}\|<\infty\), where \(\rho_v\) denotes
the marginal distribution of \(\rho\) on \(\mathcal V\).
\end{assumption}
Enlarging \(R\) if necessary, we also have
\(\sup_{d_1\ge1}\|(\mathcal E_1^{d_1}u,\gamma)\|_2\le R\)
\(\rho_u\)-almost surely. The condition on the output encoder is satisfied under a simple coordinate-decay
condition. It is enough that there exists a sequence
\(\{a_j\}_{j\ge1}\in\ell^2\) such that
\[
\operatorname*{ess\,sup}_{v\sim\rho_v}
\bigl|(\mathcal E_2^{d_2}v)_j\bigr|
\le a_j,
\qquad 1\le j\le d_2,\ d_2\ge1,
\]
together with \(\sup_{d_2\ge1}\|\mathcal D_2^{d_2}\|<\infty\). Equivalently, the encoded output coordinates are uniformly controlled by a
square-summable envelope.



Under Assumption~\ref{ass:output-encoding}, the preceding  NTK
analysis extends to the encoder--decoder neural operator setting.

\begin{theorem}\label{thm:NTK-encoder-decoder}
Suppose that Assumption~\ref{source condition} holds with \(r>0\), namely
\(\mathcal G^\dagger=L_{K^{\mathrm{NTK}}_{\infty}}^{\,r}\mathcal G^{\mathrm{src}}\)
for some \(\mathcal G^{\mathrm{src}}\in L^2(\mathcal U,\rho_u;\mathcal V)\), where
\(K^{\mathrm{NTK}}_{\infty}\) is defined in~\eqref{operator-valued NTK}. Suppose 
further that Assumptions~\ref{assumption: NTK} and~\ref{ass:output-encoding} hold.

Let \(\{\mathcal G_t^{\mathrm{NN}}\}_{t\ge1}\) be the neural operator iterates
generated by the encoder--decoder SGD update \eqref{eq:multi-output-SGD}.
Assume that the step-size, regularization, and width conditions in either
Theorem~\ref{NTK SGD1} or Theorem~\ref{NTK SGD2} are satisfied, according to the
chosen step-size regime. Then the corresponding bound in that theorem holds for
\(\mathcal G_t^{\mathrm{NN}}\), with the additional output-encoding error
\[
\mathbb E_{v\sim(\mathcal G^\dagger)_\#\rho_u}
\bigl[
\|(I_{\mathcal V}-P_2^{d_2})v\|_{\mathcal V}^2
\bigr]
\]
added to the right-hand side.
\end{theorem}

Under the conditions of Theorem~\ref{thm:NTK-encoder-decoder}, if the input and output encoding errors decay algebraically with the encoder dimensions, then both the neural network width (and hence the number of parameters) and the number of training samples required to achieve a target accuracy $\epsilon$ grow at most polynomially in $\epsilon^{-1}$.




\paragraph{Toward weaker output-side assumptions.}
A possible way to weaken Assumption~\ref{ass:output-encoding} is to use a fully
connected multi-output network instead of separate scalar subnetworks. 
For example, let
\[
f_{\Theta}(x)
=
\frac{1}{\sqrt{M}}\,A\,\sigma(B^\top x+\gamma c),
\qquad x\in\mathbb R^{d_1},
\]
where \(A=(a_1,\dots,a_M)\in\mathbb R^{d_2\times M}\) and
\(B=(b_1,\dots,b_M)\in\mathbb R^{d_1\times M}\). Here \(c\in\mathbb R^M\) is the hidden-layer bias vector,
and all entries of \(\Theta=(A,B,c)\) are trainable. Under a symmetric initialization with \(c=0\), 
\(b_1,\dots,b_{M/2}\stackrel{\mathrm{i.i.d.}}{\sim}\mathcal N(0,I_{d_1})\),
\(b_{r+M/2}=b_r\), and i.i.d. Rademacher vectors
\(a_1,\dots,a_{M/2}\in\mathbb R^{d_2}\) satisfying
\(a_{r+M/2}=-a_r\), the network output is again initialized at zero. 

In this setting, the population matrix-valued NTK on the encoded input space is
\[
\widetilde K^{\mathrm{NTK}}(x,x')
=
k^{\mathrm{NTK}}(x,x')\,I_{d_2}.
\]
At finite width, however, the empirical NTK \(\widetilde K\) is generally not
diagonal, since the output components share trainable hidden parameters. Although
\(\mathbb E[\widetilde K]=\widetilde K^{\mathrm{NTK}}\) over the initialization,
the corresponding SGD dynamics are not governed by a diagonal kernel of the form
\(k(\cdot,\cdot)I\). A complete treatment therefore requires extending the
present kernel-SGD analysis from diagonal matrix-valued kernels to general
matrix-valued kernels.

For encoder--decoder pairs satisfying
\(\mathcal D_2^{d_2}=(\mathcal E_2^{d_2})^*\),
the empirical matrix-valued kernel \(\widetilde K\)
can be lifted through the encoder--decoder maps, yielding the operator-valued kernel
\[
(u,u')\longmapsto
\mathcal D_2^{d_2}\,
\widetilde K\bigl(\mathcal E_1^{d_1}u,\mathcal E_1^{d_1}u'\bigr)\,
\mathcal E_2^{d_2}.
\]
This formulation treats the encoded output as a vector and offers a route to
replacing the separate coordinatewise control in
Assumption~\ref{ass:output-encoding} by a direct bound on
\(\|\mathcal E_2^{d_2}v\|_2\). The proof would follow the same general strategy
developed in this paper, extending the convergence estimates to general
matrix-valued empirical NTKs. We leave the detailed analysis of this extension
for future work.

\section{Proofs for Section~\ref{section 2.1}} \label{Sec 7}
Throughout this section, the identity operator $I$ denotes the identity on the
ambient space specified by the context. For $T\in\bn$, 
we consider the error $\mathbb{E}_{z^T}\bigl[\|\mathcal{G}_{T+1}-\mathcal{G}^{\dagger}\|^2_{\rho_u}\bigr]$ and begin with a preliminary decomposition:
\begin{equation} \label{first step}
   \be_{z^T}\l[\l\|\mg_{T+1}-\mg^{\dagger}\r\|^2_{\rho_u}\r]
   =\be_{z^T}\l[\l\|\mg_{T+1}-{P_2^{d_2}}\circ\mathcal{G}^{\dagger}\r\|^2_{\rho_u}\r]+
\l\|{P_2^{d_2}}\circ\mathcal{G}^{\dagger}-\mathcal{G}^{\dagger}\r\|^2_{\rho_u}.
\end{equation}
This identity holds since, for any $t\geq 0$, the range of $\mathcal{G}^{}_t$ is contained in the image of the orthogonal projection ${P_2^{d_2}}$. We therefore restrict attention to 
$\mathbb{E}_{z^T}\bigl[\|\mathcal{G}^{}_{T+1}-{P_2^{d_2}}\circ\mathcal{G}^{\dagger}\|^2_{\rho_u}\bigr]$. 

We introduce two regularized estimators of the projected target ${P_2^{d_2}}\circ\mathcal{G}^{\dagger}$, corresponding to the finite-resolution operator kernel and the projected limiting kernel, respectively:
\begin{equation} 
        \mathcal{G}^{}_\lambda:=\mathop{\arg\min}_{\mathcal{G}\in\H_{{d_1,d_2}}}\l\{\mathbb{E}\l[\l\|\mathcal{G}(u)-{P_2^{d_2}}v\r\|^2_{\mv}\r]+\lambda\|\mathcal{G}\|_{K_{d_1,d_2}}^2\r\},
\end{equation}
and
\begin{equation} 
        \mathcal{G}'_\lambda:=\mathop{\arg\min}_{\mathcal{G}\in\H_{{\infty,d_2}}}\l\{\mathbb{E}\l[\l\|\mathcal{G}(u)-{P_2^{d_2}}v\r\|^2_{\mv}\r]+\lambda\l\|\mathcal{G}\r\|_{{K_{\infty,d_2}}}^2\r\}.
\end{equation}
A direct computation yields
\begin{equation} \label{lambda1}
    \begin{aligned}
        \mathcal{G}^{}_\lambda&=\l(L_{{d_1,d_2}}+\lambda I\r)^{-1}L_{{d_1,d_2}}\mg^{\dagger} \\
        &=\l(\l(\l(L_{d_1}+\lambda I\r)^{-1}L_{d_1}\r)\otimes {P_2^{d_2}}\r)\mg^{\dagger},
    \end{aligned}
\end{equation}
and
\begin{equation} \label{lambda2}
    \begin{aligned}
        \mathcal{G}'_\lambda&=\l(L_{{\infty,d_2}}+\lambda I\r)^{-1}L_{{\infty,d_2}}\mg^{\dagger}
        \\
        &=\l(\l(\l(L_{\infty}+\lambda I\r)^{-1}L_{\infty}\r)\otimes {P_2^{d_2}}\r)\mg^{\dagger}.
    \end{aligned}
\end{equation}
The following decomposition, based on the intermediate regularized estimators defined above, plays a central role:
\begin{equation} \label{second step}
    \begin{aligned}
        \l\|\mg^{}_{T+1}-{P_2^{d_2}}\circ\mathcal{G}^{\dagger}\r\|^2_{\rho_u}\leq
    3\l\|\mg^{}_{T+1}-\mg^{}_{\lambda}\r\|^2_{\rho_u}
    +3\l\|\mg^{}_{\lambda}-\mg'_{\lambda}\r\|^2_{\rho_u}
    +3\l\|\mg'_{\lambda}-{P_2^{d_2}}\circ\mathcal{G}^{\dagger}\r\|_{\rho_u}^2.
    \end{aligned}
\end{equation}
Each of the four terms appearing in the decompositions \eqref{first step} and \eqref{second step} corresponds to a different contribution to the overall error:
\begin{itemize}
    \item $\l\|\mg^{}_{T+1}-\mg^{}_{\lambda}\r\|^2_{\rho_u}$: the deviation of the finite-step SGD iterate from its corresponding regularized estimator, measured in the $L^2(\mathcal{U},\rho_u;\mv)$ norm; we refer to this as the \emph{finite-iteration  error};
    \item $\l\|\mg^{}_{\lambda}-\mg'_{\lambda}\r\|^2_{\rho_u}$: the discrepancy induced by replacing the projected limiting kernel $K_{\infty,d_2}$ with the finite-resolution operator kernel $K_{d_1,d_2}$, referred to as the \emph{kernel approximation error};
    \item $\l\|\mg'_{\lambda}-{P_2^{d_2}}\circ\mathcal{G}^{\dagger}\r\|_{\rho_u}^2$: the error incurred when approximating the projected target within $\mathcal{H}_{\infty,d_2}$, referred to as the \emph{approximation error};
    \item $\l\|{P_2^{d_2}}\circ\mathcal{G}^{\dagger}-\mathcal{G}^{\dagger}\r\|^2_{\rho_u}$: the discrepancy arising solely from the output projection, referred to as the \emph{output-encoding error}.

\end{itemize}

In the analysis below, we derive bounds for the first three terms.
The contribution due to the output projection is handled separately.

\subsection{Estimation of the Finite-Iteration  Error}
We begin by deriving an explicit representation of the deviation $\mathcal{G}^{}_{T+1}-\mathcal{G}^{}_{\lambda}$ by unrolling the SGD recursion.

\begin{lemma} \label{lemma temp}
    For any $T\geq1$, the following identity holds:
    \begin{equation} \label{eq:one-step}
        \mg^{}_{T+1}-\mg^{}_{\lambda}=-\prod_{t=1}^T\l(I-\eta_t\l(\lambda I+L_{{d_1,d_2}}\r)\r)\mg^{}_{\lambda}+\sum_{t=1}^T\eta_t\prod_{j=t+1}^T\l(I-\eta_j\l(\lambda I+L_{{d_1,d_2}}\r)\r)\W_t,
    \end{equation}
    where 
    $\mathcal{W}_t:=L_{{d_1,d_2}}\l(\mathcal{G}^{}_t-{P_2^{d_2}}\circ\mg^{\dagger}\r)-k_{d_1}(\cdot, u_t){P_2^{d_2}}\l(\mathcal{G}^{}_t(u_t)-{P_2^{d_2}}(v_t)\r)$ satisfies $\be_{z_t\sim\rho}\l[\mathcal{W}_t\r]=0$.
\end{lemma}

\begin{proof}
    Using the update rule \eqref{SGD1} together with the identity $\l(L_{{d_1,d_2}}+\lambda I\r)\mathcal{G}^{}_\lambda=L_{{d_1,d_2}}\mg^{\dagger}$ from \eqref{lambda1}, we obtain, for any $t\geq 1$,
    \begin{equation*}
        \begin{aligned}
            \mg^{}_{t+1}-\mg^{}_{\lambda}&=
            \mathcal{G}^{}_t-\eta_t\left(k_{d_1}(\cdot, u_t){P_2^{d_2}}(\mathcal{G}^{}_t(u_t)-v_t)+\lambda \mathcal{G}^{}_t\right)-\mg^{}_{\lambda}
            \\&=\l(I-\eta_t\l(\lambda I+L_{{d_1,d_2}}\r)\r)\l(\mathcal{G}^{}_t-\mg^{}_{\lambda}\r)+\eta_t\mathcal{W}_t,
        \end{aligned}
    \end{equation*}
    where
    \begin{equation*}
        \begin{aligned}
            \mathcal{W}_t&=L_{{d_1,d_2}}\l(\mathcal{G}^{}_t-\mg^{\dagger}\r)-k_{d_1}(\cdot, u_t){P_2^{d_2}}(\mathcal{G}^{}_t(u_t)-v_t)
            \\&=L_{{{d_1,d_2}}}\l(\mathcal{G}^{}_t-{P_2^{d_2}}\circ\mg^{\dagger}\r)-k_{d_1}(\cdot, u_t){P_2^{d_2}}\l(\mathcal{G}^{}_t(u_t)-{P_2^{d_2}}(v_t)\r).
        \end{aligned}
    \end{equation*}
    By construction, $\mathcal{W}_t$ has zero mean under $z_t\sim\rho$. Here we have used the identity $L_{{d_1,d_2}}=L_{{d_1}}\otimes {P_2^{d_2}}$, which implies $L_{{d_1,d_2}}\mg^{\dagger}=L_{{d_1,d_2}}\l({P_2^{d_2}}\circ\mg^{\dagger}\r)$.
    Iterating the above recurrence from $t=1$ to $t=T$ yields \eqref{eq:one-step}.
\end{proof}

The following proposition builds on Lemma~\ref{lemma temp} to decompose the
finite-iteration  error into a contraction term arising from the deterministic part
of the iteration and a noise-induced term reflecting the accumulation of sampling fluctuations. Both terms will be controlled using the spectral
properties of $L_{d_1,d_2}$ and the step-size sequence
$\{\eta_t\}_{t\in\mathbb{N}_T}$.

\begin{proposition}[Decomposition of the finite-iteration  error] \label{prop3}
Let $T\geq1$, and let $\mathcal{G}^{}_{t+1}$ be generated by \eqref{SGD1} with step sizes $\eta_t>0$ and regularization parameter $\lambda>0$ for all $t\in\bn_T$. Then 
    \begin{equation*}
    \begin{aligned}
        \be_{z^T}\l[\l\|\mg^{}_{T+1}-\mg^{}_{\lambda}\r\|^2_{\rho_u}\r]\leq\T_1+\T_2,
    \end{aligned}
\end{equation*}
where 
\[
\T_1:=\l\|\prod_{t=1}^T\l(I-\eta_t\l(\lambda I+L_{{d_1,d_2}}\r)\r)\mg^{}_{\lambda}\r\|_{\rho_u}^2,
\]
and
\[
\T_2:=2\sum_{t=1}^T\eta_t^2\kappa^2\l(\be_{z^{t-1}}\l[\l\|\mathcal{G}^{}_t-{P_2^{d_2}}\circ\mathcal{G}^{\dagger}\r\|_{\rho_u}^2\r]+\sigma^2\r)\l\|L_{{d_1,d_2}}^{1/2}\prod_{j=t+1}^T\l(I-\eta_j\l(\lambda I+L_{d_1,d_2}\r)\r)\r\|^2.
\]
\end{proposition}

\begin{proof}
    By Lemma \ref{lemma temp}, we have
    \begin{equation*}
    \begin{aligned}
        \mathbb{E}_{z^T}\l\|\mg^{}_{T+1}-\mg^{}_{\lambda}\r\|^2_{\rho_u}\leq&\l\|\prod_{t=1}^T\l(I-\eta_t\l(\lambda I+L_{{d_1,d_2}}\r)\r)\mg^{}_{\lambda}\r\|_{\rho_u}^2
        \\&+\be_{z^T}\l[\l\|\sum_{t=1}^T\eta_t\prod_{j=t+1}^T\l(I-\eta_j\l(\lambda I+L_{d_1,d_2}\r)\r)\W_t\r\|_{\rho_u}^2\r].
    \end{aligned}
\end{equation*}
Expanding the squared norm, all mixed inner-product terms vanish by the zero-mean property of $\mathcal{W}_t$ together with the independence of $\mathcal{W}_{t'}$ and $z_t$ for $t>t'$. Consequently,
\[
\be_{z^T}\l[\l\|\sum_{t=1}^T\eta_t\prod_{j=t+1}^T\l(I-\eta_j\l(\lambda I+L_{{d_1,d_2}}\r)\r)\W_t\r\|_{\rho_u}^2\r]= \sum_{t=1}^T\eta_t^2\be_{z^t}\l[\l\|\prod_{j=t+1}^T\l(I-\eta_j\l(\lambda I+L_{d_1,d_2}\r)\r)\W_t\r\|_{\rho_u}^2\r].
\]
Writing $\mathcal{W}_t$ in centered form 
\[
\mathcal{W}_t=k_{d_1}(\cdot, u_t){P_2^{d_2}}\l({P_2^{d_2}}(v_t)-\mathcal{G}^{}_t(u_t)\r)-\be_{z_t\sim\rho}\l[k_{d_1}(\cdot, u_t){P_2^{d_2}}\l({P_2^{d_2}}(v_t)-\mathcal{G}^{}_t(u_t)\r)\r],
\]
we obtain
\begin{align*}
    &\be_{z^T}\l[\l\|\sum_{t=1}^T\eta_t\prod_{j=t+1}^T\l(I-\eta_j\l(\lambda I+L_{d_1,d_2}\r)\r)\W_t\r\|_{\rho_u}^2\r] \\&\quad\leq\sum_{t=1}^T\eta_t^2\be_{z^t}\l[\l\|\prod_{j=t+1}^T\l(I-\eta_j\l(\lambda I+L_{d_1,d_2}\r)\r)k_{d_1}(\cdot, u_t){P_2^{d_2}}\l(\mathcal{G}^{}_t(u_t)-{P_2^{d_2}}(v_t)\r)\r\|_{\rho_u}^2\r].
\end{align*}
Using \(\|h\|_{\rho_u}
=\|L_{d_1,d_2}^{1/2}h\|_{K_{d_1,d_2}}\)
for \(h\in\mathcal H_{d_1,d_2}\), we rewrite the right-hand side as
\[
\begin{aligned}
&\be_{z^t}\l[\l\|\prod_{j=t+1}^T\l(I-\eta_j\l(\lambda I+L_{d_1,d_2}\r)\r)k_{d_1}(\cdot, u_t){P_2^{d_2}}\l(\mathcal{G}^{}_t(u_t)-{P_2^{d_2}}(v_t)\r)\r\|_{\rho_u}^2\r]\\
=&
\be_{z^t}\l[\l\|L_{d_1,d_2}^{1/2}\prod_{j=t+1}^T\l(I-\eta_j\l(\lambda I+L_{d_1,d_2}\r)\r)k_{d_1}(\cdot, u_t){P_2^{d_2}}\l(\mathcal{G}^{}_t(u_t)-{P_2^{d_2}}(v_t)\r)\r\|_{K_{d_1,d_2}}^2\r]\\
\le&
2\be_{z^t}\l[\l\|L_{d_1,d_2}^{1/2}\prod_{j=t+1}^T\l(I-\eta_j\l(\lambda I+L_{d_1,d_2}\r)\r)k_{d_1}(\cdot, u_t){P_2^{d_2}}\l(\mathcal{G}^{}_t(u_t)-{P_2^{d_2}}\circ\mathcal{G}^{\dagger}(u_t)\r)\r\|_{K_{d_1,d_2}}^2\r]\\
&+ 2\be_{z^t}\l[\l\|L_{d_1,d_2}^{1/2}\prod_{j=t+1}^T\l(I-\eta_j\l(\lambda I+L_{d_1,d_2}\r)\r)k_{d_1}(\cdot, u_t){P_2^{d_2}}\l(\mathcal{G}^{\dagger}(u_t)-v_t\r)\r\|_{K_{d_1,d_2}}^2\r],
\end{aligned}
\]
where we used $\big(P_2^{d_2}\big)^2={P_2^{d_2}}$.

By the reproducing property of the operator-valued kernel
$K_{d_1,d_2}=k_{d_1}\cdot {P_2^{d_2}}$ and the assumption
$\sup_{u\in\mathrm{supp}(\rho_u)}k_{d_1}(u,u)\leq\kappa^2$, we have
\[
\|k_{d_1}(\cdot,u)\,{P_2^{d_2}}v\|_{K_{d_1,d_2}}^2
=\langle v,\,k_{d_1}(u,u)\,{P_2^{d_2}}v\rangle_{\mathcal{V}}
\le \|k_{d_1}(u,u){P_2^{d_2}}\|\,\|v\|_{\mathcal{V}}^2 \le \kappa^2 \|v\|_{\mathcal{V}}^2.
\]
which implies 
\begin{equation}\label{eq:kappa-bound}
\|k_{d_1}(\cdot,u_t)\,{P_2^{d_2}}\|\le \kappa.
\end{equation}
Combining \eqref{eq:kappa-bound} with $\be_{(u,v)\sim\rho}\l[\l\|\mathcal{G}^{\dagger}(u)-v\r\|_{\mv}^2\r]\leq\sigma^2$ yields the desired bound, and the proof is complete.
\end{proof}

\paragraph{Preliminary bounds for $\mathcal{T}_1$ and $\mathcal{T}_2$.}
We first estimate the deterministic term $\mathcal{T}_1$, which relies on norm bounds
for operator products of the form
$C^{\beta}\prod_{t=l}^{m}(I-\eta_t(C+\lambda I))$.

\begin{lemma}[Lemma 5.2 in \cite{yang2025learning}]  \label{lemma ref1}
		Let $\beta>0$ and $L$ be a positive operator with $\|L\|\leq\kappa^2$. Let $l,m$ be integers satisfying $1\leq l \leq m$. Suppose that $\eta_t(\kappa^2+\lambda)\leq1$ for any $l\leq t \leq m$. Then, the following estimates hold:
        \begin{itemize}
            \item[(1)] $\left\|L^\beta\prod_{t=l}^m\left(I-\eta_t(L+\lambda I)\right)\right\|\leq\exp\left\{-\lambda\sum_{t=l}^{m}\eta_{t}\right\}\frac{2(\kappa^{2\beta}+(\beta/e)^\beta)}{1+\left(\sum_{t=l}^{m}\eta_{t}\right)^\beta}.$
            \item[(2)] $\left\|L^\beta\prod_{t=l}^m\left(I-\eta_t(L+\lambda I)\right)^2\right\|
			\leq\left(\frac{\beta}{2e}\right)^{\beta}\left(\sum_{t=l}^{m}\eta_{t}\right)^{-\beta}\exp\left\{-2\lambda\sum_{t=l}^{m}\eta_{t}\right\}$.
            \item[(3)] $\left\|L^\beta\prod_{t=l}^m\left(I-\eta_t(L+\lambda I)\right)^2\right\|
		\leq\exp\left\{-2\lambda\sum_{t=l}^{m}\eta_{t}\right\}\frac{2(\kappa^{2\beta}+(\beta/(2e))^\beta)}{1+\left(\sum_{t=l}^{m}\eta_{t}\right)^\beta}.$
        \end{itemize}
	\end{lemma}

Lemma~\ref{lemma ref1} provides the norm estimate that underlies the following bound for $\mathcal{T}_1$. 

\begin{proposition} \label{T1}
    Let $T\geq2$. Suppose that Assumption~\ref{source condition} holds and  $\l(\kappa^2+\lambda\r)\eta_t\leq1$ for all $1\leq t\leq T$. Then
    \begin{equation*}
    \begin{aligned}
        \l\|\prod_{t=1}^T\l(I-\eta_t\l(\lambda I+L_{{d_1,d_2}}\r)\r)\mg^{}_{\lambda}\r\|_{\rho_u}^2\lesssim\;&\Delta^2\cdot\exp\left\{-2\lambda\sum_{t=1}^{T}\eta_{t}\right\}\l(\sum_{t=1}^T\frac{\eta_t}{1+\left(\sum_{i=1}^{t-1}\eta_{i}\right)^r}\r)^2
        \\& + \exp\left\{-2\lambda\sum_{t=1}^{T}\eta_{t}\right\}\l(1+\left(\sum_{t=1}^{T}\eta_{t}\right)^{2r}\r)^{-1}.
    \end{aligned}
\end{equation*}
Moreover, for \(0<r\le1\), writing
\(S_T:=\sum_{t=1}^T\eta_t\), we have
\[
\mathcal T_1\lesssim e^{-2\lambda S_T}
\left[\Delta^{2r}+\frac{1}{1+S_T^{2r}}\right].
\]
\end{proposition}

\begin{proof}
    By the representation of $\mathcal{G}^{}_\lambda$ in \eqref{lambda1} and Assumption~\ref{source condition}, we obtain
    \begin{equation*}
    \begin{aligned}
        \l\|\prod_{t=1}^T\l(I-\eta_t\l(\lambda I+L_{d_1,d_2}\r)\r)\mg^{}_{\lambda}\r\|_{\rho_u}^2&=\l\|\prod_{t=1}^T\l(I-\eta_t\l(\lambda I+L_{d_1,d_2}\r)\r)\l(L_{d_1,d_2}+\lambda I\r)^{-1}L_{d_1,d_2}\mg^{\dagger}\r\|_{\rho_u}^2
        \\&\leq\l\|\prod_{t=1}^T\l(I-\eta_t\l(\lambda I+L_{d_1,d_2}\r)\r)\l(L_{d_1,d_2}+\lambda I\r)^{-1}L_{d_1,d_2}L_{\infty,\infty}^r\r\|^2\l\|\mathcal{G}^{\mathrm{src}}\r\|_{\rho_u}^2.
    \end{aligned}
\end{equation*}
Since $L_{{d_1,d_2}}=L_{d_1}\otimes {P_2^{d_2}}$, $L_{\infty,\infty}=L_\infty\otimes I$, and $\l\|{P_2^{d_2}}\r\|=1$, there holds
\begin{equation} \label{temp5}
    \begin{aligned}
        \l\|\prod_{t=1}^T\l(I-\eta_t\l(\lambda I+L_{d_1,d_2}\r)\r)\mg^{}_{\lambda}\r\|_{\rho_u}^2&\leq\l\|\mathcal{G}^{\mathrm{src}}\r\|_{\rho_u}^2\l\|\prod_{t=1}^T\l(I-\eta_t\l(\lambda I+L_{d_1}\r)\r)\l(L_{d_1}+\lambda I\r)^{-1}L_{d_1}L_{\infty}^r\r\|^2
        \\&\leq\l\|\mathcal{G}^{\mathrm{src}}\r\|_{\rho_u}^2\l\|\prod_{t=1}^T\l(I-\eta_t\l(\lambda I+L_{d_1}\r)\r)L_{\infty}^r\r\|^2.
    \end{aligned}
\end{equation}
This reduces the problem to a purely input-side operator bound.
To bound the last term, define $f_t(L):=\prod_{i=1}^t \l(I-\eta_i(\lambda I+L)\r)$ for $t\geq 1$, with $f_0(L):=I$. 
By a telescoping argument, we obtain
\begin{equation*}
    \begin{aligned}
        f_T(L_{d_1})-f_T(L_{\infty}) &= \l(I-\eta_T\l(\lambda I+L_{d_1}\r)\r)f_{T-1}(L_{d_1}) - \l(I-\eta_T\l(\lambda I+L_{\infty}\r)\r)f_{T-1}(L_{\infty})
        \\ &= \l(I-\eta_T\l(\lambda I+L_{d_1}\r)\r)\l(f_{T-1}(L_{d_1})-f_{T-1}(L_{\infty})\r) - \eta_T\l(L_{d_1} - L_{\infty}\r)f_{T-1}(L_{\infty})
        \\ &= -\sum_{t=1}^T\eta_t\prod_{i=t+1}^T\l(I-\eta_i\l(\lambda I+L_{d_1}\r)\r)\l(L_{d_1}-L_{\infty}\r)f_{t-1}\l(L_{\infty}\r).
    \end{aligned}
\end{equation*}
Hence,
\begin{equation} \label{temp1}
    \begin{aligned}
        &\l\|\prod_{t=1}^T\l(I-\eta_t\l(\lambda I+L_{d_1}\r)\r)L_{\infty}^r - \prod_{t=1}^T\l(I-\eta_t\l(\lambda I+L_{\infty}\r)\r)L_{\infty}^r\r\| = \l\|\l(f_T(L_{d_1})-f_T(L_{\infty})\r)L_{\infty}^r\r\|
        \\ \leq& \Delta\cdot\sum_{t=1}^T\eta_t\l\|\prod_{i=t+1}^T\l(I-\eta_i\l(\lambda I+L_{d_1}\r)\r)\r\|\l\|\prod_{i=1}^{t-1}\l(I-\eta_i\l(\lambda I+L_{\infty}\r)\r)L_{\infty}^r\r\|.
    \end{aligned}
\end{equation}
By functional calculus for positive self-adjoint operators and the assumption $(\kappa^2+\lambda)\eta_i\leq 1$, one has
\begin{equation} \label{temp2}
    \begin{aligned}
        \l\|\prod_{i=t+1}^T\l(I-\eta_i\l(\lambda I+L_{d_1}\r)\r)\r\|\leq\exp\l\{{-\lambda\sum_{i=t+1}^T}\eta_i\r\}.
    \end{aligned}
\end{equation}
Applying \eqref{temp2} and Lemma~\ref{lemma ref1}(1) with $\beta=r$, $L=L_\infty$ to \eqref{temp1} yields
\begin{equation}\label{temp3}
    \begin{aligned}
        &\l\|\prod_{t=1}^T\l(I-\eta_t\l(\lambda I+L_{d_1}\r)\r)L_{\infty}^r - \prod_{t=1}^T\l(I-\eta_t\l(\lambda I+L_{\infty}\r)\r)L_{\infty}^r\r\|
        \\ \lesssim&\ \Delta\cdot\sum_{t=1}^T\eta_t\exp\left\{-\lambda\sum_{i=t+1}^{T}\eta_{i}\right\}\frac{\exp\left\{-\lambda\sum_{i=1}^{t-1}\eta_{i}\right\}}{1+\left(\sum_{i=1}^{t-1}\eta_{i}\right)^r}
        \\ \lesssim&\ \Delta\cdot\exp\left\{-\lambda\sum_{t=1}^{T}\eta_{t}\right\}\sum_{t=1}^T\frac{\eta_t}{1+\left(\sum_{i=1}^{t-1}\eta_{i}\right)^r}.
    \end{aligned}
\end{equation}

Finally, applying Lemma~\ref{lemma ref1}(1) with $\beta=r$, $L=L_\infty$ directly yields
\begin{equation} \label{temp4}
    \begin{aligned}
        \l\|\prod_{t=1}^T\l(I-\eta_t\l(\lambda I+L_{\infty}\r)\r)L_{\infty}^r\r\|\lesssim\frac{\exp\left\{-\lambda\sum_{t=1}^{T}\eta_{t}\right\}}{1+\left(\sum_{t=1}^{T}\eta_{t}\right)^r}
    \end{aligned}
\end{equation}

Combining \eqref{temp3}, \eqref{temp4}, and \eqref{temp5} establishes the first bound.

To prove the second bound, first observe that
\[
\begin{aligned}
\|f_T(L_{d_1})L_\infty\|
&=\|f_T(L_{d_1})L_{d_1}
+f_T(L_{d_1})(L_\infty-L_{d_1})\|\\
&\le \|f_T(L_{d_1})L_{d_1}\|
+\Delta\|f_T(L_{d_1})\|\\
&\lesssim
\exp\left\{-\lambda\sum_{t=1}^{T}\eta_t\right\}
\left[\frac{1}{1+\sum_{t=1}^{T}\eta_t}+\Delta\right].
\end{aligned}
\]
Here the last inequality follows from
Lemma~\ref{lemma ref1}(1) with $\beta=1$ and $L=L_{d_1}$,
together with
$\|f_T(L_{d_1})\|\le
\exp\{-\lambda\sum_{t=1}^{T}\eta_t\}$,
which follows by the same functional calculus argument as in
\eqref{temp2}. 

For a bounded operator \(F\), a bounded positive self-adjoint
operator \(H\), and \(0<r<1\), we have
\begin{equation}\label{eq:operator-moment}
\|FH^r\|\le \|F\|^{1-r}\|FH\|^r.
\end{equation}
Indeed, apply the spectral moment inequality
\(\|H^r x\|\le\|x\|^{1-r}\|Hx\|^r\)
to \(x=F^*y\), and take the supremum over \(\|y\|=1\).
Applying \eqref{eq:operator-moment} with
\(F=f_T(L_{d_1})\) and \(H=L_\infty\) gives
\[
\|f_T(L_{d_1})L_\infty^r\|
\lesssim e^{-\lambda S_T}
\left[\Delta+\frac1{1+S_T}\right]^r.
\]
Squaring and using \eqref{temp5} proves the second bound. 
The case \(r=1\) follows directly from the preceding estimate.
\end{proof}

For the term $\mathcal{T}_2$ in Proposition~\ref{prop3}, we apply
Lemma~\ref{lemma ref1}(3) with $\beta=1$ and $L=L_{d_1,d_2}$,
under the uniform second-moment bound stated in Proposition~\ref{T2},
where the constant $C>0$ will be specified later in the proof of the main theorems.
\begin{proposition} \label{T2}
    Assume that $(\kappa^2+\lambda)\eta_t\leq1$
for every $1\leq t\leq T$. Suppose that there exists a constant $C>0$ such that 
    \[
    \be_{z^{t-1}}\l[\l\|\mathcal{G}^{}_t-{P_2^{d_2}}\circ\mathcal{G}^{\dagger}\r\|_{\rho_u}^2\r] \leq C, \quad\forall t\in\bn_{T}.
    \]
    Then 
    \[
    \T_2\lesssim\l(C+1\r)\sum_{t=1}^T\eta_t^2\frac{\exp\left\{-2\lambda\sum_{j=t+1}^{T}\eta_{j}\right\}}{1+\sum_{j=t+1}^{T}\eta_{j}}.
    \]    
\end{proposition}

\begin{proof}
    The claim follows directly from Lemma~\ref{lemma ref1}(3) with $\beta=1$ and $L=L_{d_1,d_2}$, together with the assumed uniform bound.
\end{proof}

\paragraph{Explicit bounds for intermediate terms.} 
We derive explicit estimates for the auxiliary quantities
appearing in Propositions~\ref{T1} and~\ref{T2}, namely,
\[
\exp\left\{-2\lambda\sum_{t=1}^{T}\eta_{t}\right\},\quad \l(\sum_{t=1}^T\frac{\eta_t}{1+\left(\sum_{i=1}^{t-1}\eta_{i}\right)^r}\r)^2,\quad \left(\sum_{t=1}^{T}\eta_{t}\right)^{-2r},\quad \sum_{t=1}^T\eta_t^2\frac{\exp\left\{-2\lambda\sum_{j=t+1}^{T}\eta_{j}\right\}}{1+\sum_{j=t+1}^{T}\eta_{j}}.
\]
These estimates are carried out separately under two step-size regimes.
We begin with polynomially decaying step sizes.

\begin{proposition} \label{prop4}
    Let $r>0$. Let the step sizes be given by $\l\{\eta_t=\eta_1 t^{-\theta}:t\in\bn_T\r\}$, where $0<\eta_1\leq 1$ and $0<\theta<1$, and let $\lambda>0$. Then the following bounds hold:
    \begin{enumerate}[(1)]
        \item $\exp\left\{-2\lambda\sum_{t=1}^{T}\eta_{t}\right\}\leq\exp\l\{-2\frac{1-2^{\theta-1}}{1-\theta}\lambda\eta_1 T^{1-\theta}\r\}$.
        \item $\left(\sum_{t=1}^{T}\eta_{t}\right)^{-2r}\lesssim\eta_1^{-2r}T^{-2r(1-\theta)}$.
        \item If $\eta_1^rT^{r(1-\theta)}\leq1$, then
    \[
    \l(\sum_{t=1}^T\frac{\eta_t}{1+\left(\sum_{i=1}^{t-1}\eta_{i}\right)^r}\r)^2\;\lesssim\; \eta_1^2T^{2-2\theta}\;\lesssim\;1.
    \]
    If $\eta_1^rT^{r(1-\theta)}>1$, then
    \[
    \l(\sum_{t=1}^T\frac{\eta_t}{1+\left(\sum_{i=1}^{t-1}\eta_{i}\right)^r}\r)^2\;\lesssim\;\begin{cases}
            \eta_1^{2-2r}T^{(2-2r)(1-\theta)}, & \text{if } r<1, \\
            1+\l(\log \l(\eta_1^{1/(1-\theta)}T\r)\r)^2, & \text{if } r=1, \\
            1, & \text{if } r>1.
        \end{cases}
    \]
    \end{enumerate}
\end{proposition}

\begin{proof}
    The case $T=1$ is trivial, so we assume $T\ge 2$.  
    We first observe that
    \[
    \sum_{t=1}^{T}\eta_{t}\geq\eta_1\int_{1}^{T+1} x^{-\theta}\mathrm{d}x\geq\frac{\eta_1 \l((T+1)^{1-\theta}-1\r)}{1-\theta}
    \geq\frac{1-2^{\theta-1}}{1-\theta}\eta_1T^{1-\theta},
    \]
    where the last inequality follows from $T^{1-\theta}-1\geq (1-2^{\theta-1})T^{1-\theta}$ for $T\geq 2$. This establishes claims (1) and (2).
    
    To prove (3), we estimate
    \begin{equation*}
        \begin{aligned}
            \sum_{t=1}^T\frac{\eta_t}{1+\left(\sum_{i=1}^{t-1}\eta_{i}\right)^r}&\leq\eta_1\sum_{t=1}^T\frac{t^{-\theta}}{1+\eta_1^r\left(\int_{1}^{t}x^{-\theta}\mathrm{d}x\right)^r}
            \\&=\eta_1\sum_{t=1}^T\frac{t^{-\theta}}{1+\eta_1^r\left(\frac{t^{1-\theta}-1}{1-\theta}\right)^r}
            \\&\leq\l(\frac{1-\theta}{1-2^{\theta-1}}\r)^r\eta_1\l(1+\sum_{t=2}^T\frac{t^{-\theta}}{1+\eta_1^rt^{r(1-\theta)}}\r),
        \end{aligned}
    \end{equation*}
    where we again used $t^{1-\theta}-1\geq (1-2^{\theta-1})t^{1-\theta}$ for $t\geq 2$.  
    
    If $\eta_1^rT^{r(1-\theta)}\leq1$, then 
    \[
    \sum_{t=1}^T\frac{\eta_t}{1+\left(\sum_{i=1}^{t-1}\eta_{i}\right)^r}\;\lesssim\; \eta_1T^{1-\theta}\;\lesssim\;1.
    \]
    If $\eta_1^rT^{r(1-\theta)}>1$, then
    \begin{equation*}
        \begin{aligned}
            \sum_{t=1}^T\frac{\eta_t}{1+\left(\sum_{i=1}^{t-1}\eta_{i}\right)^r}
            \;&\leq \sum_{t=1}^{\l\lfloor\eta_1^{-1/(1-\theta)}\r\rfloor}\frac{\eta_t}{1+\left(\sum_{i=1}^{t-1}\eta_{i}\right)^r}+\eta_1^{1-r}\sum_{t=\l\lfloor\eta_1^{-1/(1-\theta)}\r\rfloor}^T t^{-\theta-r(1-\theta)}
            \\\;&\lesssim\;1+\eta_1^{1-r}\sum_{t=\l\lfloor\eta_1^{-1/(1-\theta)}\r\rfloor}^T t^{-\theta-r(1-\theta)}
            \\\;&\lesssim\;
        \begin{cases}
            \eta_1^{1-r}T^{(1-r)(1-\theta)}, & \text{if } r<1, \\
            1+\log \l(\eta_1^{1/(1-\theta)}T\r), & \text{if } r=1, \\
            1, & \text{if } r>1,
        \end{cases}
        \end{aligned}
    \end{equation*}
which completes the proof.
\end{proof}

We next estimate the remaining summation term appearing in $\mathcal T_2$.

\begin{proposition}\label{prop5}
Let the step sizes be given by $\l\{\eta_t=\eta_1 t^{-\theta}:t\in\bn_T\r\}$, where $0<\theta<1$ and $0<\eta_1\leq 1-\theta$, and let $\lambda>0$. Then
\begin{equation}\label{prop5_decay}
\begin{aligned}
&\sum_{t=1}^{T}\eta_t^2
\frac{
    \exp\left\{-2\lambda\sum_{j=t+1}^{T}\eta_j\right\}
}{
    1+\sum_{j=t+1}^{T}\eta_j
}
\lesssim
\eta_1(T+1)^{-\theta}
\begin{cases}
1+\log(1+\lambda^{-1}),
&0<\theta<\frac12,\\[1mm]
1+\log\left(1+(\lambda\eta_1)^{-1}\right),
&\theta=\frac12,\\[1mm]
1+(\lambda\eta_1)^{-\frac{2\theta-1}{1-\theta}},
&\frac12<\theta<1.
\end{cases}
\end{aligned}
\end{equation}
Moreover,
\begin{equation}\label{prop5_uniform}
\sum_{t=1}^{T}\eta_t^2
\frac{
    \exp\left\{-2\lambda\sum_{j=t+1}^{T}\eta_j\right\}
}{
    1+\sum_{j=t+1}^{T}\eta_j
}
\lesssim \eta_1.
\end{equation}
The implicit constants depend only on $\theta$.
In particular, for fixed $\lambda$, $\eta_1$, and $\theta$,
the sum in \eqref{prop5_decay} is $\mathcal{O}(T^{-\theta})$.
\end{proposition}

\begin{proof}
We follow the exponentially weighted summation argument in
Smale and Zhou \cite[Lemma~2]{smale2009online}, retaining the
denominator in the present setting.
For $T=1$, the sum equals $\eta_1^2$.
Since $\eta_1\leq 1$, both estimates follow.
Henceforth assume $T\geq 2$.

Since
\[
\sum_{j=t+1}^{T}j^{-\theta}
\geq
\int_{t+1}^{T+1}x^{-\theta}\,\mathrm{d}x
=
\frac{(T+1)^{1-\theta}-(t+1)^{1-\theta}}{1-\theta},
\]
we have, for $1\leq t\leq T$,
\begin{equation}\label{temp6}
\begin{aligned}
&\eta_t^2
\frac{
    \exp\left\{-2\lambda\sum_{j=t+1}^{T}\eta_j\right\}
}{
    1+\sum_{j=t+1}^{T}\eta_j
}
\leq
\eta_1^2t^{-2\theta}
\frac{
    \exp\left\{
    -\frac{2\lambda\eta_1}{1-\theta}
    \left((T+1)^{1-\theta}-(t+1)^{1-\theta}\right)
    \right\}
}{
    1+\frac{\eta_1}{1-\theta}
    \left((T+1)^{1-\theta}-(t+1)^{1-\theta}\right)
}.
\end{aligned}
\end{equation}

Split the sum at $t=\lfloor T/2\rfloor$, interpreting empty
sums as zero.
For $1\leq t\leq\lfloor T/2\rfloor-1$,
\[
(T+1)^{1-\theta}-(t+1)^{1-\theta}
\geq
(1-2^{\theta-1})(T+1)^{1-\theta}.
\]
Consequently,
\begin{equation}\label{temp7}
\begin{aligned}
&\sum_{t=1}^{\lfloor T/2\rfloor-1}
\eta_t^2
\frac{
    \exp\left\{-2\lambda\sum_{j=t+1}^{T}\eta_j\right\}
}{
    1+\sum_{j=t+1}^{T}\eta_j
}
\\
&\leq
\frac{(1-\theta)\eta_1}{1-2^{\theta-1}}
(T+1)^{\theta-1}
\exp\left\{
    -\frac{(2-2^\theta)\lambda\eta_1}{1-\theta}
    T^{1-\theta}
\right\}
\sum_{t=1}^{\lfloor T/2\rfloor-1}t^{-2\theta}
\\
&\lesssim
\eta_1
\exp\left\{
    -\frac{(2-2^\theta)\lambda\eta_1}{1-\theta}
    T^{1-\theta}
\right\}
\begin{cases}
(T+1)^{-\theta},
&0<\theta<\frac12,\\[1mm]
(T+1)^{-1/2}\log(T+1),
&\theta=\frac12,\\[1mm]
(T+1)^{\theta-1},
&\frac12<\theta<1.
\end{cases}
\end{aligned}
\end{equation}

For the remaining terms, we keep both the exponential factor
and the denominator in \eqref{temp6}.
For $\lfloor T/2\rfloor\leq t\leq T-1$ and
$u\in[t+1,t+2]$, we have
$u^{-2\theta}\geq 3^{-2\theta}t^{-2\theta}$.
Moreover, the exponential divided by the denominator on the
right-hand side of \eqref{temp6}, with $t+1$ replaced by $u$,
is nondecreasing in $u$.
Integrating over each interval $[t+1,t+2]$ and separating
the term $t=T$ therefore gives
\begin{equation*}
\begin{aligned}
&\sum_{t=\lfloor T/2\rfloor}^{T}
\eta_t^2
\frac{
    \exp\left\{-2\lambda\sum_{j=t+1}^{T}\eta_j\right\}
}{
    1+\sum_{j=t+1}^{T}\eta_j
}
\\
&\leq
3^{2\theta}\eta_1^2
\int_{\lfloor T/2\rfloor+1}^{T+1}
u^{-2\theta}
\frac{
    \exp\left\{
    -\frac{2\lambda\eta_1}{1-\theta}
    \left((T+1)^{1-\theta}-u^{1-\theta}\right)
    \right\}
}{
    1+\frac{\eta_1}{1-\theta}
    \left((T+1)^{1-\theta}-u^{1-\theta}\right)
}
\,\mathrm{d}u
+\eta_1^2T^{-2\theta}.
\end{aligned}
\end{equation*}
Make the change of variables
\[
\xi=
\frac{\eta_1}{1-\theta}
\left((T+1)^{1-\theta}-u^{1-\theta}\right),
\qquad
\mathrm{d}\xi=-\eta_1u^{-\theta}\,\mathrm{d}u.
\]
Using
$u\geq\lfloor T/2\rfloor+1\geq(T+1)/2$
and enlarging the upper integration limit, we obtain
\[
\begin{aligned}
&\sum_{t=\lfloor T/2\rfloor}^{T}
\eta_t^2
\frac{
    \exp\left\{-2\lambda\sum_{j=t+1}^{T}\eta_j\right\}
}{
    1+\sum_{j=t+1}^{T}\eta_j
}
\\
&\leq
3^{2\theta}2^\theta\eta_1(T+1)^{-\theta}
\int_0^{\frac{\eta_1}{1-\theta}(T+1)^{1-\theta}}
\frac{e^{-2\lambda\xi}}{1+\xi}\,\mathrm{d}\xi
+\eta_1^2T^{-2\theta}.
\end{aligned}
\]
Since $\eta_1\leq 1-\theta$, the integral satisfies
\[
\begin{aligned}
\int_0^{\frac{\eta_1}{1-\theta}(T+1)^{1-\theta}}
\frac{e^{-2\lambda\xi}}{1+\xi}\,\mathrm{d}\xi
&\leq
\log\left(1+(T+1)^{1-\theta}\right)
\\
&\lesssim\log(T+1).
\end{aligned}
\]
It also satisfies the bound
\[
\begin{aligned}
\int_0^\infty
\frac{e^{-2\lambda\xi}}{1+\xi}\,\mathrm{d}\xi
&\leq
\int_0^{1/(2\lambda)}\frac{1}{1+\xi}\,\mathrm{d}\xi
+2\lambda\int_{1/(2\lambda)}^\infty
e^{-2\lambda\xi}\,\mathrm{d}\xi
\\
&=
\log\left(1+\frac{1}{2\lambda}\right)+e^{-1}
\\
&\leq 1+\log(1+\lambda^{-1}).
\end{aligned}
\]
Finally,
$\eta_1^2T^{-2\theta}\lesssim\eta_1(T+1)^{-\theta}$,
and
$\min\{\log(T+1),1+\log(1+\lambda^{-1})\}\geq\log 2$.
Hence
\begin{equation}\label{temp9}
\begin{aligned}
&\sum_{t=\lfloor T/2\rfloor}^{T}
\eta_t^2
\frac{
    \exp\left\{-2\lambda\sum_{j=t+1}^{T}\eta_j\right\}
}{
    1+\sum_{j=t+1}^{T}\eta_j
}
\\
&\qquad\lesssim
\eta_1(T+1)^{-\theta}
\min\left\{
\log(T+1),\,
1+\log(1+\lambda^{-1})
\right\}.
\end{aligned}
\end{equation}
We now combine \eqref{temp7} and \eqref{temp9}.
If $0<\theta<1/2$, bounding the exponential in
\eqref{temp7} by one directly gives \eqref{prop5_decay}.

If $\theta=1/2$, use
\[
\begin{aligned}
\log(T+1)
&\leq 2\log(1+\sqrt{T})
\\
&\leq
2\log\left(1+(\lambda\eta_1)^{-1}\right)
+2\log(1+\lambda\eta_1\sqrt{T}).
\end{aligned}
\]
Since
$\sup_{x\geq0}e^{-(4-2\sqrt{2})x}\log(1+x)<\infty$,
it follows that
\[
e^{-(4-2\sqrt{2})\lambda\eta_1\sqrt{T}}\log(T+1)
\lesssim
1+\log\left(1+(\lambda\eta_1)^{-1}\right).
\]
Together with $\eta_1\leq1$ and \eqref{temp9},
this proves \eqref{prop5_decay} when $\theta=1/2$.

If $1/2<\theta<1$, we use the inequality
\[
e^{-cx}
\leq
\left(\frac{\alpha}{ec}\right)^\alpha x^{-\alpha},
\qquad x>0,\quad c>0,\quad \alpha>0.
\]
Applying this inequality with
\[
x=T^{1-\theta},
\qquad
c=\frac{(2-2^\theta)\lambda\eta_1}{1-\theta},
\qquad
\alpha=\frac{2\theta-1}{1-\theta},
\]
and using $T+1\leq 2T$, we obtain
\[
\begin{aligned}
&(T+1)^{2\theta-1}
\exp\left\{
    -\frac{(2-2^\theta)\lambda\eta_1}{1-\theta}
    T^{1-\theta}
\right\}
\\
&\quad\leq
2^{2\theta-1}
\left(
    \frac{2\theta-1}{e(2-2^\theta)}
\right)^{\frac{2\theta-1}{1-\theta}}
(\lambda\eta_1)^{-\frac{2\theta-1}{1-\theta}}
\\
&\quad\lesssim
(\lambda\eta_1)^{-\frac{2\theta-1}{1-\theta}}.
\end{aligned}
\]
Also, since $\eta_1\leq1$ and logarithmic growth is dominated
by any positive power,
\[
1+\log(1+\lambda^{-1})
\lesssim
1+(\lambda\eta_1)^{-\frac{2\theta-1}{1-\theta}}.
\]
Combining these estimates with \eqref{temp7} and
\eqref{temp9} proves \eqref{prop5_decay}.

To prove \eqref{prop5_uniform}, bound the exponential in
\eqref{temp7} by one and use the $\log(T+1)$ bound in
\eqref{temp9}.
For each fixed $\theta\in(0,1)$, all of
\[
(T+1)^{-\theta},\qquad
(T+1)^{-1/2}\log(T+1),\qquad
(T+1)^{\theta-1},\qquad
(T+1)^{-\theta}\log(T+1)
\]
are bounded uniformly for $T\geq1$.
This gives \eqref{prop5_uniform} and completes the proof.
\end{proof}

We next turn to the case of constant step sizes.
\begin{proposition} \label{lemma 1}
Let $r>0$. Let $\eta_t\equiv \eta_1$ with $0<\eta_1\le 1$ and $\lambda>0$. Then, for any integer $t\ge 1$,
\[
\sum_{k=1}^t\frac{\eta_k}{1+\Bigl(\sum_{i=1}^{k-1}\eta_i\Bigr)^r}
\;\lesssim\;
\begin{cases}
\eta_1 t, & \text{if }\eta_1(t-1)\le 1,\\[2mm]
\eta_1^{1-r}t^{1-r}, & \text{if }\eta_1(t-1)>1,\ 0<r<1,\\
1+\log(\eta_1 t), & \text{if }\eta_1(t-1)>1,\ r=1,\\
1, & \text{if }\eta_1(t-1)>1,\ r>1.
\end{cases}
\]
\end{proposition}

\begin{proof}
		By rewriting the sum, we obtain
		\begin{align*}
			\sum_{k=1}^t\frac{\eta_k}{1+\left(\sum_{i=1}^{k-1}\eta_{i}\right)^r}
			&=\eta_{1}\sum_{k=1}^{t}\frac{1}{1+\left((k-1)\eta_{1}\right)^r}=\eta_{1}\l(1+\sum_{k=1}^{t-1}\frac{1}{1+\left(k\eta_{1}\right)^r}\r)	
			\\ &\leq
			\eta_{1}\l(1+\int_{0}^{t-1}\frac{1}{1+\left(u\eta_{1}\right)^r}\mathrm{d}u\r)	
            =\eta_{1}\l(1+\frac{1}{\eta_1}\int_{0}^{\eta_1(t-1)}\frac{1}{1+u^r}\mathrm{d}u\r)
            \\ &\lesssim
            \begin{cases}
                \eta_1t,& \text{ if }\eta_1(t-1)\leq1, \\
                1 + \int_{1}^{\eta_1(t-1)}u^{-r}\mathrm{d}u, & \text{ if }\eta_1(t-1)>1.
            \end{cases}
		\end{align*}
		When $\eta_1(t-1)>1$, evaluating the integral $\int_{1}^{\eta_1(t-1)} u^{-r}\,\mathrm{d}u$
yields
		\begin{align*}
			\sum_{k=1}^t\frac{\eta_k}{1+\left(\sum_{i=1}^{k-1}\eta_{i}\right)^r}
			&\lesssim\begin{cases}
				1+\frac{\eta_1^{1-r}t^{1-r}}{1-r}, & \text{if } 0<r<1, \\ 
				1+\log\left(\eta_{1}(t-1)\right), &\text{if }  r=1, \\
				\frac{r}{r-1}, &\text{if }  r>1,
			\end{cases}
            \\&\lesssim\begin{cases}
				\eta_1^{1-r}t^{1-r}, & \text{if } 0<r<1, \\ 
				1+\log\left(\eta_{1}t\right), &\text{if }  r=1, \\
				1, &\text{if }  r>1.
			\end{cases}
		\end{align*}	
    This completes the proof.
	\end{proof}

\begin{proposition} \label{prop9}
Fix $T\ge 1$ and let the step sizes be constant over iterations,
$\eta_t\equiv \eta_1$, with $\eta_1=\eta T^{-\theta'}$, where $0<\eta\le 1$ and
$0<\theta'<1$. Then, for any $\lambda>0$,
\[
\sum_{t=1}^{T}\eta_t^2\,\frac{\exp\!\bigl\{-2\lambda\sum_{j=t+1}^{T}\eta_{j}\bigr\}}{1+\sum_{j=t+1}^{T}\eta_{j}}
\;\lesssim\;
\eta \,T^{-\theta'}\min\bigl\{\log(T+1),\,1+\log(1+\lambda^{-1})\bigr\}.
\]
\end{proposition}

\begin{proof} 
    Substituting the constant step size $\eta_t=\eta_1$ yields
    \begin{equation*}
        \begin{aligned}
            \sum_{t=1}^T\eta_t^2\frac{\exp\left\{-2\lambda\sum_{j=t+1}^{T}\eta_{j}\right\}}{1+\sum_{j=t+1}^{T}\eta_{j}}&=
            \eta_1^2\sum_{t=1}^T\frac{\exp\left\{-2\lambda(T-t)\eta_1\right\}}{1+(T-t)\eta_1}
            \\& = \eta_1^2\sum_{t=0}^{T-1}\frac{\exp\left\{-2\lambda t\eta_1\right\}}{1+t\eta_1}.
        \end{aligned}
    \end{equation*}
    By monotonicity and a change of variables, we obtain
\begin{equation}
\begin{aligned}
\sum_{t=0}^{T-1}\frac{\exp\{-2\lambda t\eta_1\}}{1+t\eta_1}
&\le 1+\int_0^{T-1}
\frac{\exp\{-2\lambda\eta_1 x\}}{1+\eta_1 x}\,\mathrm{d}x=1+\frac{1}{\eta_1}\int_0^{\eta_1(T-1)}
\frac{\exp\{-2\lambda x\}}{1+x}\,\mathrm{d}x.
\end{aligned}
\end{equation}
The integral is at most
$\log(1+\eta_1(T-1))\le\log T$.
It is also bounded by
\[
\begin{aligned}
\int_0^\infty\frac{\exp\{-2\lambda x\}}{1+x}\,\mathrm{d}x
&\le \int_0^{1/(2\lambda)}\frac{1}{1+x}\,\mathrm{d}x
+2\lambda\int_{1/(2\lambda)}^\infty \exp\{-2\lambda x\}\,\mathrm{d}x\\
&=\log\left(1+\frac{1}{2\lambda}\right)+\exp\{-1\}\\
&\le 1+\log(1+\lambda^{-1}).
\end{aligned}
\]
    Finally, using $\eta_1^2\le\eta_1$ and substituting
$\eta_1=\eta T^{-\theta'}$ yields
    \begin{equation*}
        \begin{aligned}
            \sum_{t=1}^T\eta_t^2\frac{\exp\left\{-2\lambda\sum_{j=t+1}^{T}\eta_{j}\right\}}{1+\sum_{j=t+1}^{T}\eta_{j}}
            &\lesssim\eta T^{-\theta'}\min\bigl\{\log(T+1),\,1+\log(1+\lambda^{-1})\bigr\},
        \end{aligned}
    \end{equation*}
which completes the proof.
\end{proof}

\subsection{Estimation of the Kernel Approximation Error}
We next consider the term $\l\|\mg^{}_{\lambda}-\mg'_{\lambda}\r\|^2_{\rho_u}$, which quantifies the discrepancy between the regularized estimators associated with the finite-resolution operator kernel and the projected limiting kernel.
\begin{proposition} \label{prop1}
Suppose that Assumption~\ref{source condition} holds. Then, for any $\lambda>0$, it holds that
\begin{equation*}
    \l\|\mg^{}_{\lambda}-\mg'_{\lambda}\r\|^2_{\rho_u}\lesssim
    \Delta^{2\min\{r,1\}}.    
\end{equation*}
\end{proposition}

\begin{proof}
    
Using the explicit expressions of $\mg^{}_{\lambda}$ and $\mg'_{\lambda}$ in \eqref{lambda1}--\eqref{lambda2}, together with Assumption~\ref{source condition}, we obtain
\begin{equation*}
    \begin{aligned}
        \l\|\mg^{}_{\lambda}-\mg'_{\lambda}\r\|^2_{\rho_u}&=
        \l\|\l(\l(\l(L_{d_1}+\lambda I\r)^{-1}L_{d_1}-\l(L_{\infty}+\lambda I\r)^{-1}L_{\infty}\r)\otimes {P_2^{d_2}}\r)\mg^{\dagger}\r\|^2_{\rho_u}
        \\&=\l\|\l(\l(\l(L_{d_1}+\lambda I\r)^{-1}L_{d_1}-\l(L_{\infty}+\lambda I\r)^{-1}L_{\infty}\r)L_{\infty}^r\otimes {P_2^{d_2}}\r)\mg^{\mathrm{src}}\r\|^2_{\rho_u}
        \\&\leq\l\|\l(\l(L_{d_1}+\lambda I\r)^{-1}L_{d_1}-\l(L_{\infty}+\lambda I\r)^{-1}L_{\infty}\r)L_{\infty}^r\r\|^2\l\|\mg^\mathrm{src}\r\|_{\rho_u}^2.
    \end{aligned}
\end{equation*}
Since
\[
(L_{d_1}+\lambda I)^{-1}L_{d_1}=I-\lambda(L_{d_1}+\lambda I)^{-1},\quad
(L_{\infty}+\lambda I)^{-1}L_{\infty}=I-\lambda(L_{\infty}+\lambda I)^{-1},
\]
it follows that
\begin{equation*}
    \begin{aligned}
        \l\|\mg^{}_{\lambda}-\mg'_{\lambda}\r\|^2_{\rho_u}&\leq \lambda^2\l\|\mg^\mathrm{src}\r\|_{\rho_u}^2
        \l\|\l(\l(L_{d_1}+\lambda I\r)^{-1}-\l(L_{\infty}+\lambda I\r)^{-1}\r)L_{\infty}^r\r\|^2
        \\&=\lambda^2\l\|\mg^\mathrm{src}\r\|_{\rho_u}^2
        \l\|\l(L_{d_1}+\lambda I\r)^{-1}(L_{d_1}-L_{\infty})\l(L_{\infty}+\lambda I\r)^{-1}L_{\infty}^r\r\|^2
        \\&\leq\l\|\mg^\mathrm{src}\r\|_{\rho_u}^2\Delta^2\l\|\l(L_{\infty}+\lambda I\r)^{-1}L_{\infty}^r\r\|^2.
    \end{aligned}
\end{equation*}
By functional calculus for positive self-adjoint operators, we have
    \begin{equation} \label{functional calculus}
        \begin{aligned}
            \l\|\l(L_{\infty}+\lambda I\r)^{-1}L_{\infty}^r\r\|&\leq\sup_{0\leq x\leq\kappa^2}\frac{x^r}{x+\lambda}
            \\&\leq
            \begin{cases}
                \kappa^{2r-2}, &\text{ if }r\geq1, \\
                r^r(1-r)^{1-r}\lambda^{r-1}, &\text{ if }r<1.
            \end{cases}
        \end{aligned}
    \end{equation}
This proves the claim for \(r\ge1\). For \(0<r<1\), set
\[
C_\lambda:=
(L_{d_1}+\lambda I)^{-1}L_{d_1}
-(L_\infty+\lambda I)^{-1}L_\infty.
\]
Since \(C_\lambda\) is the difference of two positive
contractions, \(\|C_\lambda\|\le1\). The resolvent identity gives
\[
\|C_\lambda L_\infty\|
=
\|\lambda(L_{d_1}+\lambda I)^{-1}
(L_{d_1}-L_\infty)(L_\infty+\lambda I)^{-1}L_\infty\|
\le\Delta.
\]
Thus \eqref{eq:operator-moment} yields
\(\|C_\lambda L_\infty^r\|\le\Delta^r\),
which proves the claim.
\end{proof}

\subsection{Estimation of the Approximation Error}
We now bound the approximation error
$\|\mathcal{G}'_{\lambda}-{P_2^{d_2}}\circ\mathcal{G}^{\dagger}\|_{\rho_u}^2$.
\begin{proposition} \label{prop2}
    Suppose that Assumption~\ref{source condition} holds. Then, for any $\lambda>0$, it holds that
    \begin{equation*}
        \l\|\mg'_{\lambda}-{P_2^{d_2}}\circ\mathcal{G}^{\dagger}\r\|_{\rho_u}^2 \lesssim\lambda^{\min\l\{2r,2\r\}}.
    \end{equation*}
\end{proposition}

\begin{proof}
    Using the representation of $\mg'_{\lambda}$ in \eqref{lambda2} together with Assumption~\ref{source condition}, we obtain
    \begin{equation*}
    \begin{aligned}
        \l\|\mg'_{\lambda}-{P_2^{d_2}}\circ\mathcal{G}^{\dagger}\r\|_{\rho_u}^2        &=\l\|\l(\l(\l(L_{\infty}+\lambda I\r)^{-1}L_{\infty}-I\r)\otimes {P_2^{d_2}}\r)\mg^{\dagger}\r\|^2_{\rho_u}
        \\&=\lambda^2\l\|\l(\l(L_{\infty}+\lambda I\r)^{-1}\otimes {P_2^{d_2}}\r)\l(L_{\infty}^r\otimes I_{\mv}\r)\mathcal{G}^{\mathrm{src}}\r\|^2_{\rho_u}
        \\&\leq\lambda^2\l\|\mathcal{G}^{\mathrm{src}}\r\|^2_{\rho_u}\l\|\l(L_{\infty}+\lambda I\r)^{-1}L_{\infty}^r\r\|^2.
    \end{aligned}
\end{equation*}
By the functional calculus estimate in \eqref{functional calculus}, it follows that
\[
\|\mg'_{\lambda}-{P_2^{d_2}}\circ\mathcal{G}^{\dagger}\|_{\rho_u}^2
\;\lesssim\;\lambda^{\min\{2r,2\}},
\]
which completes the proof.
\end{proof}

\subsection{Proof of the Theorems in Subsection \ref{subsection: upper bound}}
We now combine the preceding error estimates to complete the proofs of the main results
in Subsection~\ref{subsection: upper bound}. One remaining ingredient is the uniform bound
\[
\mathbb{E}_{z^{t-1}}\!\left[\left\|\mathcal{G}_{t}-{P_2^{d_2}}\circ\mathcal{G}^{\dagger}\right\|_{\rho_u}^2\right]\le C,
\quad t\in\mathbb{N}_T,
\]
which was used in Proposition~\ref{T2} to control the term $\mathcal{T}_2$.
We show below that this bound indeed holds under the conditions underlying the results of
Subsection~\ref{subsection: upper bound}.


Once this uniform bound is established, the proofs of
Theorems~\ref{Thm1} and~\ref{Thm2} follow by combining the preceding error
decomposition with the corresponding estimates derived for each term. Theorem~\ref{Thm1} addresses the regime of decreasing step sizes, while
Theorem~\ref{Thm2} concerns the finite-horizon setting with constant step sizes.
In each case, the stated bounds follow by substituting the corresponding
component-wise estimates into the general error decomposition.

\begin{proposition} \label{prop6}
    Let the step sizes be decreasing, $\l\{\eta_t=\eta_1t^{-\theta}:\forall t\geq1\r\}$ with $0<\theta<1$ and $0<\eta_1\le 1-\theta$, 
    and let $0<\lambda\le1$.
    Then, there exists 
    a
constant $\bar\eta>0$, independent of $\Delta$, $\lambda$,
and $T$, such that for all
$0<\eta_1\leq \bar\eta$, the following bound holds:
\begin{equation*}
    \be_{z^{t}}\l[\l\|\mathcal{G}^{}_{t+1}-{P_2^{d_2}}\circ\mathcal{G}^{\dagger}\r\|_{\rho_u}^2\r]
    \;\lesssim\;1,
    \quad \forall\, t\ge0.
\end{equation*}
\end{proposition}

\begin{proof}
    We prove by induction a uniform bound of the form $\mathbb{E}_{z^{t}}[\|\mathcal{G}^{}_{t+1}-{P_2^{d_2}}\circ\mathcal{G}^{\dagger}\|_{\rho_u}^2]\le C$ for all $t\geq0$, where $C$ will be chosen so that the bound is self-consistent. Let $C\geq\max\{1,\|\mathcal{G}^\dagger\|_{\rho_u}^2\}$, to be enlarged below if necessary. Take $\bar\eta\le(\kappa^2+1)^{-1}$, so that
$\eta_k(\kappa^2+\lambda)\le1$ for every $k\ge1$.

For $t=0$, we have 
\[
\|\mathcal{G}^{}_{1}-{P_2^{d_2}}\circ\mathcal{G}^{\dagger}\|_{\rho_u}^2
= \|{P_2^{d_2}}\circ\mathcal{G}^{\dagger}\|_{\rho_u}^2
\le \|\mathcal{G}^{\dagger}\|_{\rho_u}^2\leq C,\]
and the claim holds.

For $t=1$, the update gives
$\mathcal{G}_2=\eta_1k_{d_1}(\cdot,u_1)P_2^{d_2}v_1$.
Since $\eta_1\leq1$,
\[
\mathbb{E}_{z^1}
\left[\|\mathcal{G}_2-P_2^{d_2}\circ \mathcal{G}^\dagger\|_{\rho_u}^2\right]
\leq
2\kappa^4\mathbb{E}_{v\sim\rho_v}\|v\|_{\mathcal V}^2
+2\|\mathcal{G}^\dagger\|_{\rho_u}^2.
\]
We choose $C$ large enough to cover this bound.

For $t\geq2$, assume that
\[
\mathbb{E}_{z^{k-1}}
\left[\|\mathcal{G}_k-P_2^{d_2}\circ \mathcal{G}^\dagger\|_{\rho_u}^2\right]
\leq C,\qquad 1\leq k\leq t.
\]
Equation~\eqref{lambda1} also gives
\[
\mg_\lambda-P_2^{d_2}\circ\mg^\dagger
=-\bigl(\lambda(L_{d_1}+\lambda I)^{-1}
\otimes P_2^{d_2}\bigr)\mg^\dagger.
\]
Since $x/(x+\lambda)$ and $\lambda/(x+\lambda)$ belong to
$[0,1]$ for $x\ge0$, functional calculus yields
\[
\max\left\{
\|\mg_\lambda\|_{\rho_u},
\|\mg_\lambda-P_2^{d_2}\circ\mg^\dagger\|_{\rho_u}
\right\}
\le \|\mg^\dagger\|_{\rho_u}.
\]
At horizon $t$, $\T_1\le\|\mg^\dagger\|_{\rho_u}^2$.
The two-term decomposition through $\mg_\lambda$ and
Proposition~\ref{prop3} therefore give
\[
\begin{aligned}
\be_{z^t}\|\mg_{t+1}-P_2^{d_2}\circ\mg^\dagger\|_{\rho_u}^2
&\le 2\T_1+2\T_2
+2\|\mg_\lambda-P_2^{d_2}\circ\mg^\dagger\|_{\rho_u}^2\\
&\le4\|\mg^\dagger\|_{\rho_u}^2+2\T_2.
\end{aligned}
\]
By Proposition~\ref{T2}, \eqref{prop5_uniform}, and the
induction hypothesis, $\T_2\lesssim\eta_1(C+1)$.
Thus the preceding right-hand side is at most
\[
4\|\mg^\dagger\|_{\rho_u}^2+K\eta_1(C+1),
\]
where $K>0$ is independent of $\Delta,\lambda,t,\eta_1$ and $C$.
Choose $C\ge8\|\mg^\dagger\|_{\rho_u}^2+1$, also large enough
to cover the base cases, and further require
$\bar\eta\le(2K)^{-1}$.
The last expression is then bounded by $C$,
which closes the induction.


\end{proof}

\begin{proof}[Proof of Theorem \ref{Thm1}]
Let $C\ge1$ be a uniform bound supplied by
Proposition~\ref{prop6}. The case $t=1$ follows from Proposition~\ref{prop6}, hence assume $t\ge2$. 
We combine the error decomposition \eqref{second step} together with
Proposition~\ref{prop3} (decomposition of the finite-iteration  error), the bounds for the deterministic
term $\mathcal{T}_1$ from Propositions~\ref{T1} and~\ref{prop4}, the bounds for the term $\mathcal{T}_2$ from Propositions~\ref{T2} and~\ref{prop5}, and the
kernel/approximation error bounds from Propositions~\ref{prop1}--\ref{prop2}.
By Proposition~\ref{T2} and the uniform estimate~\eqref{prop5_uniform}
in Proposition~\ref{prop5}, the stochastic term at horizon $t$
satisfies
\[
T_2
\lesssim
(C+1)\sum_{k=1}^{t}\eta_k^2
\frac{
\exp\left\{-2\lambda\sum_{j=k+1}^{t}\eta_j\right\}
}{
1+\sum_{j=k+1}^{t}\eta_j
}
\lesssim \eta_1(C+1).
\]
Collecting all contributions yields

\begin{equation} \label{temp11}
            \begin{aligned}
                \be_{z^t}\l[\l\|\mg^{}_{t+1}-{P_2^{d_2}}\circ\mathcal{G}^{\dagger}\r\|^2_{\rho_u}\r]\lesssim\;&\Delta^{2\min\{r,1\}}+\lambda^{\min\l\{2r,2\r\}}
                \\&+\Delta^2\exp\l\{-2\frac{1-2^{\theta-1}}{1-\theta}\lambda\eta_1 t^{1-\theta}\r\}\l(\sum_{k=1}^t\frac{\eta_k}{1+\left(\sum_{i=1}^{k-1}\eta_{i}\right)^r}\r)^2
        \\&+\exp\l\{-2\frac{1-2^{\theta-1}}{1-\theta}\lambda\eta_1 t^{1-\theta}\r\}\l(1+\left(\sum_{k=1}^{t}\eta_{k}\right)^{2r}\r)^{-1}
        \\& +\eta_1(C+1).
        \end{aligned}
        \end{equation} 
For \(r>1\), Proposition~\ref{prop4}(3) yields the following
estimates. When $\eta_1^rt^{r(1-\theta)}\le1$,
\begin{equation} \label{temp17}
    \Delta^2\exp\l\{-2\frac{1-2^{\theta-1}}{1-\theta}\lambda\eta_1 t^{1-\theta}\r\}\l(\sum_{k=1}^t\frac{\eta_k}{1+\left(\sum_{i=1}^{k-1}\eta_{i}\right)^r}\r)^2\lesssim\Delta^2,
\end{equation}
and when $\eta_1^rt^{r(1-\theta)}>1$,
\begin{equation} \label{temp18}
    \begin{aligned}
        &\Delta^2\exp\l\{-2\frac{1-2^{\theta-1}}{1-\theta}\lambda\eta_1 t^{1-\theta}\r\}\l(\sum_{k=1}^t\frac{\eta_k}{1+\left(\sum_{i=1}^{k-1}\eta_{i}\right)^r}\r)^2
        \\\lesssim\;&\Delta^2\exp\l\{-2\frac{1-2^{\theta-1}}{1-\theta}\lambda\eta_1 t^{1-\theta}\r\}
        \lesssim\;\Delta^2
    \end{aligned}
\end{equation}
where the last step uses
$\exp(-k_1x)\le(k_2/(ek_1))^{k_2}x^{-k_2}$
for $k_1,k_2,x>0$. 
For \(r>1\), use the preceding estimates and the first bound
in Proposition~\ref{T1}. For \(0<r\le1\), use its second bound. 
Together with Proposition~\ref{prop4}(1), this gives,
for all $r>0$,
\begin{equation}\label{temp14}
\begin{aligned}
\T_1\lesssim{}&
\Delta^{2\min\{r,1\}}+\exp\left\{-2\frac{1-2^{\theta-1}}{1-\theta}
\lambda\eta_1t^{1-\theta}\right\}
\left(1+\left(\sum_{k=1}^{t}\eta_k\right)^{2r}\right)^{-1}.
\end{aligned}
\end{equation}
    Next, by Proposition~\ref{prop4}(2)
    \[
    \exp\l\{-2\frac{1-2^{\theta-1}}{1-\theta}\lambda\eta_1 t^{1-\theta}\r\}\l(1+\left(\sum_{k=1}^{t}\eta_{k}\right)^{2r}\r)^{-1}\;\lesssim\;\exp\l\{-2\frac{1-2^{\theta-1}}{1-\theta}\lambda\eta_1 t^{1-\theta}\r\}\eta_1^{-2r}t^{-2r(1-\theta)}.
    \]
    Using the inequality $\exp(-k_1x)\le(\tfrac{k_2}{ek_1})^{k_2}x^{-k_2}$ for any $k_1,k_2>0$ and $x>0$, we obtain the sharper bound when $\theta>\frac{2r}{1+2r}$:
    \begin{equation} \label{temp15}
        \begin{aligned}
            \exp\l\{-2\frac{1-2^{\theta-1}}{1-\theta}\lambda\eta_1 t^{1-\theta}\r\}\eta_1^{-2r}t^{-2r(1-\theta)}
            &\;=\;\exp\l\{-2\frac{1-2^{\theta-1}}{1-\theta}\lambda\eta_1 t^{1-\theta}\r\}t^{-2r+(2r+1)\theta}\eta_1^{-2r}t^{-\theta}            
            \\&\;\lesssim\; \l(\lambda\eta_1\r)^{2r-\frac{\theta}{1-\theta}}\eta_1^{-2r}t^{-\theta}.
        \end{aligned}
    \end{equation}
    On the other hand, when $\theta\leq\frac{2r}{1+2r}$, the trivial estimate gives
    \[
    \exp\l\{-2\frac{1-2^{\theta-1}}{1-\theta}\lambda\eta_1 t^{1-\theta}\r\}\eta_1^{-2r}t^{-2r(1-\theta)}\lesssim\ \eta_1^{-2r}t^{-\theta}.
    \]
By Proposition~\ref{T2} and~\eqref{prop5_decay} in Proposition \ref{prop5}, we have
\begin{equation}  \label{temp16}
\begin{aligned}
\T_2
&\lesssim
(C+1)\sum_{k=1}^{t}\eta_k^2
\frac{
\exp\left\{-2\lambda\sum_{j=k+1}^{t}\eta_j\right\}
}{
1+\sum_{j=k+1}^{t}\eta_j
}
\\
&\lesssim
\eta_1\,t^{-\theta}
\begin{cases}
1+\log(1+\lambda^{-1}),
&0<\theta<\frac12,\\[1mm]
1+\log\left(1+(\lambda\eta_1)^{-1}\right),
&\theta=\frac12,\\[1mm]
(\lambda\eta_1)^{-\frac{2\theta-1}{1-\theta}},
&\frac12<\theta<1.
\end{cases}
\end{aligned}
\end{equation}
Here we used $\lambda\eta_1\le1$ and absorbed the uniform
constant from Proposition~\ref{prop6} into $\lesssim$.
    

    Combining \eqref{second step}, Proposition~\ref{prop3},
Propositions~\ref{prop1}--\ref{prop2}, and \eqref{temp14}
with the subsequent transient and stochastic estimates yields
    \begin{equation*} 
            \begin{aligned}
                \be_{z^t}\l[\l\|\mg^{}_{t+1}-{P_2^{d_2}}\circ\mathcal{G}^{\dagger}\r\|^2_{\rho_u}\r]\lesssim&\;\Delta^{2\min\{r,1\}}+\lambda^{\min\l\{2r,2\r\}}
        \\&+
        \eta_1^{-2r}\l(\lambda\eta_1\r)^{\min\{2r-\frac{\theta}{1-\theta},0\}}t^{-\theta}
            \\&+\eta_1\,t^{-\theta}
\begin{cases}
1+\log(1+\lambda^{-1}),
&0<\theta<\frac12,\\[1mm]
1+\log\left(1+(\lambda\eta_1)^{-1}\right),
&\theta=\frac12,\\[1mm]
(\lambda\eta_1)^{-\frac{2\theta-1}{1-\theta}},
&\frac12<\theta<1.
\end{cases}
        \end{aligned}
        \end{equation*} 
For notational convenience, we collect the iteration-independent terms into
$E_{\mathrm{enc}}(\lambda,\Delta,r)$ and the optimization  terms into
$C_{\mathrm{sgd}}(\lambda,\eta_1,r,\theta)$.
Then
\begin{equation*}
\begin{aligned}
\be_{z^t}\l[\left\|\mathcal{G}^{}_{t+1}-{P_2^{d_2}}\circ\mathcal{G}^{\dagger}\right\|_{\rho_u}^2\r]
&\;\lesssim\;
\underbrace{E_{\mathrm{enc}}(\lambda,\Delta,r)}_{\text{input-side  error}}
\;+\;
\underbrace{C_{\mathrm{sgd}}(\lambda,\eta_1,r,\theta)
\;t^{-\theta}}_{\text{optimization  error}}.
\end{aligned}
\end{equation*}
Here
\[
E_{\mathrm{enc}}(\lambda, \Delta, r)
:=
\Delta^{2\min\{r,1\}}
+\lambda^{\min\{2r,2\}},
\]
and
\[
\begin{aligned}
C_{\mathrm{sgd}}
:={}&
\eta_1^{-2r}
(\lambda\eta_1)^{
\min\left\{2r-\frac{\theta}{1-\theta},\,0\right\}
}
+\eta_1
\begin{cases}
1+\log(1+\lambda^{-1}),
&0<\theta<\frac12,\\[1mm]
1+\log\left(1+(\lambda\eta_1)^{-1}\right),
&\theta=\frac12,\\[1mm]
(\lambda\eta_1)^{-\frac{2\theta-1}{1-\theta}},
&\frac12<\theta<1.
\end{cases}
\end{aligned}
\]
Finally, adding the output-encoding error $\l\|{P_2^{d_2}}\circ\mathcal{G}^{\dagger}-\mathcal{G}^{\dagger}\r\|_{\rho_u}^2$
completes the proof.
\end{proof}

We next consider the constant-step regime, where the step size is chosen as
$\eta_t=\eta T^{-\theta'}$ for $t\in\mathbb{N}_T$.
As before, the control of the term $\mathcal{T}_2$ relies on a uniform bound of the form
\[
\mathbb{E}_{z^t}\!\left[\|\mathcal{G}_{t+1}-{P_2^{d_2}}\circ\mathcal{G}^{\dagger}\|_{\rho_u}^2\right]\le C.
\]
The following proposition establishes this bound under the constant-step setting.

\begin{proposition} \label{prop10}
    Let the step sizes be constant, $\l\{\eta_t=\eta T^{-\theta'}:\forall t\in\bn_T\r\}$ with $0<\theta'<1$ and $\eta \le 1$, 
    and let $0<\lambda\le1$.
    Then, there exists a constant $\bar\eta'>0$, independent of
$\Delta$, $\lambda$, and $T$,
    such that for all
$0<\eta \leq \bar\eta'$, the following bound holds:
\begin{equation*}
    \be_{z^{t}}\l[\l\|\mathcal{G}^{}_{t+1}-{P_2^{d_2}}\circ\mathcal{G}^{\dagger}\r\|_{\rho_u}^2\r]
    \;\lesssim\;1,
    \qquad \forall0\leq t\leq T.
\end{equation*}
\end{proposition}

\begin{proof}
    We prove by induction a uniform bound of the form
\[
\be_{z^{t}}\!\left[\left\|\mathcal{G}^{}_{t+1}-{P_2^{d_2}}\circ\mathcal{G}^{\dagger}\right\|_{\rho_u}^2\right]\le C,
\quad \forall t\in\bn_T,
\]
where $C$ will be chosen so that the bound is self-consistent.
Throughout, take $C\ge \|\mathcal{G}^{\dagger}\|_{\rho_u}^2$. Take $\bar\eta'\le(\kappa^2+1)^{-1}$, so that
$\eta_k(\kappa^2+\lambda)\le1$ for all $k\in\bn_T$.
  
For $t=0$, we have 
\[
\|\mathcal{G}^{}_{1}-{P_2^{d_2}}\circ\mathcal{G}^{\dagger}\|_{\rho_u}^2
= \|{P_2^{d_2}}\circ\mathcal{G}^{\dagger}\|_{\rho_u}^2
\le \|\mathcal{G}^{\dagger}\|_{\rho_u}^2\leq C,\]
and the claim holds.

For $1\le t\le T$, assume that
\[
\be_{z^{k-1}}
\|\mg_k-P_2^{d_2}\circ\mg^\dagger\|_{\rho_u}^2
\le C,\qquad 1\le k\le t.
\]
The contraction argument in the proof of
Proposition~\ref{prop6} gives
\[
\be_{z^t}\|\mg_{t+1}-P_2^{d_2}\circ\mg^\dagger\|_{\rho_u}^2
\le4\|\mg^\dagger\|_{\rho_u}^2+2\T_2.
\]
Write $\eta_1=\eta T^{-\theta'}$.
Reindexing the stochastic sum in Proposition~\ref{T2}, we have
\[
\begin{aligned}
\eta_1^2\sum_{n=0}^{t-1}
\frac{\exp\{-2\lambda n\eta_1\}}{1+n\eta_1}
\le \eta_1^2\sum_{n=0}^{T-1}
\frac{\exp\{-2\lambda n\eta_1\}}{1+n\eta_1}\lesssim \eta T^{-\theta'}\log(T+1)
\lesssim \eta.
\end{aligned}
\]
Here we used Proposition~\ref{prop9} and the boundedness of
$T^{-\theta'}\log(T+1)$ for fixed $\theta'>0$. Consequently, Proposition~\ref{T2} and the induction hypothesis
yield $\T_2\lesssim\eta(C+1)$.
Choosing $C\ge8\|\mg^\dagger\|_{\rho_u}^2+1$ and further
decreasing $\bar\eta'$ if necessary closes the induction,
as in the proof of Proposition~\ref{prop6}.

\end{proof}

\begin{proof}[Proof of Theorem \ref{Thm2}]
The case $T=1$ follows from Proposition~\ref{prop10}, hence assume $T\ge2$. Write $\eta_1=\eta T^{-\theta'}$.
For $r>1$ and $1\le t\le T$, we first record
the following estimates. 
We split the argument according to whether $\eta_1(t-1)\le 1$ or $\eta_1(t-1)>1$.
In the former case, Proposition~\ref{lemma 1} gives
\begin{equation}  \label{temp19}
    \begin{aligned}
        &\Delta^2\exp\left\{-2\lambda\sum_{k=1}^{t}\eta_{k}\right\}\l(\sum_{k=1}^t\frac{\eta_k}{1+\left(\sum_{i=1}^{k-1}\eta_{i}\right)^r}\r)^2
        \\\lesssim\;&\Delta^2\eta_1^{2}t^2\exp\left\{-2\lambda\eta_{1}t\right\}\lesssim\Delta^2\sup_{0<x\leq2,0<\lambda\leq1}\l\{x^2\exp\l\{-2\lambda x\r\}\r\}\lesssim\Delta^2.
    \end{aligned}
\end{equation}
In the latter case, Proposition~\ref{lemma 1} yields
\begin{equation} \label{temp20}
    \begin{aligned}
        &\Delta^2\exp\left\{-2\lambda\sum_{k=1}^{t}\eta_{k}\right\}\l(\sum_{k=1}^t\frac{\eta_k}{1+\left(\sum_{i=1}^{k-1}\eta_{i}\right)^r}\r)^2
        \\\lesssim\;&\Delta^2\exp\left\{-2\lambda\eta_{1}t\right\}
            \lesssim\;\Delta^2
    \end{aligned}
\end{equation}
For $r>1$, use \eqref{temp19}--\eqref{temp20} and the
first bound in Proposition~\ref{T1}; for $0<r\le1$, use its second bound.
Combining these estimates with \eqref{second step},
Propositions~\ref{prop3}, \ref{T2}, \ref{prop9}, \ref{prop10},
and \ref{prop1}--\ref{prop2}, we obtain
\begin{equation}
            \begin{aligned}
                \be_{z^T}\l[\l\|\mg^{}_{T+1}-{P_2^{d_2}}\circ\mathcal{G}^{\dagger}\r\|^2_{\rho_u}\r]\;\lesssim\;&\Delta^{2\min\{r,1\}}+\lambda^{\min\l\{2r,2\r\}}
        \\& + \exp\left\{-2\lambda\eta T^{1-\theta'}\right\}\eta ^{-2r}T^{-2r(1-\theta')}
        \\&+\eta T^{-\theta'}
\bigl[1+\log(1+\lambda^{-1})\bigr].
        \end{aligned}
        \end{equation}      
        Using the inequality $\exp(-k_1x)\le\l(\tfrac{k_2}{ek_1}\r)^{k_2}x^{-k_2}$ for any $k_1,k_2>0$ and $x>0$, we further have
        \begin{equation*}
        \begin{aligned}
            \exp\left\{-2\lambda\eta T^{1-\theta'}\right\}\eta ^{-2r}T^{-2r(1-\theta')}&\leq T^{-\theta'}
            \exp\left\{-2\lambda\eta T^{1-\theta'}\right\}\eta ^{-2r}T^{-2r+(2r+1)\theta'}
            \\&\lesssim T^{-\theta'}\eta ^{-2r}
            \begin{cases}
                1, &\text{ if }  \theta'\leq\frac{2r}{2r+1},\\
                \l(\lambda\eta \r)^{-\frac{(2r+1)\theta'-2r}{1-\theta'}},&\text{ if }  \theta'>\frac{2r}{2r+1}.
            \end{cases}
        \end{aligned}
        \end{equation*}

        Collecting the above bounds, we decompose the right-hand side into:
(i) a $T$-independent term capturing the intrinsic bias induced by regularization and kernel approximation,
and (ii) a $T$-dependent remainder that vanishes as $T\to\infty$ under the chosen step-size regime.
Accordingly, we obtain
        \begin{equation*}
\begin{aligned}
\be_{z^T}\l[\|\mathcal{G}^{}_{T+1}-{P_2^{d_2}}\circ\mathcal{G}^{\dagger}\|_{\rho_u}^2\r]
&\;\lesssim\;
\underbrace{\widetilde{E}_{\mathrm{enc}}(\lambda,\Delta,r)}_{\text{input-side  error}}
\;+\;
\underbrace{\widetilde{C}_{\mathrm{sgd}}(\lambda,\eta ,r,\theta')
\,T^{-\theta'}}_{\text{optimization  error}}.
\end{aligned}
\end{equation*}
Here,
\[
\widetilde{E}_{\mathrm{enc}}(\lambda, \Delta, r)
:=
\Delta^{2\min\{r,1\}}
+\lambda^{\min\{2r,2\}},
\]
and
\[
\widetilde{C}_{\mathrm{sgd}}(\lambda,\eta ,r,\theta')
:=
\eta ^{-2r}
(\lambda\eta )^{\min\!\left\{-\frac{(2r+1)\theta'-2r}{1-\theta'},\,0\right\}}
+\eta\bigl[1+\log(1+\lambda^{-1})\bigr].
\]
Finally, adding the output-encoding error $\l\|{P_2^{d_2}}\circ\mathcal{G}^{\dagger}-\mathcal{G}^{\dagger}\r\|_{\rho_u}^2$ completes the proof.
\end{proof}

\subsection{Proof of Theorem \ref{lower bound}}
We begin by characterizing the spectral structure of the operators
$L_\infty$ and $L_{d_1}$ associated with the linear kernel
$k(u,u')=\langle u,u'\rangle_{\mathcal U}$.
For any $f\in L^2(\mathcal{U},\rho_u)$ and $u\in\mathcal{U}$, we compute
\[
\begin{aligned}
L_\infty f(u)
&=\int_{\mathcal{U}}\langle u,u'\rangle_{\mathcal{U}}\, f(u')\,\mathrm{d}\rho_u(u') \\
&=\Bigl\langle u,\; \mathbb{E}_{u'\sim\rho_u}\bigl[u' f(u')\bigr]\Bigr\rangle_{\mathcal{U}} .
\end{aligned}
\]
Similarly,
\[
L_{d_1} f(u)
=\Bigl\langle u,\; P^{d_1}_1\,\mathbb{E}_{u'\sim\rho_u}\bigl[u' f(u')\bigr]\Bigr\rangle_{\mathcal{U}} .
\]

For $\varphi\in\mathcal{U}$, denote the linear functional
\[
\hat\varphi := \langle \cdot,\varphi\rangle_{\mathcal{U}} \in L^2(\mathcal{U},\rho_u).
\]
The preceding integral representations and the completeness of
\(\{\varphi_i\}_{i\ge1}\) in \(\mathcal U\) imply
\[
\mathrm{ran}(L_\infty)\subseteq
\overline{\mathrm{span}}\{\hat\varphi_1,\hat\varphi_2,\ldots\},
\qquad
\mathrm{ran}(L_{d_1})\subseteq
\overline{\mathrm{span}}\{\hat\varphi_1,\ldots,\hat\varphi_{d_1}\}.
\]

A straightforward calculation shows that
\[
\|\hat\varphi_i\|_{\rho_u}^2
=\langle \varphi_i, \Gamma\varphi_i\rangle_{\mathcal{U}}
=\mu_i,
\]
and hence the family $\{\mu_i^{-1/2}\hat\varphi_i\}_{i\ge1}$ forms an orthonormal system in
$L^2(\mathcal{U},\rho_u)$.

Next, we compute
\[
\begin{aligned}
L_\infty\bigl(\mu_i^{-1/2}\hat\varphi_i\bigr)
&=\mu_i^{-1/2}\Bigl\langle \cdot,\;
\mathbb{E}_{u'\sim\rho_u}\bigl[u'\hat\varphi_i(u')\bigr]\Bigr\rangle_{\mathcal{U}} \\
&=\mu_i^{-1/2}\langle \cdot, \Gamma\varphi_i\rangle_{\mathcal{U}} \\
&=\mu_i^{1/2}\hat\varphi_i .
\end{aligned}
\]
Therefore, the spectral decomposition of $L_\infty$ reads
\[
L_\infty
=\sum_{i\ge1}\mu_i\,
\langle \cdot,\mu_i^{-1/2}\hat\varphi_i\rangle_{\rho_u}\,
\mu_i^{-1/2}\hat\varphi_i ,
\]
and similarly,
\[
L_{d_1}
=\sum_{i=1}^{d_1}\mu_i\,
\langle \cdot,\mu_i^{-1/2}\hat\varphi_i\rangle_{\rho_u}\,
\mu_i^{-1/2}\hat\varphi_i .
\]
Consequently,
\[
\Delta=\|L_\infty-L_{d_1}\|=\mu_{d_1+1}.
\]

We define the source class
\[
\mathcal O'
:=\Bigl\{
f^\dagger
= L_{\infty}^r f^{\mathrm{src}}
:\;
f^{\mathrm{src}}\in L^2(\mathcal U,\rho_u),
\ \|f^{\mathrm{src}}\|_{\rho_u}\le1
\Bigr\},
\]
and the associated worst-case error is
\[
\mathfrak R_{d_1}':=\sup_{f^\dagger\in\mathcal O'}\ 
\inf_{f\in\mathcal H_{k_{d_1}}}
\bigl\|f^\dagger-f\bigr\|_{\rho_u}.
\]
Next, we establish the relation between $\mathfrak R_{d_1,d_2}$ and $\mathfrak R_{d_1}'$.
\begin{equation*}
\begin{aligned}
    \mathfrak R_{d_1,d_2}
&=\sup_{\mathcal G^\dagger\in\mathcal O}\ 
\inf_{\mathcal G\in\mathcal H_{K_{d_1,d_2}}}
\bigl\|{P_2^{d_2}}\circ\mathcal G^\dagger-\mathcal G\bigr\|_{\rho_u}
\\&\geq\sup_{f^\dagger\in\mathcal O'}\ 
\inf_{\mathcal G\in\mathcal H_{K_{d_1,d_2}}}
\bigl\|f^\dagger(\cdot)\psi_1-\mathcal G\bigr\|_{\rho_u}
\\&=\sup_{f^\dagger\in\mathcal O'}\ 
\inf_{f\in\mathcal H_{k_{d_1}}}
\bigl\|f^\dagger(\cdot)\psi_1-f(\cdot)\psi_1\bigr\|_{\rho_u}
\\&=\mathfrak R_{d_1}'.
\end{aligned}
\end{equation*}
We now prove that $\mathfrak R_{d_1}'=\mu_{d_1+1}^r$.
Any $f\in\mathcal{O}'$ admits the expansion
\[
f=\sum_{i\ge1}\mu_i^{\,r} a_i\,\bigl(\mu_i^{-1/2}\hat\varphi_i\bigr),
\qquad
\sum_{i\ge1}a_i^2\le1 .
\]
Moreover, elements of $\overline{\mathcal{H}_{k_{d_1}}}=\mathrm{ran}(L_{d_1}^{1/2})$ can be written as
\[
g=\sum_{i=1}^{d_1} \mu_i^{1/2} b_i\,\bigl(\mu_i^{-1/2}\hat\varphi_i\bigr).
\]
Hence,
\[
\begin{aligned}
\mathfrak R_{d_1}'
&=\sup_{f\in\mathcal{O}'}\mathrm{dist}(f,\mathcal{H}_{k_{d_1}}) \\
&=\sup_{f\in\mathcal{O}'}
\left(\sum_{i\ge d_1+1}\mu_i^{2r}a_i^2\right)^{1/2} \\
&=\mu_{d_1+1}^{\,r},
\end{aligned}
\]
which completes the proof.

\appendix

	\section*{Appendix}
 
    \setcounter{equation}{0}
    \setcounter{theorem}{0}
    \addtocounter{section}{0}

\section{Proof of Proposition~\ref{prop1.1}} \label{Appendix 1}

For any $x\in\mathbb R^{d_1}$ and $y\in\mathbb R^{d_2}$, define
\[
\Phi_{x,y} := k(\cdot,x)y \in \mathcal H_{\widetilde K}.
\]
By definition of the lifting operator $T$, we have
\[
(T\Phi_{x,y})(u)
= \mathcal D_2^{d_2}\,\Phi_{x,y}\bigl(\mathcal E_1^{d_1}u\bigr)
= k\bigl(\mathcal E_1^{d_1}u,x\bigr)\,\mathcal D_2^{d_2}y.
\]

Let $u_0\in\mathcal U$ be such that $\mathcal E_1^{d_1}u_0=x$,
which exists by the surjectivity  of $\mathcal{E}_1^{d_1}$. Set
$w:=\mathcal D_2^{d_2}y\in\operatorname{ran}({P_2^{d_2}})$.
By definition of the operator-valued kernel $K$, it follows that
\[
\bigl(K(\cdot,u_0)w\bigr)(u)
= k\bigl(\mathcal E_1^{d_1}u,\mathcal E_1^{d_1}u_0\bigr)\,{P_2^{d_2}}w
= k\bigl(\mathcal E_1^{d_1}u,x\bigr)\,\mathcal D_2^{d_2}y.
\]
Hence, $T\Phi_{x,y}=K(\cdot,u_0)w\in\mathcal H_K$.

Next, for $x,x'\in\mathbb R^{d_1}$ and $y,y'\in\mathbb R^{d_2}$, define
$w=\mathcal D_2^{d_2}y$ and $w'=\mathcal D_2^{d_2}y'$.
By the reproducing property in $\mathcal H_{\widetilde K}$,
\[
\langle \Phi_{x,y},\Phi_{x',y'}\rangle_{\mathcal H_{\widetilde K}}
= k(x,x')\,\langle y,y'\rangle_{2}
= k(x,x')\,\langle {P_2^{d_2}}w,{P_2^{d_2}}w'\rangle_{\mathcal V}.
\]
On the other hand, by the reproducing property in $\mathcal H_K$,
\[
\langle T\Phi_{x,y},T\Phi_{x',y'}\rangle_{\mathcal H_K}
= \langle K(u_0',u_0)w,w'\rangle_{\mathcal V}
= k(x,x')\,\langle {P_2^{d_2}}w,w'\rangle_{\mathcal V},
\]
where $u_0'\in\mathcal U$ satisfies $\mathcal E_1^{d_1}u_0'=x'$,
again by surjectivity of $\mathcal E_1^{d_1}$.
Therefore, $T$ preserves inner products on
\[
\operatorname{span}\{\Phi_{x,y}:x\in\mathbb R^{d_1},\,y\in\mathbb R^{d_2}\},
\]
which is dense in $\mathcal H_{\widetilde K}$.

By continuity, $T$ extends uniquely to a linear isometry from
$\mathcal H_{\widetilde K}$ into $\mathcal H_K$, and the extension satisfies
\[
(Tf)(u)=\mathcal D_2^{d_2}\,f\bigl(\mathcal E_1^{d_1}u\bigr),
\qquad \forall f\in\mathcal H_{\widetilde K}.
\]

To show surjectivity, observe that for any $u\in\mathcal U$ and
$w\in\operatorname{ran}({P_2^{d_2}})$,
\[
K(\cdot,u)w = T\Phi_{\mathcal E_1^{d_1}u,\mathcal E_2^{d_2}w}.
\]
Since the linear span of $\{K(\cdot,u)w\}$ is dense in $\mathcal H_K$
and the range of $T$ is closed (as $T$ is isometric), it follows that
$\mathrm{ran}(T)=\mathcal H_K$.
Hence, $T$ is an isometric isomorphism between
$\mathcal H_{\widetilde K}$ and $\mathcal H_K$.

Finally, for any $\mathcal G=Tf$, let $x=\mathcal E_1^{d_1}u$ and
$y=\mathcal E_2^{d_2}v$.
Using the orthogonal decomposition induced by ${P_2^{d_2}}$, we obtain
\begin{align*}
\|\mathcal G(u)-v\|_{\mathcal V}^2
&= \|\mathcal G(u)-{P_2^{d_2}}v\|_{\mathcal V}^2
  + \|(I-{P_2^{d_2}})v\|_{\mathcal V}^2 \\
&= \|f(x)-y\|_2^2
  + \|(I-{P_2^{d_2}})v\|_{\mathcal V}^2.
\end{align*}
This establishes the equivalence of the two objective functionals
up to an additive constant independent of the estimator.

\section{Proofs for Section \ref{Section examples}} \label{Appendix 2}

\begin{proof}[Proof of Proposition \ref{prop20}]
Let $\phi$ be $\alpha$--H\"older continuous on the relevant interval.
\paragraph{Radial kernels.}
For $k(x,x')=\phi(\|x-x'\|_2)$ and admissible $x_1,x_1',x_2,x_2'$,
\[
\begin{aligned}
|k(x_1,x_1')-k(x_2,x_2')|
&=|\phi(\|x_1-x_1'\|_2)-\phi(\|x_2-x_2'\|_2)|\\
&\lesssim \bigl|\|x_1-x_1'\|_2-\|x_2-x_2'\|_2\bigr|^{\alpha}\\
&\lesssim \|x_1-x_2\|_2^{\alpha}+\|x_1'-x_2'\|_2^{\alpha}.
\end{aligned}
\]
\paragraph{Dot product kernels.}
For $k(x,x')=\phi(\langle x,x'\rangle_2)$ and admissible $x_1,x_1',x_2,x_2'$,
\[
\begin{aligned}
|k(x_1,x_1')-k(x_2,x_2')|
&=|\phi(\langle x_1,x_1'\rangle_2)-\phi(\langle x_2,x_2'\rangle_2)|\\
&\lesssim |\langle x_1,x_1'\rangle_2-\langle x_2,x_2'\rangle_2|^{\alpha}\\
&\le \bigl(\|x_1-x_2\|_2\|x_1'\|_2+\|x_2\|_2\|x_1'-x_2'\|_2\bigr)^{\alpha}\\
&\lesssim \|x_1-x_2\|_2^{\alpha}+\|x_1'-x_2'\|_2^{\alpha},
\end{aligned}
\]
where the last inequality uses the boundedness of the admissible domain. All implied constants are independent of $d$.
\end{proof}

\begin{proof}[Proof of Proposition \ref{prop21}]
\noindent(1)
    We verify positive definiteness on $\mathcal{U}$ by reduction to a finite-dimensional
Gram matrix. 

Let $m\in\mathbb{N}$ and $u_1,\dots,u_m\in\mathcal{U}$ be arbitrary, and let
$c_1,\dots,c_m\in\mathbb{R}$. Denote by $H:=\mathrm{span}\{u_1,\dots,u_m\}$ the (at most $m$-dimensional)
subspace of $\mathcal{U}$. Choose an orthonormal basis of $H$ and identify $H$ isometrically with
$\mathbb{R}^d$ where $d:=\dim(H)\le m$. Under this identification, the quantities
$\|u_i-u_j\|_{\mathcal{U}}$ (in the radial case) and $\langle u_i,u_j\rangle_{\mathcal{U}}$ (in the dot product
case) agree with the corresponding Euclidean quantities in $\mathbb{R}^d$. Hence the Gram matrix
$[k(u_i,u_j)]_{i,j=1}^m$ coincides with the Gram matrix generated by the same formula in $\mathbb{R}^d$.
By the assumption that $k$ is positive definite on $\mathbb{R}^d$ for all $d\ge1$, we obtain
\[
\sum_{i,j=1}^m c_i c_j\, k(u_i,u_j)\ge 0.
\]
Since $m$, $\{u_i\}$ and $\{c_i\}$ are arbitrary, $k$ is positive definite on $\mathcal{U}$.

\medskip
\noindent(2) The $\alpha$--H\"older continuity follows by the same argument as in the proof of
Proposition~\ref{prop20}.
\end{proof}

\begin{proof}[Proof of Corollary~\ref{prop22}]
By definition of the encoder-induced kernel,
\[
k_{d_1}(u,u')
=
k_{\infty}'\l(\big(\big(\mathcal{E}_1^{d_1}\big)^*\mathcal{E}_1^{d_1}\big)^{1/2}(u),
\big(\big(\mathcal{E}_1^{d_1}\big)^*\mathcal{E}_1^{d_1}\big)^{1/2}(u')\r).
\]
Moreover, 
\[
\Delta
\le
\|L_{d_1}-L_\infty\|_{\mathrm{HS}}
=
\|k_{d_1}-k_{\infty}\|_{L^2(\mathcal{U}\times\mathcal{U},\,\rho_u\otimes\rho_u)}.
\]
The conclusion then follows directly from Proposition~\ref{prop21}.
\end{proof}

\section{Proofs for Section \ref{section 2.2}} \label{Sec 8}
This appendix provides the proofs of the results stated in
Section~\ref{section 2.2}.
We follow the organization of that section.
We first establish the basic properties of the NTK and of its lift to the input space $\mathcal U$,
including the construction of the limiting kernel $k^{\mathrm{NTK}}_\infty$
and the bounds quantifying $\Delta$, which captures both finite-width and finite-encoding effects.
We then prove the coupling results that relate neural SGD to the corresponding kernel-SGD recursion,
and complete the proofs of the main theorems.



\subsection{Proofs for Subsection \ref{section: Limiting NTK}}

\begin{proof}[Proof of Proposition \ref{k Lip}]
We decompose the difference by the triangle inequality:
    \[
     \l\lvert k^{\mathrm{NTK}}\left(x_1,x_1^{\prime}\right)-k^{\mathrm{NTK}}\left(x_2,x_2^{\prime}\right)\r\rvert\leq \l\lvert k^{\mathrm{NTK}}\left(x_1,x_1^{\prime}\right)-k^{\mathrm{NTK}}\left(x_2,x_1^{\prime}\right)\r\rvert + \l\lvert k^{\mathrm{NTK}}\left(x_2,x_1^{\prime}\right)-k^{\mathrm{NTK}}\left(x_2,x_2^{\prime}\right)\r\rvert.
    \]
    By symmetry of the kernel, it suffices to control the first term.

    Using Assumption~\ref{assumption: NTK} together with Jensen's inequality, we obtain
    \begin{equation*}
        \begin{aligned}
            \l\lvert k^{\mathrm{NTK}}\left(x_1,x_1^{\prime}\right)-k^{\mathrm{NTK}}\left(x_2,x_1^{\prime}\right)\r\rvert\leq& C_\sigma\mathbb{E}_{b\sim\mathcal{N}(0, I_{d})}\left[\l\lvert\sigma\left(b^\top x_1\right)-\sigma\left(b^\top x_2\right)\r\rvert\right]
            \\&+\l\lvert(x_1-x_2)^\top x_1'\r\rvert\mathbb{E}_{b\sim\mathcal{N}(0, I_{d})}\left[\l\lvert\sigma^{\prime}\left(b^\top x_1\right)\sigma^{\prime}\left(b^\top x_1^{\prime}\right)\r\rvert\right]
            \\&+\l\lvert x_2^\top x_1'+\gamma^2\r\rvert C_\sigma\mathbb{E}_{b\sim\mathcal{N}(0, I_{d})}\left[\l\lvert\sigma'\left(b^\top x_1\right)-\sigma'\left(b^\top x_2\right)\r\rvert\right]
            \\\leq&C_{\sigma}^2\sqrt{\mathbb{E}_{b\sim\mathcal{N}(0, I_{d})}\l\lvert b^\top (x_1-x_2)\r\rvert^2} + C_{\sigma}^2R\|x_1-x_2\|_2
            \\&+\l(R^2+\gamma^2\r)C_{\sigma}^2\sqrt{\mathbb{E}_{b\sim\mathcal{N}(0, I_{d})}\l\lvert b^\top (x_1-x_2)\r\rvert^2}
            \\\leq&\l(1+R+R^2+\gamma^2\r)C_{\sigma}^2\|x_1-x_2\|_2.
        \end{aligned}
    \end{equation*}
    By symmetry, the same bound holds with respect to the second argument, and therefore
        \[
        \l\lvert k^{\mathrm{NTK}}\left(x_1,x_1^{\prime}\right)-k^{\mathrm{NTK}}\left(x_2,x_2^{\prime}\right)\r\rvert\leq\l(1+R+R^2+\gamma^2\r)C_{\sigma}^2\l(\|x_1-x_2\|_2+\|x_1'-x_2'\|_2\r).
        \]
    This completes the proof.
\end{proof}

\begin{proof}[Proof of Proposition \ref{prop11}]
    Fix $(\xi_1,\xi_2,\eta)\in\mathcal D_r^R$, so that $\xi_1,\xi_2\ge r>0$.
By Lemma~\ref{lemma NTK1}, the function $\phi$ admits the representation
    \begin{align*}
        \phi\l(\xi_1,\xi_2,\eta\r)=& \mathbb{E}_{(b_1,b_2)\sim\mathcal{N}(0, I_2)}\left[\sigma\left(b_1\xi_1\right)\sigma\left(\l(b_1\cos\alpha + b_2\sin\alpha\r)\xi_2\right)\right]\\&+\left(\eta
        +\gamma^2\right)\mathbb{E}_{(b_1,b_2)\sim\mathcal{N}(0, I_2)}\left[\sigma^{\prime}\left(b_1\xi_1\right)\sigma^{\prime}\left(\l(b_1\cos\alpha + b_2\sin\alpha\r)\xi_2\right)\right],
    \end{align*}
    where $\cos\alpha=\eta/(\xi_1\xi_2)$. It therefore suffices to show that each expectation is $\frac12$–H\"older continuous
on $\mathcal D_r^R$.

    We first consider the term involving $\sigma$.
Note that
    \[
    \l(b_1\cos\alpha + b_2\sin\alpha\r)\xi_2=b_1\frac{\eta}{\xi_1}+b_2\sqrt{\xi_2^2-\frac{\eta^2}{\xi_1^2}}.
    \]
    For any $\l(\xi_1,\xi_2,\eta\r),\;\l(\xi'_1,\xi'_2,\eta'\r)\in \mathcal{D}_r^R$, there holds
    \begin{equation*}
        \begin{aligned}
            &\l\lvert\sigma\left(\l(b_1\cos\alpha + b_2\sin\alpha\r)\xi_2\right)-\sigma\left(\l(b_1\cos\alpha' + b_2\sin\alpha'\r)\xi'_2\right)\r\rvert
            \\\leq& C_\sigma \vert b_1\vert\l\lvert\frac{\eta}{\xi_1}-\frac{\eta'}{\xi'_1}\r\rvert 
            + C_\sigma \vert b_2\vert\l\lvert\sqrt{\xi_2^2-\frac{\eta^2}{\xi_1^2}}-\sqrt{(\xi'_2)^2-\frac{{\eta'}^2}{(\xi'_1)^2}}\r\rvert
            \\\leq&C_\sigma \vert b_1\vert\l(\frac{1}{r}\lvert\eta-\eta'\rvert+\frac{R^2}{r^2}\lvert\xi_1-\xi_1 '\rvert\r)+C_\sigma \vert b_2\vert\l(\sqrt{\l\lvert\xi_2^2-(\xi'_2)^2\r\rvert}+\sqrt{\l\lvert\frac{\eta^2}{\xi_1^2}-\frac{{\eta'}^2}{(\xi'_1)^2}\r\rvert}\r)
            \\\leq&C_\sigma \vert b_1\vert\frac{R^2}{r^2}\lvert\xi_1-\xi_1 '\rvert+C_\sigma \vert b_2\vert \sqrt{2R}\sqrt{\lvert\xi_2-\xi_2 '\rvert}+\frac{C_\sigma \vert b_1\vert}{r}\lvert\eta-\eta'\rvert
            \\&+C_\sigma \vert b_2\vert\frac{\sqrt{2}R}{r}\sqrt{\lvert\eta-\eta'\rvert}+C_\sigma \vert b_2\vert\frac{R^2\sqrt{2R}}{r^2}\sqrt{\lvert\xi_1-\xi_1 '\rvert}.
        \end{aligned}
    \end{equation*}
    Thus, we obtain the following bound:
    \begin{equation*}
        \begin{aligned}
            &\l\lvert\sigma\left(b_1\xi_1\right)\sigma\left(\l(b_1\cos\alpha + b_2\sin\alpha\r)\xi_2\right)-\sigma\left(b_1\xi_1'\right)\sigma\left(\l(b_1\cos\alpha' + b_2\sin\alpha'\r)\xi_2'\right)\r\rvert
            \\\leq&C_\sigma\l\lvert \sigma\left(b_1\xi_1\right)-\sigma\left(b_1\xi_1'\right)\r\rvert + C_\sigma\l\lvert\sigma\left(\l(b_1\cos\alpha + b_2\sin\alpha\r)\xi_2\right)-\sigma\left(\l(b_1\cos\alpha' + b_2\sin\alpha'\r)\xi'_2\right)\r\rvert
            \\\leq&C_\sigma^2\vert b_1\vert\l(1+\frac{R^2}{r^2}\r)\lvert\xi_1-\xi_1 '\rvert+C_\sigma^2 \vert b_2\vert\frac{R^2\sqrt{2R}}{r^2}\sqrt{\lvert\xi_1-\xi_1 '\rvert}+C_\sigma^2 \vert b_2\vert \sqrt{2R}\sqrt{\lvert\xi_2-\xi_2 '\rvert}
            \\&+\frac{C_\sigma^2 \vert b_1\vert}{r}\lvert\eta-\eta'\rvert+C_\sigma^2 \vert b_2\vert\frac{\sqrt{2}R}{r}\sqrt{\lvert\eta-\eta'\rvert}.
        \end{aligned}
    \end{equation*}
    Taking expectations and using $\mathbb{E}_{b \sim \mathcal{N}(0,1)}|b|\le 1$,
we conclude that
    \[
    \mathbb{E}_{(b_1,b_2)\sim\mathcal{N}(0, I_2)}\left[\sigma\left(b_1\xi_1\right)\sigma\left(\l(b_1\cos\alpha + b_2\sin\alpha\r)\xi_2\right)\right]
    \]
    is $\frac{1}{2}-$H\"older continuous.

    The same argument applies to the term involving $\sigma'$, and therefore $\phi$
is $\frac12$–H\"older continuous on $\mathcal D_r^R$.
\end{proof}

\begin{proof}[Proof of Proposition \ref{NTK: lipschitz}]
    We begin by proving that $k_{\infty}'$ is continuous on $\mathcal{U} \times \mathcal{U}$.\\
    \textbf{Case 1:}
    If $u,u'\not=0$, then
    \begin{align*}
        k_{\infty}'\left(u,u^{\prime}\right)=& \mathbb{E}_{(b_1,b_2)\sim\mathcal{N}(0, I_2)}\left[\sigma\left(b_1\|u\|_{\mathcal{U}}\right)\sigma\left(\l(b_1\cos\alpha + b_2\sin\alpha\r)\|u'\|_{\mathcal{U}}\right)\right]\\&+\left(\langle u, u^{\prime}\rangle_{\mathcal{U}}
        +\gamma^2\right)\mathbb{E}_{(b_1,b_2)\sim\mathcal{N}(0, I_2)}\left[\sigma^{\prime}\left(b_1\|u\|_{\mathcal{U}}\right)\sigma^{\prime}\left(\l(b_1\cos\alpha + b_2\sin\alpha\r)\|u'\|_{\mathcal{U}}\right)\right],
    \end{align*}
    where $\alpha=\arccos \frac{\langle u,u'\rangle_{\mathcal{U}}}{\|u\|_{\mathcal{U}}\|u'\|_{\mathcal{U}}}\in[0,\pi]$. Since $(u,u')\mapsto(\cos\alpha,\sin\alpha)$ is continuous at any $(u,u')\not=0$, $k_{\infty}'$ is continuous at any $(u,u')\not=0$. 
    \\\textbf{Case 2:}
    If $u=0$ and $u'\not=0$, then
    \[
    k_{\infty}'\left(u,u^{\prime}\right)=\sigma(0)\be_{b\sim\mathcal{N}(0, 1)}\l[\sigma(b\|u'\|_{\mathcal{U}})\r]+\gamma^2\sigma'(0)\be_{b\sim\mathcal{N}(0, 1)}\l[\sigma'(b\|u'\|_{\mathcal{U}})\r].
    \]
    Suppose that $(u_n,u_n')\to(0,u')$. For all $u_n\not=0$, by the Dominated Convergence Theorem, we obtain
    \begin{equation*}
        \begin{aligned}
            &\l\lvert\mathbb{E}_{(b_1,b_2)\sim\mathcal{N}(0, I_2)}\left[\sigma\left(b_1\|u_n\|_{\mathcal{U}}\right)\sigma\left(\l(b_1\cos\alpha + b_2\sin\alpha\r)\|u_n'\|_{\mathcal{U}}\right)\right]-\sigma(0)\be_{b\sim\mathcal{N}(0, 1)}\l[\sigma(b\|u'\|_{\mathcal{U}})\r]\r\rvert 
            \\\leq&C_\sigma\mathbb{E}_{b_1\sim\mathcal{N}(0, 1)}\left\lvert\sigma\left(b_1\|u_n\|_{\mathcal{U}}\right)-\sigma(0)\right\rvert
            \\&+\lvert\sigma(0)\rvert\l\lvert\mathbb{E}_{(b_1,b_2)\sim\mathcal{N}(0, I_2)}\left[\sigma\left(\l(b_1\cos\alpha + b_2\sin\alpha\r)\|u_n'\|_{\mathcal{U}}\right)\right]-\be_{b\sim\mathcal{N}(0, 1)}\l[\sigma(b\|u'\|_{\mathcal{U}})\r]\r\rvert
            \\=&C_\sigma\mathbb{E}_{b\sim\mathcal{N}(0, 1)}\left\lvert\sigma\left(b\|u_n\|_{\mathcal{U}}\right)-\sigma(0)\right\rvert
            +\lvert\sigma(0)\rvert\be_{b\sim\mathcal{N}(0, 1)}\l\lvert\sigma(b\|u_n'\|_{\mathcal{U}})-\sigma(b\|u'\|_{\mathcal{U}})\r\rvert
            \xrightarrow{n\to\infty}0,
        \end{aligned}
    \end{equation*}
    where we have used the fact $b_1\cos\alpha + b_2\sin\alpha\sim\mathcal{N}(0, 1)$. By similar arguments, it holds that
    \[
    \l\lvert\mathbb{E}_{(b_1,b_2)\sim\mathcal{N}(0, I_2)}\left[\sigma'\left(b_1\|u_n\|_{\mathcal{U}}\right)\sigma'\left(\l(b_1\cos\alpha + b_2\sin\alpha\r)\|u_n'\|_{\mathcal{U}}\right)\right]-\sigma'(0)\be_{b\sim\mathcal{N}(0, 1)}\l[\sigma'(b\|u'\|_{\mathcal{U}})\r]\r\rvert \xrightarrow{n\to\infty}0.
    \]
    Thus, we have
    \[
    k_{\infty}'\left(u_n,u_n^{\prime}\right)\xrightarrow{n\to\infty}k_{\infty}'\left(0,u^{\prime}\right).
    \]
    For all $u_n=0$ and $u_n'\to u'$, it is sufficient to prove that
    \begin{equation*}
        \begin{aligned}
            \lvert\sigma(0)\rvert\be_{b\sim\mathcal{N}(0, 1)}\l\lvert\sigma(b\|u_n'\|_{\mathcal{U}})-\sigma(b\|u'\|_{\mathcal{U}})\r\rvert+\gamma^2\lvert\sigma'(0)\rvert\be_{b\sim\mathcal{N}(0, 1)}\l\lvert\sigma'(b\|u_n'\|_{\mathcal{U}})-\sigma'(b\|u'\|_{\mathcal{U}})\r\rvert\xrightarrow{n\to\infty}0,
        \end{aligned}
    \end{equation*}
    which is also clear by the Dominated Convergence Theorem. 
    \\\textbf{Case 3:}
    If $u\not=0$ and $u'=0$, the result follows from Case 2 by a similar proof.
    \\\textbf{Case 4:}
    If $u=0$ and $u'=0$, suppose that $(u_n,u_n')\to(0,0)$. By considering the cases where $u_n \neq 0$ and $u_n' \neq 0$, $u_n = 0$ and $u_n' \neq 0$, and $u_n \neq 0$ and $u_n' = 0$, and using the same argument as in the preceding cases, we obtain continuity at $(0, 0)$.
    
    Next, we prove the Lipschitz continuity of $k_{\infty}'$. Suppose that $\l\{u^k\r\}_{k\geq1}$ is a complete orthogonal basis of $\mathcal{U}$. For any $d > 1$, define the projection operator $P_{d} : \mathcal{U} \to \mathbb{R}^{d}$ by $P_{d}(u) = \sum_{k=1}^{d} \langle u, u^k \rangle_{\mathcal{U}} u^k$.
    For any $u_1,u_1',u_2,u_2'\in\mathcal{B}_R(\mathcal{U})$, by the continuity of $k_{\infty}'$, we have
    \[
    k_{\infty}'(P_{d_1}u_i,P_{d_1}u_i')\xrightarrow{d_1\to\infty} k_{\infty}'(u_i,u_i'), \quad\forall i=1,2.
    \]
    Since
    \[
    k_{\infty}'(P_{d_1}u_i,P_{d_1}u_i')=k^{\mathrm{NTK}}\l(\l(\langle u_i,u^k\rangle_{\mathcal{U}}\r)_{k=1}^{d_1},\l(\langle u_i,u^k\rangle_{\mathcal{U}}\r)_{k=1}^{d_1}\r), \quad\forall i=1,2,
    \]
    by Proposition \ref{k Lip}, the desired result follows.
\end{proof}

\begin{proof}[Proof of Corollary \ref{coro1}]
By definition,
\[
k^{\mathrm{NTK}}_{d_1}(u,u')
= k^{\mathrm{NTK}}\bigl(\mathcal E_1^{d_1}u,\mathcal E_1^{d_1}u'\bigr).
\]
By the dimension-free representation of the NTK and the definition of
$k'_\infty$, this can be rewritten as
\[
k^{\mathrm{NTK}}_{d_1}(u,u')
=
k'_\infty\!\left(
\bigl((\mathcal E_1^{d_1})^*\mathcal E_1^{d_1}\bigr)^{1/2}u,\;
\bigl((\mathcal E_1^{d_1})^*\mathcal E_1^{d_1}\bigr)^{1/2}u'
\right).
\]
Since \(A_{d_1}u,A_{d_1}u',Tu,Tu'\in B_{R_*}(\mathcal U)\),
applying Proposition~\ref{NTK: lipschitz} on this ball yields the pointwise bound
\eqref{temp21}.

    Next, since the kernels are real-valued, the Hilbert--Schmidt norm of the
difference of the associated integral operators satisfies
    \begin{equation*}
        \begin{aligned}
            \l\|L_{k^{\mathrm{NTK}}_{d_1}}-L_{k^{\mathrm{NTK}}_\infty}\r\|_{\mathrm{HS}}&\leq\sqrt{\int\int\l(k^{\mathrm{NTK}}_{d_1}(u,u')-k^{\mathrm{NTK}}_\infty(u,u')\r)^2\mathrm{d}\rho_u(u)\mathrm{d}\rho_u(u')}
            \\&\leq 2C_{\sigma}^2\l(1+R_*+R_*^2+\gamma^2\r)\sqrt{\be_{u\sim\rho_u}\l[\l\|\bigl((\mathcal E_1^{d_1})^*\mathcal E_1^{d_1}\bigr)^{1/2}u-Tu\r\|_{\mathcal{U}}^2\r]},
        \end{aligned}
    \end{equation*}
    which completes the proof.
\end{proof}

The proof of Proposition~\ref{prop12} uses the following Bernstein-type inequality for Hilbert space-valued random variables \cite[Proposition~2]{caponnetto2007optimal}.
\begin{proposition}[Bernstein inequality] \label{concentration}
    Let $w_{1}, \cdots, w_{n}$ be i.i.d random variables in a separable Hilbert space with norm $\|\cdot\|$. Suppose that there exist two positive constants $B$ and $\sigma^{2}$ such that
$$\mathbb{E}\left[\left\|w_{1}-\mathbb{E}\left[w_{1}\right]\right\|^{l}\right] \leq \frac{1}{2} l! B^{l-2} \sigma^{2}, \quad \forall l \geq 2.$$
Then, for any $0<\delta<1$, the following holds with probability at least $1-\delta$,
$$
\left\|\frac{1}{n} \sum_{r=1}^{n} w_{r}-\mathbb{E}\left[w_{1}\right]\right\| \leq\left(\frac{2 B}{n}+\frac{2 \sigma}{\sqrt{n}}\right) \log \frac{2}{\delta}.$$
\end{proposition}

\begin{proof}[Proof of Proposition~\ref{prop12}]
It is standard that for real-valued kernels,
\[
\|L_{k^{\mathrm{NTK}}_{d_1}}-L_{k^M_{d_1}}\|_{\mathrm{HS}}
=
\|k^{\mathrm{NTK}}_{d_1}-k^M_{d_1}\|_{L^2(\rho_u\otimes\rho_u)}.
\]

For $u,u'\in\mathcal U$, set $x:=\mathcal E_1^{d_1}u$ and $x':=\mathcal E_1^{d_1}u'$.
Define the population kernels
\begin{align*}
k_1(u,u') &:= \mathbb{E}_{b\sim\mathcal N(0,I_{d_1})}\bigl[\sigma(b^\top x)\,\sigma(b^\top x')\bigr],\\
k_2(u,u') &:= \mathbb{E}_{b\sim\mathcal N(0,I_{d_1})}\bigl[\sigma'(b^\top x)\,\sigma'(b^\top x')\bigr],
\end{align*}
and their empirical counterparts
\begin{align*}
k_1'(u,u') &:= \frac1M\sum_{r=1}^M \sigma(b_r^\top x)\,\sigma(b_r^\top x'),\\
k_2'(u,u') &:= \frac1M\sum_{r=1}^M \sigma'(b_r^\top x)\,\sigma'(b_r^\top x').
\end{align*}
Then
\[
k^{\mathrm{NTK}}_{d_1}(u,u')
=
k_1(u,u')+\bigl(\langle x,x'\rangle_2+\gamma^2\bigr)\,k_2(u,u'),
\qquad
k^M_{d_1}(u,u')
=
k_1'(u,u')+\bigl(\langle x,x'\rangle_2+\gamma^2\bigr)\,k_2'(u,u').
\]

Let $\kappa_E:=\sup_{d_1\ge1}\|\mathcal E_1^{d_1}\|<\infty$.
Under $\mathrm{supp}(\rho_u)\subset\mathcal B_R(\mathcal U)$, we have
$\|x\|_2\le \kappa_E R$ and $\|x'\|_2\le \kappa_E R$, hence
$|\langle x,x'\rangle_2|\le \kappa_E^2 R^2$.
Therefore,
\begin{equation}\label{eq:split-ntk}
\|k^{\mathrm{NTK}}_{d_1}-k^M_{d_1}\|_{L^2(\rho_u\otimes\rho_u)}
\le
\|k_1-k_1'\|_{L^2(\rho_u\otimes\rho_u)}
+(\kappa_E^2R^2+\gamma^2)\,\|k_2-k_2'\|_{L^2(\rho_u\otimes\rho_u)}.
\end{equation}

We now control each term by Proposition~\ref{concentration}.
Consider the Hilbert space
$H:=L^2(\mathcal U\times\mathcal U,\rho_u\otimes\rho_u)$ with norm
$\|\cdot\|_{H}$.
For $i=1,2$ and $r=1,\dots,M$, define $H$-valued random variables
\[
w_r^{(1)}(u,u'):=\sigma(b_r^\top x)\,\sigma(b_r^\top x'),
\qquad
w_r^{(2)}(u,u'):=\sigma'(b_r^\top x)\,\sigma'(b_r^\top x').
\]
Then $\mathbb E[w_r^{(1)}]=k_1$ and $\frac1M\sum_{r=1}^M w_r^{(1)}=k_1'$,
and similarly for $i=2$.
By Assumption~\ref{assumption: NTK},
\[
\|w_r^{(i)}\|_{\infty}\le C_\sigma^2
\quad\Longrightarrow\quad
\|w_r^{(i)}\|_{H}\le C_\sigma^2,
\qquad i=1,2.
\]
Hence $\|w_r^{(i)}-\mathbb E[w_r^{(i)}]\|_{H}\le 2C_\sigma^2$ almost surely,
which implies the moment condition in Proposition~\ref{concentration} (up to absolute
constants). 
By the paired initialization,
\[
\frac1M\sum_{r=1}^M w_r^{(i)}
=\frac2M\sum_{r=1}^{M/2}w_r^{(i)},\qquad i=1,2.
\]
Applying Proposition~\ref{concentration} to $\{w_r^{(i)}\}_{r=1}^{M/2}$
yields that, with probability at least $1-2\delta$, for $i=1,2$,
\[
\|k_i-k_i'\|_{H}
=
\left\|\frac1M\sum_{r=1}^M w_r^{(i)}-\mathbb E[w_1^{(i)}]\right\|_{H}
\;\lesssim\;
\frac{C_\sigma^2}{\sqrt{M}}\log\!\left(\frac{2}{\delta}\right).
\]

Combining these bounds with \eqref{eq:split-ntk} gives, with probability at least $1-2\delta$,
\[
\|k^{\mathrm{NTK}}_{d_1}-k^M_{d_1}\|_{L^2(\rho_u\otimes\rho_u)}
\;\lesssim\;
\bigl(1+\kappa_E^2R^2+\gamma^2\bigr)\,\frac{C_\sigma^2}{\sqrt M}
\log\!\left(\frac{2}{\delta}\right).
\]
Using $\|L_{k^{\mathrm{NTK}}_{d_1}}-L_{k^M_{d_1}}\|_{\mathrm{HS}}
=\|k^{\mathrm{NTK}}_{d_1}-k^M_{d_1}\|_{L^2(\rho_u\otimes\rho_u)}$ completes the proof.
\end{proof}

\subsection{Proofs for Subsection \ref{section: Neural SGD to Kernel SGD}}
Throughout this subsection, we work with an equivalent augmented-input
representation in which the bias parameters are absorbed into the weights.
Specifically, the original model
\[
f_\Theta(x)
=
\frac{1}{\sqrt M}\sum_{r=1}^M a_r\,\sigma(b_r^\top x+\gamma c_r)
\]
can be viewed as a two-layer network acting on an augmented input obtained by
appending a constant coordinate.

Moreover, since $\|u\|_{\mathcal U}\le R$ almost surely and the encoder
$\mathcal E_1^{d_1}$ is uniformly bounded in operator norm, all inputs appearing
in this subsection—namely, the encoded inputs and their augmented versions after
absorbing the bias—are uniformly bounded in Euclidean norm.
Below, \(x\) and \(b_r\) denote the augmented variables
\((x,\gamma)\) and \((b_r,c_r)\), respectively, and the network
and empirical NTK are understood in this representation.

We begin by establishing uniform bounds on the network output and on the deviation
of the parameter gradient from its value at initialization.
\begin{lemma}\label{lemma3}
Suppose that Assumption~\ref{assumption: NTK} and
Assumption~\ref{assumption: sample 1} hold.
Let $\Theta_1$ denote the initialization.
If
\[
M \;\ge\; \|\Theta-\Theta_1\|_2^2,
\]
then the following bounds hold:
\[
\sup_{\|x\|_2\le R} |f_\Theta(x)|
\;\le\;
C_\sigma(1+2R)\,\|\Theta-\Theta_1\|_2,
\]
and
\[
\sup_{\|x\|_2\le R}
\bigl\|
\nabla_\Theta f_\Theta(x)
-
\nabla_\Theta f_\Theta(x)\big|_{\Theta=\Theta_1}
\bigr\|_2^2
\;\le\;
\frac{R^2(3+2R^2)C_\sigma^2}{M}\,
\|\Theta-\Theta_1\|_2^2.
\]
\end{lemma}

\begin{proof}
Recall that under the symmetric initialization,
$f_{\Theta_1}=0$.
For any input $x$, we estimate
\[
|f_\Theta(x)|
=
|f_\Theta(x)-f_{\Theta_1}(x)|.
\]
Writing $\Theta=(a_r,b_r)_{r=1}^M$ and
$\Theta_1=(a_r^{(1)},b_r^{(1)})_{r=1}^M$, we have
\begin{equation*}
\begin{aligned}
|f_\Theta(x)|
&=
\Bigl|
\frac{1}{\sqrt M}\sum_{r=1}^M
\bigl(
a_r\sigma(b_r^\top x)
-
a_r^{(1)}\sigma((b_r^{(1)})^\top x)
\bigr)
\Bigr|
\\
&\le
\frac{1}{\sqrt M}\sum_{r=1}^M
|a_r-a_r^{(1)}|\,|\sigma((b_r^{(1)})^\top x)|
\\
&\quad
+
\frac{1}{\sqrt M}\sum_{r=1}^M
|a_r-a_r^{(1)}|\,
|\sigma(b_r^\top x)-\sigma((b_r^{(1)})^\top x)|
\\
&\quad
+
\frac{1}{\sqrt M}\sum_{r=1}^M
|a_r^{(1)}|\,
|\sigma(b_r^\top x)-\sigma((b_r^{(1)})^\top x)|.
\end{aligned}
\end{equation*}
Using Assumption~\ref{assumption: NTK} and $\|x\|_2\le R$, we obtain
\begin{equation*}
\begin{aligned}
|f_\Theta(x)|
&\le
C_\sigma\frac{1}{\sqrt M}\sum_{r=1}^M|a_r-a_r^{(1)}|
+
C_\sigma R\frac{1}{\sqrt M}\sum_{r=1}^M
|a_r-a_r^{(1)}|\,\|b_r-b_r^{(1)}\|_2
\\
&\quad
+
C_\sigma R\frac{1}{\sqrt M}\sum_{r=1}^M\|b_r-b_r^{(1)}\|_2
\\
&\le
C_\sigma\|\Theta-\Theta_1\|_2
+
\frac{C_\sigma R}{\sqrt M}\|\Theta-\Theta_1\|_2^2
+
C_\sigma R\|\Theta-\Theta_1\|_2.
\end{aligned}
\end{equation*}
Under the condition $M\ge\|\Theta-\Theta_1\|_2^2$, this yields
\[
\|f_\Theta\|_\infty
\le
C_\sigma(1+2R)\,\|\Theta-\Theta_1\|_2.
\]

We next bound the gradient deviation.
A direct computation gives
\begin{equation*}
\begin{aligned}
\bigl\|
\nabla_\Theta f_\Theta(x)
-
\nabla_\Theta f_\Theta(x)\big|_{\Theta=\Theta_1}
\bigr\|_2^2
&=
\frac{1}{M}\sum_{r=1}^M
\Bigl(
\big|\sigma(b_r^\top x)-\sigma((b_r^{(1)})^\top x)\big|^2
\\
&\qquad
+
\big|a_r\sigma'(b_r^\top x)-a_r^{(1)}\sigma'((b_r^{(1)})^\top x)\big|^2
\|x\|_2^2
\Bigr)
\\
&\le
\frac{C_\sigma^2R^2}{M}\sum_{r=1}^M
\Bigl(
\|b_r-b_r^{(1)}\|_2^2
+
2|a_r-a_r^{(1)}|^2
+
2R^2\|b_r-b_r^{(1)}\|_2^2
\Bigr)
\\
&\le
\frac{R^2(3+2R^2)C_\sigma^2}{M}\,
\|\Theta-\Theta_1\|_2^2.
\end{aligned}
\end{equation*}
This completes the proof.
\end{proof}

To relate SGD for the two-layer network to a kernel-based recursion, we
linearize the network output with respect to the parameters around the
initialization $\Theta_1$. Specifically, we consider the first-order
approximation
\begin{equation}\label{linear}
\begin{aligned}
f_\Theta^{\mathrm{lin}}(x)
&:= \big\langle\nabla_\Theta f_{\Theta}(x)\big|_{\Theta=\Theta_1},\,\Theta-\Theta_1\big\rangle_2 \\
&= \frac{1}{\sqrt{M}}\sum_{r=1}^M\Big(
(a_r-a_r^{(1)})\,\sigma\big((b_r^{(1)})^\top x\big)
+a_r^{(1)}\,\sigma'\big((b_r^{(1)})^\top x\big)\,(b_r-b_r^{(1)})^\top x
\Big).
\end{aligned}
\end{equation}

\begin{proposition}\label{prop15}
Suppose that Assumption~\ref{assumption: NTK} and
Assumption~\ref{assumption: sample 1} hold.
Let $f_\Theta^{\mathrm{lin}}$ be the linearization defined in \eqref{linear}.
Then
the following bound holds:
\[
\sup_{\|x\|_2\le R}
\bigl|f_\Theta(x)-f_\Theta^{\mathrm{lin}}(x)\bigr|
\;\le\;
\left(\frac{R^2}{2}+R\right)
\frac{C_\sigma}{\sqrt{M}}\,
\|\Theta-\Theta_1\|_2^2.
\]
\end{proposition}

\begin{proof}
Recall that under the symmetric initialization we have
$f_{\Theta_1}\equiv 0$.
Fix $x$ with $\|x\|_2\le R$.
We estimate the pointwise deviation
$|f_\Theta(x)-f_\Theta^{\mathrm{lin}}(x)|$.

By definition,
\begin{equation*}
\begin{aligned}
|f_\Theta(x)-f_\Theta^{\mathrm{lin}}(x)|
&=
\bigl|f_\Theta(x)-f_{\Theta_1}(x)-f_\Theta^{\mathrm{lin}}(x)\bigr| \\
&=
\Bigg|
\frac{1}{\sqrt{M}}\sum_{r=1}^M
\Big(
a_r\sigma(b_r^\top x)
-
a_r^{(1)}\sigma\bigl((b_r^{(1)})^\top x\bigr)
\\
&\qquad\qquad
-
(a_r-a_r^{(1)})\sigma\bigl((b_r^{(1)})^\top x\bigr)
-
a_r^{(1)}\sigma'\bigl((b_r^{(1)})^\top x\bigr)
(b_r-b_r^{(1)})^\top x
\Big)
\Bigg|.
\end{aligned}
\end{equation*}

Rearranging the terms yields
\begin{equation*}
\begin{aligned}
|f_\Theta(x)-f_\Theta^{\mathrm{lin}}(x)|
&\le
\frac{1}{\sqrt{M}}\sum_{r=1}^M
\Big|
(a_r-a_r^{(1)})
\bigl(\sigma(b_r^\top x)-\sigma((b_r^{(1)})^\top x)\bigr)
\\
&\qquad\qquad
+
a_r^{(1)}
\Big(
\sigma(b_r^\top x)
-
\sigma((b_r^{(1)})^\top x)
-
\sigma'((b_r^{(1)})^\top x)
(b_r-b_r^{(1)})^\top x
\Big)
\Big|.
\end{aligned}
\end{equation*}

By Assumption~\ref{assumption: NTK}, the activation function $\sigma$
has uniformly bounded first and second derivatives.
Applying the mean value theorem and Taylor's expansion with remainder,
and using $\|x\|_2\le R$, we obtain
\begin{equation*}
\begin{aligned}
|f_\Theta(x)-f_\Theta^{\mathrm{lin}}(x)|
&\le
\frac{1}{\sqrt{M}}\sum_{r=1}^M
\left[
C_\sigma R |a_r-a_r^{(1)}|\,\|b_r-b_r^{(1)}\|_2
+
\frac{C_\sigma R^2}{2}\|b_r-b_r^{(1)}\|_2^2
\right].
\end{aligned}
\end{equation*}

Applying the Cauchy--Schwarz inequality to the first term gives
\begin{equation*}
\begin{aligned}
|f_\Theta(x)-f_\Theta^{\mathrm{lin}}(x)|
&\le
\frac{RC_\sigma}{\sqrt{M}}
\Big(\sum_{r=1}^M |a_r-a_r^{(1)}|^2\Big)^{1/2}
\Big(\sum_{r=1}^M \|b_r-b_r^{(1)}\|_2^2\Big)^{1/2}
\\
&\quad
+
\frac{R^2C_\sigma}{2\sqrt{M}}
\sum_{r=1}^M \|b_r-b_r^{(1)}\|_2^2
\\
&\le
\left(\frac{R^2}{2}+R\right)
\frac{C_\sigma}{\sqrt{M}}\,
\|\Theta-\Theta_1\|_2^2.
\end{aligned}
\end{equation*}

Taking the supremum over $\|x\|_2\le R$ completes the proof.
\end{proof}

We next restate and prove a fundamental a priori bound on the deviation of
the SGD iterates from their initialization in the over-parameterized regime.
This result corresponds to Proposition~\ref{prop13} in
Subsection~\ref{section: Neural SGD to Kernel SGD} and provides a quantitative control on the growth of
the parameter deviation along the training trajectory.
In particular, it shows that when the network width $M$ is sufficiently large,
the SGD iterates remain within a neighborhood of radius
$C_2\lambda^{-1}$ around the initialization throughout training.
\begin{proposition}[Uniform stability of the parameter trajectory (decreasing stepsizes)]\label{prop16}
Suppose that Assumption~\ref{assumption: NTK} and
Assumption~\ref{assumption: sample 1} hold. Let $0<\lambda\le1$. 
Consider SGD with polynomially decaying step sizes
$\{\eta_t=\eta_1 t^{-\theta}\}_{t\in\mathbb{N}_T}$,
where $0<\theta<1$ and
\[
0<\eta_1 \le
\min\Bigl\{\frac{1}{(1+R^2)C_\sigma^2+1},\,1-\theta\Bigr\}.
\]
Assume that
\[
M \ge C_1 \lambda^{-4}.
\]
Then, for all $1 \le t \le T+1$, the SGD iterates satisfy
\[
\|\Theta_t - \Theta_1\|_2 \le C_2 \lambda^{-1}.
\]
Here,
\[
C_1
:= \max\Bigl\{C_2^2,\;
\bigl(4\bigl(\tfrac{R^2}{2}+R\bigr)\sqrt{1+R^2}\,C_\sigma^2  C_2\bigr)^2,\;
\bigl(4R\sqrt{3+2R^2}\,(C_\sigma(1+2R)+B)\,C_\sigma C_2 \bigr)^2
\Bigr\},
\]
and
\[
C_2 := \max\bigl\{1,\;2B\sqrt{1+R^2}\,C_\sigma\bigr\}.
\]
\end{proposition}

\begin{proof}
We proceed by induction on $t$.
First, note that by direct computation,
\begin{equation}\label{temp23}
\begin{aligned}
\bigl\|\nabla_{\Theta}f_{\Theta}(x)\big|_{\Theta=\Theta_1}\bigr\|_2^2
&=
\frac{1}{M}\sum_{r=1}^{M}
\Bigl[
\bigl(\sigma((b_{r}^{(1)})^{\top}x)\bigr)^{2}
+
\bigl(a_{r}^{(1)}\sigma'((b_{r}^{(1)})^{\top}x)\bigr)^{2}\|x\|_{2}^{2}
\Bigr]
\\
&\le (1+R^2)C_\sigma^2 .
\end{aligned}
\end{equation}

The claim trivially holds for $t=1$.
Assume that for some $k\le T$,
\[
\|\Theta_t-\Theta_1\|_2 \le C_2\lambda^{-1},
\qquad \forall\,1\le t\le k .
\]
We now establish the bound for $t=k+1$. Since $C_1\ge C_2^2$ and $0<\lambda\le1$,
\[
M\ge C_1\lambda^{-4}\ge C_2^2\lambda^{-2}
\ge \max_{1\le i\le k}\|\Theta_i-\Theta_1\|_2^2.
\]
Thus Lemma~\ref{lemma3} applies to these iterates.

Using the SGD recursion and inserting the linearized model, we obtain
\begin{equation*}
        \begin{aligned}\Theta_{k+1}-\Theta_{1}&=\l(1-\eta_{k}\lambda\r)\l(\Theta_{k}-\Theta_{1}\r)-\eta_{k}\l(f_{\Theta_k}(x_{k})-v_{k}\r)\nabla_{\Theta}f_{\Theta}(x_{k})|_{\Theta=\Theta_{k}}\\&=\l(1-\eta_{k}\lambda\r)\l(\Theta_{k}-\Theta_1\r)-\eta_{k}\l(f_{\Theta_k}(x_{k})-v_{k}\r)\nabla_{\Theta}f_{\Theta}(x_{k})|_{\Theta=\Theta_1}\\&\quad -\eta_{k}\l(f_{\Theta_{k}}(x_{k})-v_{k}\r)\left(\nabla_{\Theta}f_{\Theta}(x_{k})|_{\Theta=\Theta_k}-\nabla_{\Theta}f_{\Theta}(x_{k})|_{\Theta=\Theta_1}\right)\\&=\l(1-\eta_{k}\lambda\r)\l(\Theta_{k}-\Theta_1\r)-\eta_{k}f^{\mathrm{lin}}_{\Theta_{k}}(x_{k})\nabla_{\Theta}f_{\Theta}(x_{k})|_{\Theta=\Theta_1}\\&\quad -\left.\eta_{k}(f_{\Theta_k}(x_{k})-f^{\mathrm{lin}}_{\Theta_{k}}(x_{k}))\nabla_{\Theta}f_{\Theta}(x_{k})\right|_{\Theta=\Theta_1}\\&\quad-\eta_{k}\l(f_{\Theta_k}(x_{k})-v_{k}\r)\left(\nabla_{\Theta}f_{\Theta}(x_{k})|_{\Theta=\Theta_k}-\nabla_{\Theta}f_{\Theta}(x_{k})|_{\Theta=\Theta_1}\right)\\&\quad+\eta_{k}v_{k}\nabla_{\Theta}f_{\Theta}(x_{k})|_{\Theta=\Theta_1}\\&=\l(I-\eta_{k}(A_{k}+\lambda I)\r)(\Theta_{k}-\Theta_1)-\eta_{k}\Lambda_{k}^{(1)}-\eta_{k}\Lambda_{k}^{(2)}+\eta_{k}\Lambda_{k}^{(3)},
        \end{aligned}
    \end{equation*}
where, for $1\le i\le T$,
\begin{equation*}
\begin{aligned}
A_i
&:= \nabla_\Theta f_\Theta(x_i)\big|_{\Theta=\Theta_1}
\bigl(\nabla_\Theta f_\Theta(x_i)\big|_{\Theta=\Theta_1}\bigr)^\top,\\
\Lambda_i^{(1)}
&:= (f_{\Theta_i}(x_i)-f^{\mathrm{lin}}_{\Theta_i}(x_i))
\nabla_\Theta f_\Theta(x_i)\big|_{\Theta=\Theta_1},\\
\Lambda_i^{(2)}
&:= (f_{\Theta_i}(x_i)-v_i)
\bigl(\nabla_\Theta f_\Theta(x_i)\big|_{\Theta=\Theta_i}
-\nabla_\Theta f_\Theta(x_i)\big|_{\Theta=\Theta_1}\bigr),\\
\Lambda_i^{(3)}
&:= v_i \nabla_\Theta f_\Theta(x_i)\big|_{\Theta=\Theta_1}.
\end{aligned}
\end{equation*}
Iterating the recursion yields
\[
\Theta_{k+1}-\Theta_1
=
-\sum_{i=1}^{k}\eta_i
\prod_{j=i+1}^{k}\bigl(I-\eta_j(A_j+\lambda I)\bigr)
\bigl(\Lambda_i^{(1)}+\Lambda_i^{(2)}-\Lambda_i^{(3)}\bigr).
\]

By \eqref{temp23}, $A_j$ is positive semidefinite and
$\|A_j\|\le(1+R^2)C_\sigma^2$. The step-size condition therefore gives
\[
\|I-\eta_j(A_j+\lambda I)\|\le1-\lambda\eta_j.
\]
The following identity holds:
\begin{equation}\label{eq:ntk-geometric-sum}
\sum_{i=1}^{k}\eta_i
\prod_{j=i+1}^{k}(1-\lambda\eta_j)
=\sum_{i=1}^{k}\frac{1-(1-\lambda\eta_i)}{\lambda}
\prod_{j=i+1}^{k}(1-\lambda\eta_j)=
\frac{1-\prod_{j=1}^{k}(1-\lambda\eta_j)}{\lambda}
\le\lambda^{-1}.
\end{equation}
Consequently, 
\[
\|\Theta_{k+1}-\Theta_1\|_2
\le
\sum_{i=1}^{k}\eta_i
\prod_{j=i+1}^{k}(1-\lambda\eta_j)
\sum_{\ell=1}^{3}\|\Lambda_i^{(\ell)}\|_2
\le
\lambda^{-1}\sum_{\ell=1}^{3}
\max_{1\le i\le k}\|\Lambda_i^{(\ell)}\|_2.
\]



\paragraph{Contribution of $\Lambda_i^{(1)}$.}
Using Proposition~\ref{prop15}, \eqref{temp23},
and the induction hypothesis,
\[
\begin{aligned}
\lambda^{-1}\max_{1\le i\le k}\|\Lambda_i^{(1)}\|_2
&\le\lambda^{-1}\max_{1\le i\le k}\l\|f_{\Theta_{i}}-f^{\mathrm{lin}}_{\Theta_{i}}\r\|_\infty\l\|\nabla_{\Theta}f_{\Theta}(x_{i})|_{\Theta=\Theta_1}\r\|_2
\\&\le
\Bigl(\frac{R^2}{2}+R\Bigr)\sqrt{1+R^2}\,
C_\sigma^2 C_2^2\frac{\lambda^{-3}}{\sqrt M}\\
&\le \frac{C_2}{4}\lambda^{-1}.
\end{aligned}
\]

\paragraph{Contribution of $\Lambda_i^{(2)}$.}
By Lemma~\ref{lemma3}, the induction hypothesis,
and $C_2\lambda^{-1}\ge1$,
\[
\begin{aligned}
\lambda^{-1}\max_{1\le i\le k}\|\Lambda_i^{(2)}\|_2
&\le R\sqrt{3+2R^2}C_\sigma\lambda^{-1}\max_{1\le i\le k}\l(\l\|f_{\Theta_{i}}\r\|_\infty+B\r)\frac{\left\|\Theta_i-\Theta_{1}\right\|_2}{\sqrt{M}}
\\&\le
\frac{R\sqrt{3+2R^2}C_\sigma}{\lambda\sqrt M}
\bigl(C_\sigma(1+2R)C_2\lambda^{-1}+B\bigr)
C_2\lambda^{-1}\\
&\le
R\sqrt{3+2R^2}C_\sigma
\bigl(C_\sigma(1+2R)+B\bigr)
C_2^2\frac{\lambda^{-3}}{\sqrt M}\\
&\le \frac{C_2}{4}\lambda^{-1}.
\end{aligned}
\]


\paragraph{Contribution of $\Lambda_i^{(3)}$.}
Using \eqref{temp23} and the definition of $C_2$,
\[
\lambda^{-1}\max_{1\le i\le k}\|\Lambda_i^{(3)}\|_2
\le B\sqrt{1+R^2}C_\sigma\lambda^{-1}
\le \frac{C_2}{2}\lambda^{-1}.
\]

    
Combining the above estimates yields
\[
\|\Theta_{k+1}-\Theta_1\|_2
\le C_2\lambda^{-1},
\]
which closes the induction and completes the proof.
\end{proof}

We now quantify the discrepancy between the iterate generated by the
linearized network at initialization and the corresponding kernel-based
SGD iterate.
The following result shows that, in the over-parameterized regime,
the function induced by the linearized network remains uniformly close
to the kernel SGD iterate throughout training.

This estimate provides a quantitative justification for approximating
the linearized network dynamics by the associated kernel SGD dynamics
over the entire training horizon.
\begin{proposition} \label{prop17}
Let \(x_t:=(\mathcal E_1^{d_1}u_t,\gamma)\). Under the conditions of Proposition \ref{prop16}, the following uniform bound holds:
\[
\sup_{\|x\|_2 \le R}
\bigl| f^{\mathrm{lin}}_{\Theta_{t+1}}(x) - f_{t+1}(x) \bigr|
\;\lesssim\;
\frac{\lambda^{-3}}{\sqrt{M}},
\qquad \forall\, 0 \le t \le T .
\]
\end{proposition}

\begin{proof}
The case $t=0$ is immediate from initialization.
Fix $1\le t\le T$. 
Denote by $\mathcal{H}_{k^M}$ the RKHS associated with $k^M$, equipped with norm
$\|\cdot\|_{k^M}$.
By the definitions of $f^{\mathrm{lin}}_{\Theta_{t+1}}$ and $f_{t+1}$, we have
\begin{equation*}
\begin{aligned}
f^{\mathrm{lin}}_{\Theta_{t+1}}-f_{t+1}
&=
\bigl\langle\nabla_{\Theta}f_{\Theta}(\cdot)\big|_{\Theta=\Theta_1},\Theta_{t+1}-\Theta_1\bigr\rangle_2
-\Bigl((1-\eta_t\lambda)f_t-\eta_t\bigl(f_t(x_t)-v_t\bigr)k^M(\cdot,x_t)\Bigr)
\\
&=
(1-\eta_t\lambda)\bigl(f^{\mathrm{lin}}_{\Theta_t}-f_t\bigr)
-\eta_t\bigl(f_{\Theta_t}(x_t)-v_t\bigr)
\Bigl\langle\nabla_{\Theta}f_{\Theta}(\cdot)\big|_{\Theta=\Theta_1},
\nabla_{\Theta}f_{\Theta}(x_t)\big|_{\Theta=\Theta_t}\Bigr\rangle_2
\\
&\quad
+\eta_t\bigl(f_t(x_t)-v_t\bigr)k^M(\cdot,x_t).
\end{aligned}
\end{equation*}
Adding and subtracting $f^{\mathrm{lin}}_{\Theta_t}(x_t)$ in the last term yields
\begin{equation*}
\begin{aligned}
f^{\mathrm{lin}}_{\Theta_{t+1}}-f_{t+1}
&=
(1-\eta_t\lambda)\bigl(f^{\mathrm{lin}}_{\Theta_t}-f_t\bigr)
-\eta_t\bigl(f_{\Theta_t}(x_t)-v_t\bigr)
\Bigl\langle\nabla_{\Theta}f_{\Theta}(\cdot)\big|_{\Theta=\Theta_1},
\nabla_{\Theta}f_{\Theta}(x_t)\big|_{\Theta=\Theta_t}\Bigr\rangle_2
\\
&\quad
+\eta_t\bigl(f_t(x_t)-f^{\mathrm{lin}}_{\Theta_t}(x_t)\bigr)k^M(\cdot,x_t)
+\eta_t\bigl(f^{\mathrm{lin}}_{\Theta_t}(x_t)-v_t\bigr)k^M(\cdot,x_t).
\end{aligned}
\end{equation*}

Since both $f_t$ and $f^{\mathrm{lin}}_{\Theta_t}$ belong to $\mathcal{H}_{k^M}$
(by the iterative construction and the representer theorem, see, e.g., \cite[Theorem~4.21]{christmann2008support}), we have
\[
\bigl(f_t(x_t)-f^{\mathrm{lin}}_{\Theta_t}(x_t)\bigr)\,k^M(\cdot,x_t)
=
\bigl(k^M(\cdot,x_t)\otimes k^M(\cdot,x_t)\bigr)\bigl(f_t-f^{\mathrm{lin}}_{\Theta_t}\bigr).
\]
Moreover, by the definitions of the linearized model and the empirical NTK $k^M$,
\[
f^{\mathrm{lin}}_{\Theta_t}(x_t)
=
\bigl\langle\nabla_{\Theta}f_{\Theta}(x_t)\big|_{\Theta=\Theta_1},\Theta_t-\Theta_1\bigr\rangle_2,
\qquad
k^M(\cdot,x_t)
=
\Bigl\langle\nabla_{\Theta}f_{\Theta}(\cdot)\big|_{\Theta=\Theta_1},
\nabla_{\Theta}f_{\Theta}(x_t)\big|_{\Theta=\Theta_1}\Bigr\rangle_2 .
\]
Substituting these identities, we obtain
\begin{equation*}
\begin{aligned}
f^{\mathrm{lin}}_{\Theta_{t+1}}-f_{t+1}
&=
\bigl(I-\eta_t\lambda\bigr)\bigl(f^{\mathrm{lin}}_{\Theta_t}-f_t\bigr)
+\eta_t\bigl(k^M(\cdot,x_t)\otimes k^M(\cdot,x_t)\bigr)\bigl(f_t-f^{\mathrm{lin}}_{\Theta_t}\bigr)
\\
&\quad
-\eta_t\bigl(f_{\Theta_t}(x_t)-v_t\bigr)
\Bigl\langle\nabla_{\Theta}f_{\Theta}(\cdot)\big|_{\Theta=\Theta_1},
\nabla_{\Theta}f_{\Theta}(x_t)\big|_{\Theta=\Theta_t}\Bigr\rangle_2
+\eta_t\bigl(f^{\mathrm{lin}}_{\Theta_t}(x_t)-v_t\bigr)k^M(\cdot,x_t)
\\
&=
\Bigl(I-\eta_t\bigl(k^M(\cdot,x_t)\otimes k^M(\cdot,x_t)+\lambda I\bigr)\Bigr)
\bigl(f^{\mathrm{lin}}_{\Theta_t}-f_t\bigr)
\\
&\quad
-\eta_t\bigl(f_{\Theta_t}(x_t)-v_t\bigr)
\Bigl\langle\nabla_{\Theta}f_{\Theta}(\cdot)\big|_{\Theta=\Theta_1},
\nabla_{\Theta}f_{\Theta}(x_t)\big|_{\Theta=\Theta_t}
-\nabla_{\Theta}f_{\Theta}(x_t)\big|_{\Theta=\Theta_1}\Bigr\rangle_2
\\
&\quad
-\eta_t\bigl(f_{\Theta_t}(x_t)-f^{\mathrm{lin}}_{\Theta_t}(x_t)\bigr)\,k^M(\cdot,x_t).
\end{aligned}
\end{equation*}
For $1\le i\le t$, define
\begin{equation*}
\begin{aligned}
\Gamma_i
&:=
\bigl(f_{\Theta_i}(x_i)-v_i\bigr)
\Bigl\langle\nabla_{\Theta}f_{\Theta}(\cdot)\big|_{\Theta=\Theta_1},
\nabla_{\Theta}f_{\Theta}(x_i)\big|_{\Theta=\Theta_i}
-\nabla_{\Theta}f_{\Theta}(x_i)\big|_{\Theta=\Theta_1}\Bigr\rangle_2
\in\mathcal{H}_{k^M},
\\
\Gamma_i'
&:=
\bigl(f_{\Theta_i}(x_i)-f^{\mathrm{lin}}_{\Theta_i}(x_i)\bigr)\,k^M(\cdot,x_i)
\in\mathcal{H}_{k^M}.
\end{aligned}
\end{equation*}
Then,
\begin{equation*}
f^{\mathrm{lin}}_{\Theta_{t+1}}-f_{t+1}
=
\Bigl(I-\eta_t\bigl(k^M(\cdot,x_t)\otimes k^M(\cdot,x_t)+\lambda I\bigr)\Bigr)
\bigl(f^{\mathrm{lin}}_{\Theta_t}-f_t\bigr)
-\eta_t\bigl(\Gamma_t+\Gamma_t'\bigr).
\end{equation*}
Iterating the recursion yields
    \[
    f^{\mathrm{lin}}_{\Theta_{t+1}}-f_{t+1}=-\sum_{i=1}^t\eta_i\prod_{j=i+1}^{t}\l(I-\eta_j\l(k^{M}(\cdot,x_j)\otimes k^{M}(\cdot,x_j)+\lambda I\r)\r)\left(\Gamma_{i}^{}+\Gamma_{i}'\right).
    \]
    Since
\[
\|k^M(\cdot,x_j)\otimes k^M(\cdot,x_j)\|
=k^M(x_j,x_j)\le(1+R^2)C_\sigma^2,
\]
the step-size condition gives
\[
\left\|I-\eta_j
\bigl(k^M(\cdot,x_j)\otimes k^M(\cdot,x_j)+\lambda I\bigr)\right\|
\le1-\lambda\eta_j.
\]
By the reproducing property and \eqref{eq:ntk-geometric-sum},
\[
\begin{aligned}
\sup_{\|x\|_2\le R}
|f^{\mathrm{lin}}_{\Theta_{t+1}}(x)-f_{t+1}(x)|
&\le \sqrt{1+R^2}C_\sigma
\sum_{i=1}^{t}\eta_i\prod_{j=i+1}^{t}(1-\lambda\eta_j)
\bigl(\|\Gamma_i\|_{k^M}+\|\Gamma_i'\|_{k^M}\bigr)\\
&\le \frac{\sqrt{1+R^2}C_\sigma}{\lambda}
\max_{1\le i\le t}
\bigl(\|\Gamma_i\|_{k^M}+\|\Gamma_i'\|_{k^M}\bigr).
\end{aligned}
\]


    We estimate the two terms separately. 
    
    \paragraph{Contribution of $\Gamma_i$.} 
    By the feature-map representation of $k^M$,
Lemma~\ref{lemma3}, and Proposition~\ref{prop16},
    \[
\begin{aligned}
\|\Gamma_i\|_{k^M}
&\le
(\|f_{\Theta_i}\|_\infty+B)
\left\|
\nabla_\Theta f_\Theta(x_i)\big|_{\Theta=\Theta_i}
-\nabla_\Theta f_\Theta(x_i)\big|_{\Theta=\Theta_1}
\right\|_2\\
&\lesssim
\frac{(\|\Theta_i-\Theta_1\|_2+B)
\|\Theta_i-\Theta_1\|_2}{\sqrt M}\\
&\lesssim \frac{\lambda^{-2}}{\sqrt M},
\qquad \forall1\le i\le t.
\end{aligned}
\]

    \paragraph{Contribution of $\Gamma_i'$.} 
    Using Proposition~\ref{prop15}, Proposition~\ref{prop16},
and $\|k^M(\cdot,x_i)\|_{k^M}\le\sqrt{1+R^2}C_\sigma$,
    \[
\begin{aligned}
\|\Gamma_i'\|_{k^M}
&\le \sqrt{1+R^2}C_\sigma
\|f_{\Theta_i}-f^{\mathrm{lin}}_{\Theta_i}\|_\infty\\
&\lesssim
\frac{\|\Theta_i-\Theta_1\|_2^2}{\sqrt M}
\lesssim \frac{\lambda^{-2}}{\sqrt M},
\qquad \forall1\le i\le t.
\end{aligned}
\]

    Combining the above estimates yields
    \[
    \sup_{\|x\|_2 \le R}
\bigl| f^{\mathrm{lin}}_{\Theta_{t+1}}(x) - f_{t+1}(x) \bigr|\lesssim\frac{\lambda^{-3}}{\sqrt{M}},
    \]
    which completes the proof.
\end{proof}

We are now in a position to combine the preceding results and establish
a uniform bound between the nonlinear network iterate and the corresponding
kernel-based SGD iterate.
\begin{proposition}\label{prop18}
Under the conditions of Proposition \ref{prop16}, for all $0 \le t \le T$, the SGD iterates satisfy
\[
\sup_{\|x\|_2 \le R}\l|f_{\Theta_{t+1}}(x)-f_{t+1}(x)\r|\lesssim\frac{\lambda^{-3}}{\sqrt{M}}.
\]
Consequently,
\[
\operatorname*{ess\,sup}_{u\sim\rho_u}\l|\mg_{t+1}^{\mathrm{NN}}(u)-\mg_{t+1}^{}(u)\r|\lesssim\frac{\lambda^{-3}}{\sqrt{M}}.
\]
\end{proposition}

\begin{proof}
    By Proposition \ref{prop17}, Proposition \ref{prop16}, and Proposition \ref{prop15}, we decompose the error as
    \begin{equation*}
        \begin{aligned}
            \sup_{\|x\|_2 \le R}\l|f_{\Theta_{t+1}}(x)-f_{t+1}(x)\r|&\leq
            \sup_{\|x\|_2 \le R}\l|f_{\Theta_{t+1}}(x)-f^{\mathrm{lin}}_{\Theta_{t+1}}(x)\r|+\sup_{\|x\|_2 \le R}\l|f^{\mathrm{lin}}_{\Theta_{t+1}}(x)-f_{t+1}(x)\r|
            \\&\lesssim
\frac{\lambda^{-2}}{\sqrt{M}}
+\frac{\lambda^{-3}}{\sqrt{M}}
\\&\lesssim
\frac{\lambda^{-3}}{\sqrt{M}},
        \end{aligned}
    \end{equation*}
    where the last inequality uses $0<\lambda\le1$.
This completes the proof.
\end{proof}

\begin{proof}[Proof of Theorem \ref{NTK SGD1}]
    The result follows directly from Theorem~\ref{Thm1} and 
Proposition~\ref{prop18}.
\end{proof}

We next consider the constant-stepsize regime and derive an analogue of
Proposition~\ref{prop16} for the parameter deviation.
The argument follows similar lines to the decaying-stepsize case, with minor
modifications.
We therefore give a brief proof.
\begin{proposition}[Uniform stability of the parameter trajectory (constant stepsizes)] \label{prop19}
Suppose that Assumption \ref{assumption: NTK} and Assumption \ref{assumption: sample 1} hold. Let $0<\lambda\le1$. Consider SGD with constant step sizes $\l\{\eta_t=\eta T^{-\theta'}\r\}_{t\in\bn_T}$,
where $0<\theta'<1$ and
$0<\eta\le((1+R^2)C_\sigma^2+1)^{-1}$. Assume that
\[
M \gtrsim \lambda^{-4}.
\]
Then, for all $1 \le t \le T+1$, the SGD iterates satisfy
\[
\|\Theta_t - \Theta_1\|_2 \lesssim \lambda^{-1}.
\]
\end{proposition}

\begin{proof}
The proof of Proposition~\ref{prop16} applies to any positive
step-size sequence satisfying
$\eta_j((1+R^2)C_\sigma^2+\lambda)\le1$,
using \eqref{eq:ntk-geometric-sum}. Indeed, since $0<\lambda\le1$ and $\eta_j=\eta T^{-\theta'}\le\eta$,
\[
\eta_j\bigl((1+R^2)C_\sigma^2+\lambda\bigr)
\le \eta\bigl((1+R^2)C_\sigma^2+1\bigr)
\le1.
\]
The present schedule satisfies this condition.
Taking the implicit constant in the width condition at least $C_1$,
the same induction gives
$\|\Theta_t-\Theta_1\|_2\le C_2\lambda^{-1}$
for all $1\le t\le T+1$.
\end{proof}

We next bound the discrepancy between the linearized network iterate and the
corresponding kernel-based SGD iterate in the constant-stepsize regime.

\begin{proposition} \label{prop23}
Under the conditions of Proposition \ref{prop19}, the following uniform bound holds:
\[
    \sup_{\|x\|_2 \le R}
\bigl|f^{\mathrm{lin}}_{\Theta_{t+1}}(x)-f_{t+1}(x)\big|\lesssim\frac{\lambda^{-3}}{\sqrt{M}},\qquad \forall\, 0\le t\le T.
\]
\end{proposition}

\begin{proof}
The argument in the proof of Proposition~\ref{prop17} uses
\eqref{eq:ntk-geometric-sum}, the step-size condition
$\eta_j((1+R^2)C_\sigma^2+\lambda)\le1$,
and the uniform bound
$\|\Theta_j-\Theta_1\|_2\lesssim\lambda^{-1}$.
The step-size condition follows from the assumptions of
Proposition~\ref{prop19}, and the uniform parameter bound
is its conclusion. Thus the same argument proves the asserted estimate.
\end{proof}

\begin{proof}[Proof of Theorem~\ref{NTK SGD2}]
By Proposition~\ref{prop23}, Proposition~\ref{prop19}, and Proposition~\ref{prop15}, we have
\begin{equation*}
\begin{aligned}
\sup_{\|x\|_2 \le R}
\bigl|f_{\Theta_{t+1}}(x)-f_{t+1}(x)\bigr|
&\le
\sup_{\|x\|_2 \le R}
\bigl|f_{\Theta_{t+1}}(x)-f^{\mathrm{lin}}_{\Theta_{t+1}}(x)\bigr|
+
\sup_{\|x\|_2 \le R}
\bigl|f^{\mathrm{lin}}_{\Theta_{t+1}}(x)-f_{t+1}(x)\bigr|
\\
&\lesssim
\frac{\lambda^{-2}}{\sqrt{M}}
+
\frac{\lambda^{-3}}{\sqrt{M}}
\\
&\lesssim
\frac{\lambda^{-3}}{\sqrt{M}}.
\end{aligned}
\end{equation*}
Here the last inequality uses $0<\lambda\le1$. Consequently,
\[
\operatorname*{ess\,sup}_{u\sim\rho_u}
\bigl|\mathcal{G}_{t+1}^{\mathrm{NN}}(u)-\mathcal{G}_{t+1}(u)\bigr|
\;\lesssim\;
\frac{\lambda^{-3}}{\sqrt{M}}.
\]

The conclusion now follows directly from Theorem~\ref{Thm2}.
\end{proof}

\subsection{Proofs for Subsection \ref{section: Neural Operators}}

\begin{proof}[Proof of Theorem~\ref{thm:NTK-encoder-decoder}]
We present the proof for the polynomially decaying step-size regime.
The constant step-size case follows by an entirely analogous argument and is
therefore omitted.

The analysis in Subsection~\ref{section: Neural SGD to Kernel SGD} treats the
scalar-output setting. Under the encoder--decoder architecture introduced in
Subsection~\ref{section: Neural Operators}, the network predicts an encoded output
in $\mathbb R^{d_2}$, and the objective and the SGD updates are separable across $j$; as a consequence,
each scalar subnetwork can be treated independently.
As a result, each encoded component can be analyzed using the scalar-output
results established in Subsection~\ref{section: Neural SGD to Kernel SGD},
with constants depending on the magnitude of the corresponding output
encoding coordinate.

By Assumption~\ref{ass:output-encoding}, the encoded outputs satisfy the coordinate-wise essential bounds
\[
\sup_{d_2\ge 1}
\sum_{j=1}^{d_2}
\operatorname*{ess\,sup}_{v\sim\rho_v}
\bigl|(\mathcal E_2^{d_2} v)_j\bigr|^2
< \infty.
\]
For each $j=1,\ldots,d_2$, define
\[
B_j
:=
\operatorname*{ess\,sup}_{v\sim\rho_v}
\bigl|(\mathcal E_2^{d_2} v)_j\bigr|.
\]


To retain the dependence on $B_j$, repeat the induction in the proof
of Proposition~\ref{prop16} with radius $cB_j\lambda^{-1}$,
where $c:=2\sqrt{1+R^2}C_\sigma$.
Keeping $B_j$ explicit in the three perturbation estimates
and using \eqref{eq:ntk-geometric-sum} gives
\[
\|\Theta^{(j)}_{k+1}-\Theta^{(j)}_1\|_2
\le \frac{cB_j}{2\lambda}
+C\frac{B_j^2\lambda^{-3}}{\sqrt M},
\]
where $C$ depends only on $R$ and $C_\sigma$.
For $B_j=0$, the iterates remain at initialization.
Thus the same induction shows that if
\[
M \gtrsim (B_j^2+1)\,\lambda^{-4},
\]
then the parameter iterates of the $j$-th subnetwork satisfy
\[
\|\Theta^{(j)}_t-\Theta^{(j)}_1\|_2
\;\lesssim\;
B_j\,\lambda^{-1},
\qquad 1\le t\le T+1.
\]

We decompose, for each $j=1,\ldots,d_2$,
\[
\sup_{\|x\|_2\le R}
\bigl|
f_{\Theta^{(j)}_{t+1}}(x)-f^{(j)}_{t+1}(x)
\bigr|
\le
\sup_{\|x\|_2\le R}
\bigl|
f_{\Theta^{(j)}_{t+1}}(x)-f^{\mathrm{lin}}_{\Theta^{(j)}_{t+1}}(x)
\bigr|
+
\sup_{\|x\|_2\le R}
\bigl|
f^{\mathrm{lin}}_{\Theta^{(j)}_{t+1}}(x)-f^{(j)}_{t+1}(x)
\bigr|.
\]
Repeating the estimates in the proof of Proposition~\ref{prop17}
with the preceding parameter bound gives
\[
\sup_{\|x\|_2\le R}
\bigl|
f^{\mathrm{lin}}_{\Theta^{(j)}_{t+1}}(x)-f^{(j)}_{t+1}(x)
\bigr|
\;\lesssim\;
\frac{B_j^{2}\,\lambda^{-3}}{\sqrt{M}},
\qquad 0\le t\le T.
\]
On the other hand, Proposition~\ref{prop15} yields
\begin{align*}
\sup_{\|x\|_2\le R}
\bigl|
f_{\Theta^{(j)}_{t+1}}(x)-f^{\mathrm{lin}}_{\Theta^{(j)}_{t+1}}(x)
\bigr|
&\le
\left(\frac{R^2}{2}+R\right)\frac{C_\sigma}{\sqrt{M}}\,
\|\Theta^{(j)}_{t+1}-\Theta^{(j)}_1\|_2^2  \\
&\lesssim
\frac{B_j^{2}\,\lambda^{-2}}{\sqrt{M}},
\end{align*}
where the last step uses the preceding parameter bound.
Combining the last three displays, we obtain
\[
\sup_{\|x\|_2\le R}
\bigl|
f_{\Theta^{(j)}_{t+1}}(x)-f^{(j)}_{t+1}(x)
\bigr|
\;\lesssim\;
\frac{B_j^{2}\,\lambda^{-3}}{\sqrt{M}},
\qquad 0\le t\le T.
\]

Aggregating the componentwise bounds yields
\begin{align*}
\sup_{\|x\|_2\le R}
\bigl\|
F_{\Theta_{t+1}}(x)-F_{t+1}(x)
\bigr\|_2
&\lesssim
\sqrt{\sum_{j=1}^{d_2}B_j^4}\,
\frac{\lambda^{-3}}{\sqrt{M}}.
\end{align*}

Finally, invoking the uniform boundedness of the decoder
$\mathcal D_2^{d_2}$ yields
\[
\operatorname*{ess\,sup}_{u\sim\rho_u}
\bigl\|
\mathcal G^{\mathrm{NN}}_{t+1}(u)-\mathcal G_{t+1}(u)
\bigr\|_{\mathcal V}
\;\lesssim\;
\sqrt{\sum_{j=1}^{d_2}B_j^4}\,
\frac{\lambda^{-3}}{\sqrt{M}}.
\]
The claimed error bound follows by combining this estimate with the corresponding
kernel-SGD error bounds in Theorem~\ref{Thm1} or Theorem~\ref{Thm2}, together with
the additional output-encoding error induced by the projection $P_2^{d_2}$.
\end{proof}

\section{Spectral Decay of Gaussian Kernel Integral Operators on Hilbert Spaces} \label{Appendix:Eigenvalues decay}
We consider the Gaussian kernel defined on an infinite-dimensional separable Hilbert space $\mathcal{U}$,
\[
k(u,u'):=\exp\l\{-\frac{\l\|u-u'\r\|_{\mathcal{U}}^2}{2\ell^2}\r\}.
\]
Let $\rho_u=\mathcal{N}(0,C)$ be a Gaussian probability measure on $\mathcal{U}$, where the covariance operator $C$ is positive self-adjoint and admits the spectral decomposition
\[
C=\sum_{i\geq1}\sigma_i^2\langle\cdot,e_i\rangle_{\mathcal{U}} \ e_i.
\]
Here $\{(\sigma_i^2,e_i)\}_{i\geq1}$ are the eigenpairs of $C$, satisfying
\[
\sum_{i\geq1}\sigma_i^2<\infty,
\]
and $\{e_i\}_{i\geq1}$ forms an orthonormal basis of $\mathcal{U}$. For a random element $u\sim\mathcal{N}(0,C)$, the coordinates satisfy
\[
\langle u,e_i\rangle\sim\mathcal{N}(0,\sigma_i^2).
\]
We further assume that the Gaussian measure is nondegenerate, i.e.,
\[
\sigma_i>0,\qquad\text{for all }i\geq1.
\]

Denote by $L_k:L^2(\mathcal{U},\rho_u)\to L^2(\mathcal{U},\rho_u)$ the integral operator induced by the kernel $k$, defined by
\[
L_kf(u):=\int_{\mathcal{U}}k(u,u')f(u')\mathrm{d}\rho_u(u').
\]
The following proposition characterizes the eigenvalue decay of $L_k$ under the nondegenerate Gaussian measure $\rho_u$.

\begin{proposition}
    For $q>0$, let $\mathrm{Tr}(L_k^q)$ and $\mathrm{Tr}(C^q)$ denote the sums of the $q-$th powers of the eigenvalues of $L_k$ and $C$, respectively. Then
    \[
    \mathrm{Tr}(L_k^q)<\infty\quad\text{if and only if}\quad\mathrm{Tr}(C^q)<\infty.
    \]
\end{proposition}
The proposition shows that the spectral behavior of the Gaussian
kernel integral operator is determined by the covariance operator
of the underlying nondegenerate Gaussian measure.
In particular, the trace condition for $L_k$ is equivalent to that
of $C$, relating the spectral complexity of the associated RKHS
to the spectrum of the covariance operator.
\begin{proof}
    Consider first the Gaussian kernel on $\mathbb{R}$,
    \[
    k_0(x,x')=\exp\l\{-\frac{\l\|x-x'\r\|_{2}^2}{2\ell^2}\r\}.
    \]
    The eigenvalues and eigenfunctions of the corresponding integral operator with respect to the Gaussian measure $\mathcal{N}(0,\sigma^2)$ are given in \cite[Section 4.3.1]{williams2006gaussian} by
    \begin{equation*}
    \begin{cases}
        \lambda_i=\sqrt{\frac{2a}{A}}B^i,\\
        \phi_i(x)=\exp\{-(c-a)x^2\}H_i\l(\sqrt{2c}x\r),
    \end{cases}
\end{equation*}
where $i\geq0$ and
\[
H_i(x)=(-1)^i\exp\{x^2\}\frac{\mathrm{d}^i}{\mathrm{d}x^i}\exp\{-x^2\}
\]
are the Hermite polynomials. The parameters are defined as
\[
a=\frac{1}{4\sigma^2},\quad b=\frac{1}{2\ell^2},\quad c=\sqrt{a^2+2ab},\quad A=a+b+c,\quad B=\frac{b}{A}.
\]

Define
\[
\Gamma:=\l\{\gamma=(\gamma_i)_{i\geq1}\in\bn_0^{\bn}:\ \mathrm{supp}(\gamma)\text{ is finite}\r\}.
\]
Since the kernel $k$ can be written as
\[
k(u,u')=\prod_{i\geq1}k_0\l(\langle u,e_i\rangle,\langle u',e_i\rangle\r),
\]
the nonzero eigenvalues of $L_k$ are given by
\[
\l(\prod_{i\geq1}\sqrt{\frac{2a_i}{A_i}}B_i^{\gamma_i}\r)_{\gamma\in\Gamma},
\]
where
\[
a_i=\frac{1}{4\sigma_i^2},\quad b=\frac{1}{2\ell^2},\quad c_i=\sqrt{a_i^2+2a_ib},\quad A_i=a_i+b+c_i,\quad B_i=\frac{b}{A_i}.
\]
The definitions give
\[
0<B_i<1,\qquad
\sqrt{\frac{2a_i}{A_i}}
=\frac{a_i+c_i}{A_i}=1-B_i,
\qquad
B_i\sim\frac{\sigma_i^2}{\ell^2}
\quad(i\to\infty).
\]
Thus \(\sum_{i\ge1}B_i<\infty\). Since
\(-\log(1-B_i)\sim B_i\), we obtain
\[
c_*:=\prod_{i\ge1}\sqrt{\frac{2a_i}{A_i}}
=\prod_{i\ge1}(1-B_i)>0.
\]
For \(q>0\), monotone convergence gives
\[
\begin{aligned}
\mathrm{Tr}(L_k^q)
&=c_*^q\sum_{\gamma\in\Gamma}
  \prod_{i\ge1}B_i^{q\gamma_i}\\
&=c_*^q\lim_{n\to\infty}
  \prod_{i=1}^n(1-B_i^q)^{-1}.
\end{aligned}
\]
Since \(-\log(1-B_i^q)\sim B_i^q\), it follows that
\[
\mathrm{Tr}(L_k^q)<\infty
\iff \sum_{i\ge1}B_i^q<\infty
\iff \sum_{i\ge1}\sigma_i^{2q}<\infty
\iff \mathrm{Tr}(C^q)<\infty.
\]
This completes the proof.
\end{proof}

\bibliographystyle{plain}
\bibliography{ref}
    
\end{document}